\documentclass{article}
\usepackage[accepted]{icml2024}
\usepackage{comment}
\usepackage{microtype}
\usepackage{graphicx}
\usepackage{subcaption}
\usepackage{booktabs}
\usepackage{hyperref}

\usepackage{icml2024}
\usepackage{amsmath}
\usepackage{amssymb}
\usepackage{mathtools}
\usepackage{amsthm}
\usepackage[capitalize,noabbrev]{cleveref}
\usepackage[textsize=tiny]{todonotes}
\theoremstyle{plain}
\newtheorem{theorem}{Theorem}[section]
\newtheorem{proposition}[theorem]{Proposition}
\newtheorem{lemma}[theorem]{Lemma}

\theoremstyle{definition}

\newtheorem{assu}[theorem]{Assumption}
\theoremstyle{remark}

\icmltitlerunning{Fast and Sample Efficient Multi-Task Representation Learning in Stochastic Contextual Bandits}

\usepackage{setspace}
\DeclareMathOperator*{\argmax}{arg\,max}
\DeclareMathOperator*{\argmin}{arg\,min}
\DeclareMathAlphabet{\pazocal}{OMS}{zplm}{m}{n}

\newcommand\scalemath[2]{\scalebox{#1}{\mbox{\ensuremath{\displaystyle #2}}}}

\newcommand{\E}{\mathbb{E}}
\newcommand{\SE}{\mathrm{SD}}  % {\mathrm{dist_{sub}}}
\newcommand{\pS}{{\pazocal{S}}}
\newcommand{\pA}{{\pazocal{A}}}
\newcommand{\pC}{{\pazocal{C}}}
\newcommand{\pX}{\pazocal{X}}

\newcommand{\pR}{\pazocal{R}}
\newcommand{\pG}{{\pazocal{G}}}
\newcommand{\N}{\pazocal{N}}
\newcommand{\bR}{\mathbb{R}}
\newcommand{\bE}{\mathbb{E}}

\renewcommand{\l}{\ell}
\newcommand{\wB}{\widehat{B}}
\newcommand{\wW}{\widehat{W}}
\newcommand{\ww}{\widehat{w}}
\newcommand{\Thetahat}{\widehat{\Theta}}
\newcommand{\Thetas}{\Theta^\star}
\newcommand{\Ot}{\tilde{O}}
\newcommand{\GradB}{\mathrm{GradB}}
\newcommand{\svdeq}{\overset{\mathrm{SVD}}=} %{\stackrel{EVD}{=}}
\newcommand{\sigmin}{{\sigma_{\min}^\star}}
\newcommand{\sigmax}{{\sigma_{\max}^\star}}
\newcommand{\Var}{\mathrm{Var}}

\newcommand{\thetats}{\theta_t^\star}
\newcommand{\thetahatmt}{\widehat{\theta}_{m-1, t}}
\newcommand{\thetahatt}{\widehat{\theta}_t}
\usepackage{dsfont}
\newcommand{\indic}{\mathds{1}}
\newcommand{\delt}{\delta_\l}
\newcommand{\delto}{\delta_{\l + 1}}
\newcommand{\nt}{{n, t}}
\newcommand{\xnt}{x_\nt}
\newcommand{\xts}{x_n^\star}
\newcommand{\pxc}{\phi(x_n, c_n)}
\newcommand{\pxsc}{\phi(x_n^\star, c_n)}
\newcommand{\pxtc}{\phi(x_\nt, c_\nt)}
\newcommand{\pxtsc}{\phi(x_\nt^\star, c_\nt)}
\newcommand{\norm}[1]{\left\lVert#1\right\rVert}
\newcommand{\cblue}{\color{blue}}
\newcommand{\sm}{{\scalebox{.6}{(m)}}}
\newcommand{\szero}{{\scalebox{.6}{(0)}}}
\newcommand{\smo}{{\scalebox{.6}{(m-1)}}}

\newcommand{\ymt}[2]{Y_{#1}^{{\scalebox{.6}{(#2)}}}}
\newcommand{\phimt}[2]{\Phi_{#1}^{{\scalebox{.6}{(#2)}}}}
\newcommand{\etamt}[2]{\eta_{#1}^{{\scalebox{.6}{(#2)}}}}
\newcommand{\Bm}[1]{\wB^{\scalebox{.6}{(#1)}}}
\newcommand{\Wm}[1]{\wW^{\scalebox{.6}{(#1)}}}

\newcommand{\iidsim}{\stackrel{\mathrm{iid}}{\thicksim }}
\newcommand{\n}{{\cal{N}}}

\begin{document}

\twocolumn[
\icmltitle{Fast and Sample Efficient Multi-Task Representation Learning\\ in Stochastic Contextual Bandits} % need to be updated

%\icmlsetsymbol{equal}{*}

\begin{icmlauthorlist}
\icmlauthor{Jiabin Lin}{yyy}
\icmlauthor{Shana Moothedath}{yyy}
\icmlauthor{Namrata Vaswani}{yyy}
%\icmlauthor{}{sch}
%\icmlauthor{}{sch}
\end{icmlauthorlist}

\icmlaffiliation{yyy}{Department of Electrical and Computer Engineering, Iowa State University, Ames IA 50011-1250, USA}

\icmlcorrespondingauthor{Jiabin Lin}{jiabin@iastate.edu}
%\icmlcorrespondingauthor{Shana Moothedath}{mshana@iastate.edu}
%\icmlcorrespondingauthor{Namrata Vaswani}{namrata@iastate.edu}

\icmlkeywords{Multi-task learning, Linear contextual bandits, Low dimensional learning, Representation learning}

\vskip 0.3in
]

\printAffiliationsAndNotice

\begin{abstract}
We study how representation learning can improve the learning efficiency of contextual bandit problems.
We study the setting where we play $T$ contextual linear bandits with dimension $d$ simultaneously, and these $T$ bandit tasks collectively share a common linear representation with a dimensionality of $r\ll d$.
We present a new algorithm based on alternating projected gradient descent (GD) and minimization estimator to recover a low-rank
feature matrix. 
%We obtain constructive provable guarantees for our estimator that provide a lower bound on the required sample complexity and an upper bound on the iteration complexity (total number of iterations needed to achieve a certain error level) of our proposed algorithm. 
%We show that our algorithm achieves $\epsilon$-accurate recovery of the feature matrix with order  $(d+T)r^3 \log(1/\epsilon)$ total samples and order $NTdr \log(1/\epsilon)$ time for any $\epsilon >0$ that is lower bounded by the noise to signal ratio (NSR).  
Using the proposed estimator, we present a multi-task learning algorithm for linear contextual bandits and prove the regret bound of our algorithm. 
We presented experiments and compared the performance of our algorithm against benchmark algorithms.
\end{abstract}

\vspace{-2 mm}
\section{Introduction}
Contextual Bandits (CB) represent an online learning problem wherein sequential decisions are made based on observed contexts, aiming to optimize rewards in a dynamic environment with immediate feedback.
In CBs, the environment presents a context in each round, and in response, the agent selects an action that yields a reward. The agent's objective is to choose actions to maximize cumulative reward over  $N$ rounds. This introduces the exploration-exploitation dilemma, as the agent must balance exploratory actions to estimate the environment's reward function and exploitative actions that maximize the overall return \citep{bubeck2012regret, lattimore2020bandit}. CB algorithms find applications in various fields, including  robotics \citep{srivastava2014surveillance}, clinical trials \citep{aziz2021multi}, communications \citep{anandkumar2011distributed}, and recommender systems \citep{li2010contextual}. 
 
Multi-task representation learning  is the problem of learning
a common low-dimensional representation among multiple
related tasks \cite{caruana1997multitask}.
Multi-task learning enables models to tackle multiple related tasks simultaneously, leveraging common patterns and improving overall performance \citep{zhang2018overview, wang2016distributed, thekumparampil2021sample}.  By sharing knowledge across tasks, multi-task learning can lead to more efficient and effective models, especially when data is limited or expensive. Multi-task bandit learning has gained interest recently \cite{deshmukh2017multi, fang2015active, cella2023multi, hu2021near, yang2020impact, lin2024distributed}. Many applications of CBs, such as recommending movies or TV shows to users and suggesting personalized treatment plans for patients with various medical conditions, involve related tasks.
These applications can significantly benefit from this approach, as demonstrated in our empirical analysis in Section~\ref{sec:sim}.
This paper investigates the benefit of using representation learning in CBs theoretically and experimentally.

While representation learning has demonstrated remarkable success across various applications \cite{bengio2013representation}, its theoretical understanding still remains underexplored. A prevalent assumption in the literature is the presence of a shared common representation among different tasks.  \cite{maurer2016benefit} introduced a general approach to learning data representation in both multi-task supervised learning and learning-to-learn scenarios. 
\cite{du2020few} delved into few-shot learning through representation learning, making assumptions about a common representation shared between source and target tasks.  
\cite{tripuraneni2021provable} specifically addressed the challenge of multi-task linear regression with low-rank representation, presenting algorithms with robust statistical rates. Related parallel works address the same mathematical problem (referred to as low-rank column-wise compressive sensing) or its generalization (called low-rank phase retrieval) and provide better - sample-efficient and faster - solutions \cite{lrpr_icml, lrpr_best,lrpr_gdmin,collins2021exploiting,lrpr_gdmin_2}.

Motivated by the outcomes of multi-task learning in supervised learning, numerous recent works have explored the advantages of representation learning in the context of sequential decision-making problems, including reinforcement learning (RL) and bandit learning. Our paper studies the multi-task bandit problem similar to the setting in \cite{hu2021near, yang2020impact, cella2023multi}.
We consider $T$ tasks of $d$-dimensional (infinte-arm) linear bandits are concurrently learned for $N$ rounds.
The expected reward for choosing an arm $x$ for a context $c$ and task $t$ is $\phi_t(x,c)^\top \thetats$, where $\thetats$ is an unknown linear parameter and $\phi_t(x,c) \in \mathbb{R}^d$ is the feature vector.
To take advantage of the
multi-task representation learning framework, we assume
that $\thetats$’s lie in an unknown $r$-dimensional subspace of $\mathbb{R}^d$, where $r$ is much smaller compared to $d$ and $T$ \cite{hu2021near, yang2020impact}.
The dependence on the tasks makes it possible to achieve a regret bound better than solving each task independently. A naive adaptation of the optimism in the face of uncertainty principle (OFUL) algorithm in \cite{abbasi2011improved} will lead to an $\Ot(Td\sqrt{N})$ regret for solving the $T$ tasks individually. By leveraging the common representation structure of these tasks, we propose an alternating projected GD and minimization-based estimator to solve the multi-task CB problem.  We provide the convergence guarantee for our estimator and the regret bound of the multi-task learning algorithm and, through extensive simulations, validate the advantage of the proposed approach over the state-of-the-art approaches.
%{\cblue Our estimator achieves $\epsilon$-accurate recovery o fthe feature matrix, where $\epsilon$ is lower bounded by noise-to-signal ratio.} We provide a lower bound on the required sample complexity and an upper bound on the iteration time and communication cost.

%{\bf Organization:} This paper is organized as follows. In Section~\ref{sec:prob}, we present the notations and problem setting. In Section~\ref{sec:rel}, we present the related work. In Section~\ref{sec:alg}, we describe the proposed algorithm and the main result. In Section~\ref{sec:sim}, we provide the numerical experiment results using synthetic data and real data, validating the effectiveness of our approach. Finally, we conclude in Section~\ref{sec:conc}. 
%The detailed analysis and proofs are presented in the supplementary material.
\vspace{-2 mm}
\section{Problem Setting}\label{sec:prob}
{\bf Notations:}
For any positive integer $n$, the set $[n]$ represents $\{1, 2, \cdots, n\}$. For any vector $x$, we use $\norm{x}$ to denote its $\ell_2$ norm. For a matrix $A$, we use $\norm{A}$ to denote the $2$-norm of $A$, $\norm{A}_F$ to denote the Frobenius norm, and $\norm{A}_{\max} = \max_{i,j} |A_{i,j}|$ to denote the max-norm. $\top$ denotes the transpose of a matrix or vector, while $|x|$ represents the element-wise absolute value of a vector $x$. The symbol $I_n$ (or sometimes just $I$) represents the $n \times n$ identity matrix. We use $e_k$ to denote the $k-$th canonical basis vector, i.e., the $k-$th column of $I$. For any matrix $A$, $a_k$ denotes its $k-$th column. \\

\vspace{-2 mm}
\subsection{Problem Formulation}
This section introduces the standard linear bandit problem and extends it to our specific setting: representation learning in linear bandits with a low-rank structure. We denote the action set as $\pX$ and the context set as $\pC$. The environment interacts through a fixed but unknown reward function $y: \pX \times \pC \rightarrow \bR$. In standard linear bandits, at each round $n \in [N]$, the agent observes a context $c_n \in \pC$ and chooses an action $x_n \in \pX$. For every combination of context and action $(x, c)$, there is a corresponding feature vector $\phi(x, c) \in \bR^d$. When the agent chooses an action $x_n$ for a given context $c_n$, it receives a reward $y_n \in \bR$,
defined as
$$
y_n := \langle \pxc, \theta^\star \rangle + \eta_n, 
$$
where $\theta^\star \in \bR^d$ represents the unknown but fixed reward parameter, and $\eta_n$ denotes a zero-mean $\sigma$-Gaussian additive noise. The term $\langle \pxc, \theta^\star \rangle$ represents the expected reward for selecting action $x_n$ in context $c_n$ at round $n$, i.e., $r_n = \bE[y_n] = \langle \pxc, \theta^\star \rangle$. The goal of the agent is to choose the best action $x_n^\star$ at each round $n \in [N]$ to maximize the cumulative reward $\sum_{n=1}^N y_n$, or in other words, to minimize the cumulative regret:
$$
\pR_N = \sum_{n=1}^N \langle \pxsc, \theta^\star \rangle - \sum_{n=1}^N \langle \pxc, \theta^\star \rangle, 
$$
where $\xts$ represents the best action at round $n$ for context $c_n$, and $x_n$ denotes the action chosen by the agent. 

This paper explores representation learning in linear bandits with a low-rank structure. We consider a scenario where $t \in [T]$ tasks deal with related sequential decision-making problems. For every round $n \in [N]$, each task $t$ observes a context $c_\nt \in \pC_\nt$ and chooses an action $\xnt \in \pA_\nt$. After an action is chosen for each task $t$ at round $n$, the environment provides a reward $y_\nt$, where $y_\nt:= \langle \pxtc, \thetats \rangle + \eta_\nt$ and $\thetats \in \bR^d$ is the unknown reward parameter for task $t$. The goal is to choose the best action $\xnt$ for each task $t \in [T]$ and each round $n \in [N]$ to maximize the cumulative reward $\sum_{n=1}^N \sum_{t=1}^T y_\nt$, which is equivalent to minimizing the cumulative (pseudo) regret 
$$
\scalemath{0.9}{\pR_{N,T} \hspace{-1.0 mm}=\hspace{-1.0 mm} \sum_{n=1}^N \sum_{t=1}^T \langle \pxtsc, \thetats \rangle - \sum_{n=1}^N \sum_{t=1}^T \langle \pxtc, \thetats \rangle, }
$$
where $x^\star_{n, t}$ denotes the best action for task $t$ at round $n$ given context $c_\nt$. We assume that $\Theta^\star = [\theta_1^\star \cdots \theta_T^\star]$ is a rank-$r$ matrix where $r \ll \min\{d, T\}$. This low-rank structure improves collaborative learning among the agents, which enhances the overall learning efficiency.

\subsection{Preliminaries}
Let $\Theta^\star \svdeq B^\star {\Sigma V^{\star}} := B^\star  W^\star$
%\begin{equation}\label{eq:Xstar}
%\Xstar \svdeq \Ustar \overbrace{\bSigma \underbrace{\Vstar{}^\top}_{\Bstar}}^{\tB}  %_{\tB}
%\Theta^\star \svdeq B^\star {\Sigma V^{\star}} := B^\star  W^\star
%\end{equation}
denote its reduced (rank $r$) SVD, i.e., $B^\star$ and ${V^\star}^\top$ are matrices with orthonormal columns {\em (basis matrices)}, $B^\star$ is  $d \times r$, $V^\star$ is $r \times T$, and $\Sigma$ is an $r \times r$ diagonal matrix with non-negative entries (singular values). We let  $W^\star:= \Sigma V^{\star}$.
We use $\sigmax$ and $\sigmin$ to denote the maximum and minimum singular values of $\Sigma$ and we define its condition number  as
 $\kappa:= \sigmax/\sigmin.$
 We now detail the assumptions we use in our analysis.
%\begin{assu}\label{assume:noise}
%Each element $\eta_t$ of the noise vector is assumed to be independently and identically distributed (i.i.d) according to a Gaussian distribution with a mean of zero and variance of $\sigma_\eta^2$, where $\sigma_\eta^2 \leqslant c \frac{\| \thetats \|^2}{k^3 \kappa^6}$, i.e., $\eta_{i, t} \iidsim \n(0, \sigma_\eta^2)$. {\cblue check}
%\end{assu}
%\begin{assu}\label{assume:norm}
%For simplicity, we assume  $\pxtc^{\top} \theta_t^{\star} \in [0,1]$, $\norm{\pxtc}_2 \leqslant 1$, and $\norm{\thetats}_2 \leqslant 1$, for all $t \in[T]$ and all $\xnt \in \pA_\nt$. 
%Furthermore, we assume that $\pxtc$ follows an i.i.d. standard Gaussian distribution. 
%\end{assu}
\begin{assu}(Gaussian design and noise)\label{assume:iid}
We assume $\pxtc$ follows an i.i.d. standard Gaussian distribution. Moreover, the additive noise variables $\eta_{n, t}$ follow i.i.d. Gaussian distribution with a zero mean and variance $\sigma_\eta^2$.
%, where $\sigma_\eta^2 \leqslant c \frac{\| \thetats \|^2}{r^3 \kappa^6}$.
%, i.e., $\eta_{i, t} \iidsim \n(0, \sigma_\eta^2)$.
 \end{assu}
Throughout, we work in the setting of random design linear regression, and in this context, Assumption~\ref{assume:iid} is standard \cite{cella2023multi, cella2021multi}. We note that while the assumption on $\pxtc$ holds for the first epoch in our algorithm, i.e., during random exploration, it is restrictive for future epochs. Let the feature vector $\pxtc$ has a mean value of $\mu_{x_\nt, c_\nt}$. The reward $y_\nt$ is given by the equation $y_\nt = (\pxtc - \mu_{x_\nt, c_\nt})^{\top} \thetats + \eta_\nt + \mu_{x_\nt, c_\nt} \thetats$, where the noise term is $\eta_\nt + \mu_{x_\nt, c_\nt} \thetats$. Such a (re)formulation allows us to relax Assumption~\ref{assume:iid} into a scenario where $\pxtc$ follows an independent Gaussian distribution with a mean of $\mu_{x_\nt, c_\nt}$. 
 %Our results do not critically rely on the noise-to-signal ratio (NSR), $\sigma_\eta^2 \leqslant c \frac{\| \thetats \|^2}{r^3 \kappa^6}$, assumption although its use simplifies several technical arguments.
\begin{assu}[Incoherence of right singular vectors] \label{assume:incoherence}
We assume that $\| w_t^\star \|^2 \leqslant \mu^2 \frac{r}{T}  {\sigma_{\max}^\star}^2$ for a constant $\mu \geqslant 1$. 
\end{assu}
Recovering the feature matrix is impossible without any structural assumption.
Notice that $y_t$s are not global functions of $\Theta^\star$: no $y_{n,t}$ is a function of the entire matrix $\Theta^\star$. %They are global for each column, but not across the different columns.
We thus need an assumption that enables correct interpolation across the different columns. The following incoherence (w.r.t. the canonical basis) assumption on the right singular vectors suffices for this purpose. 
 Such an assumption on both left and right singular vectors was first introduced in \cite{matcomp_candes} and used in recent works on representation learning \cite{tripuraneni2021provable}. %While for LRMC the measurements are 

\begin{assu}\label{assume:B}
(Common Feature Extractor). There exists a linear feature extractor denoted as $B^\star \in \bR^{d \times r}$, along with a set of linear coefficients $\{w_t\}_{t=1}^T$ such that the expected reward of the $t$-th task at the $n$-th round is given by $\bE[y_\nt] = \langle \pxtc, B^\star w_t^\star \rangle$, where $\Theta^\star = B^\star W^\star$. 
\end{assu}
%Assumptions~\ref{assume:noise} and~\ref{assume:norm} are standard assumptions widely used in the bandit literature.
%{\cblue Justification for Assumption 3}
Assumption~\ref{assume:B} is our main assumption, which assumes the existence of a common feature extractor for the reward parameter $\Theta^\star$. Because of this we can write $\Theta^\star = B^\star W^\star$, where $W^\star=[w_1^\star, w_2^\star, \ldots, w^\star_T]$. This assumption is used in many earlier works on representation learning, including \cite{yang2020impact,du2020few, hu2021near}.

\subsection{Contributions} 
In this paper, we proposed a novel alternating GD and minimization estimator for representation learning in linear bandits in the presence of a common feature extractor. Our algorithm builds upon the recently introduced technique known as alternating gradient descent and minimization (altGDmin) for low-rank matrix learning \citep{lrpr_gdmin,lrpr_gdmin_2}. 
Our work introduces two key extensions: (i) We adapt the AltGDMin approach to address sequential learning problems, specifically bandit learning, departing from static learning scenarios.
Hence our focus is on optimizing the selection of actions in addition to learning unknown parameters from observed data.
(ii) We account for noisy observed data, a common model in learning models, rather than non-noisy observations.

While there have been many recent works on multi-task learning for linear bandits \cite{hu2021near, yang2020impact, cella2023multi, tripuraneni2021provable}, those works either assume an optimal estimator that can solve the non-convex cost function in Eq.~\eqref{eq:cost} or considers a convex relaxation of the original cost function \cite{du2020few, cella2023multi}. We propose a sample and time-efficient estimator to learn the feature matrix. Our approach is GD-based, which is known to be much faster than convex relaxation methods \cite{cella2023multi, du2020few} and provides a sample-efficient estimation with guarantees. 
We prove that the alternating GD and minimization estimator achieves $\epsilon-$optimal convergence with the number of samples of the order of $(d+T)r^3 \log(1/\epsilon)$ and
order $NTdr \log(1/\epsilon)$ time, provided the noise-to-signal ratio is bounded. 
Using the proposed estimator, we propose a multitask bandit learning algorithm. We provide the regret bound for our algorithm. We validated the advantage of our algorithm through numerical simulations on synthetic and real-world MNIST datasets and illustrated the advantage of our algorithm over existing state-of-the-art benchmarks.

\section{Related Work}\label{sec:rel}

\noindent{\bf Multi-task supervised learning.} Multi-task representation learning is a well-studied problem that dates back to at least \cite{caruana1997multitask, thrun1998learning, baxter2000model}. Empirically, representation learning has shown its great power in various
domains  \cite{bengio2013representation}. The linear setting (multi-task linear regression or multi-task linear representation learning with a low rank model on the regression coefficients) was introduced in \cite{maurer2016benefit,tripuraneni2021provable,du2020few}
%The first theoretical analysis with sample complexity bounds was provided in \cite{baxter2000model}. \cite{maurer2016benefit} and follow-up works considered the setting where all tasks are i.i.d sampled from a certain distribution and analyzed the benefits of representation learning for reducing the sample complexity of the target task. \cite{maurer2016benefit}  showed that the i.i.d. assumption alone is insufficient if we want to take advantage of many samples per task.  Building upon their findings, leveraging additional assumptions on the distribution structure, \cite{du2020few} and \cite{tripuraneni2021provable} introduced algorithms to leverage all source data with improved sample complexity.
% \cite{tripuraneni2021provable}  also gives a computationally efficient algorithm for standard Gaussian inputs and a lower bound for subspace recovery in the low-dimensional linear setting.

The above works required more samples per task than the feature vector length, even while assuming a low-rank model on the regression coefficients' matrix. 
%his goes against all existing low rank matrix recovery literature
In interesting parallel works \cite{lrpr_gdmin,collins2021exploiting}, a fast and communication-efficient GD-based algorithm that was referred to as AltGDmin and FedRep, respectively, was introduced for solving the same mathematical problem that multi-task linear regression or multi-task linear representation learning solves. Follow-up work \cite{lrpr_gdmin_2} improved the guarantees for AltGDmin while also simplifying the proof. AltGDmin and FedRep algorithms are identical except for the initialization step. AltGDmin uses a better initialization and hence also has a better sample complexity by a factor of $r$. A phaseless measurements generalization of this problem, referred to as low-rank phase retrieval, was studied in \cite{lrpr_icml,lrpr_it,lrpr_best,lrpr_gdmin}. These works were motivated by applications in dynamic MRI \citep{lrpr_gdmin_mri_jp} and dynamic Fourier ptychography \citep{TCIgauri}.

  The primary emphasis of all the above works is on the statistical rate for multi-task supervised learning and does not address the exploration problem in online sequential decision-making problems such as bandits and RL. 
%\cite{tripuraneni2021provable} specifically addressed the challenge of multi-task linear regression with low-rank representation, similar to our work, and presented algorithms with robust statistical rates.
 %The problem studied in \cite{tripuraneni2021provable} is related to ours.
 %The fundamental difference is that in \cite{tripuraneni2021provable}, the feature vector of the data is independent of the task, while we consider a case where the data is task-dependent. The entire estimation process of the Method-of-Moments (MoM) estimator in \cite{tripuraneni2021provable} aligns with the initialization of our proposed algorithm, contributing to our improved results which is illustrated via numerical simulations in Section~\ref{sec:sim}. 

\noindent{\bf Multi-task RL learning.} Multi-task learning in RL domains is studied in many works, including \cite{taylor2009transfer, parisotto2015actor, d2024sharing, arora2020provable}. \cite{d2024sharing} demonstrated that representation learning has the potential to enhance the rate of the approximate value iteration algorithm.
\cite{arora2020provable} proved that representation learning can reduce the sample complexity of imitation learning. Both works
require a probabilistic assumption similar to that in \cite{maurer2016benefit} and the statistical rates are of similar forms as those in \cite{maurer2016benefit}.

\noindent{\bf Multi-task bandit learning.} The most closely related works to ours are the recent papers on multitask bandit learning \cite{hu2021near, yang2020impact,cella2023multi}. 
\cite{hu2021near} considered a concurrent learning setting with $T$  linear bandits with dimension $d$ that share a common $r$-dimensional linear representation. They proposed an optimism in the face of uncertainty principle (OFUL) algorithm that leverages the shared representation to achieve a $\Ot(T\sqrt{drN}+d\sqrt{rNT})$ regret bound, where $N$ is the number of rounds per task. The algorithm in \cite{hu2021near} requires solving a least-squares problem; however, the problem is nonconvex due to the rank condition ($r \ll \min\{d, T\}$).
 \cite{yang2020impact} considered the finite and infinite action case and proposed explore-then-commit algorithms. For the finite case, they utilize the estimator from \cite{du2020few}, and in the infinite case, they proposed a MoM-based estimator with $\Ot(Tr\sqrt{N}+d^{1.5}r\sqrt{NT})$ regret bound. However, these works assumed that the representation learning problem can be solved. \cite{du2020few} mentions that it should be possible to solve the original non-convex problem (Eq.~\eqref{eq:cost}) by solving a trace norm-based convex relaxation of it.
 \cite{cella2023multi} proposed a low-rank matrix estimation-based algorithm using trace-norm regularization and achieved $\Ot(T\sqrt{rN}+\sqrt{rNTd})$ regret bound under a restricted strong convexity condition when the rank is unknown. 
 However, there are no known guarantees to ensure that the trace norm-based relaxation solution is indeed also a solution to the original low-rank representation learning problem. The regret analysis in these works assumed solvability and optimality of the nonconvex problem.
 We focus on GD-based solutions since these are known to be much faster than convex relaxation methods \cite{cella2023multi, du2020few} and provide a sample-efficient estimator with guarantees.
 %The key difference between  \cite{cella2023multi} and \cite{hu2021near, yang2020impact} is that it does not require the knowledge of the rank $r$. However, it requires a restricted strong convexity condition.
%Their analyses often assume the solvability of the optimal solution for the regression problem.
 %\cite{yang2020impact} proposed an explore-and-then-exploit algorithm based on a MoM-based estimator with $\Ot(Tr\sqrt{N}+d^{1.5}r\sqrt{NT})$ regret bound.
 %A MoM estimator is also explored in \cite{tripuraneni2021provable} for low-rank matrix estimation for meta learning.
 %\cite{cella2023multi} proposed a low-rank matrix estimation-based algorithm using trace-norm regularization and achieved $\Ot(T\sqrt{rN}+\sqrt{rNTd})$ regret bound. The key difference between  \cite{cella2023multi} and \cite{hu2021near, yang2020impact} is that it does not require the knowledge of the rank $r$. However, it requires a restricted strong convexity condition.
%We focus on GD-based solutions since these are known to be much faster than convex relaxation methods \cite{cella2023multi, du2020few}
%Meta-learning is also explored in bandit literature \cite{kveton2021meta, basu2021no, simchowitz2021bayesian} 

\noindent{\bf Low rank and sparse bandits.} 
Some previous works also studied low rank and sparse bandits \cite{kveton2017stochastic}.
\cite{lale2019stochastic} considered a setting where the context vectors share a low-rank structure. Specifically, in their setting, the context vectors consist of two parts, i.e. $\hat{\phi}=\phi+\psi$, so that $\phi$ is from a hidden low-rank subspace and   is i.i.d. drawn from an isotropic distribution.
Works \cite{lu2021low, jun2019bilinear} studied bilinear bandits with low rank structure. They focus on estimating a low rank matrix $\Thetas$ when the reward function is given by $x^\top \Thetas y$, where $x, y$ denote the two actions chosen at each round.
Sparse interactive learning settings (e.g., bandits and reinforcement learning) are also studied in the literature \cite{cella2021multi, calandriello2014sparse, hao2020high, hao2021information}.

\section{The Proposed Algorithm: LRRL-AltGDMin}\label{sec:alg}

%\subsection{The Proposed Algorithm}

\begin{algorithm}[t]
    \caption{LRRL-AltGDMin Algorithm} 
    \label{alg1}
\begin{algorithmic}[1]
    \STATE Let $M = \lceil \log_2 \log_2 N \rceil$, $\pG_0 = 0$, $\pG_M = N$, $\pG_m = N^{1 - 2^{-m}}$ for $1 \leqslant m \leqslant M - 1$, let $\widehat{\theta}_t^{\szero} \leftarrow 0$\label{n}
    \FOR{$m \leftarrow 1, \cdots, M$}
    \FOR{$n \leftarrow \pG_{m-1} + 1, \cdots, \pG_m$}
        \STATE For each task $t \in [T]$: choose action $\xnt^\prime = \argmax_{\phi(x,c_{n,t}) \in \Psi_{t}} \pxtc^\top \widehat{\theta}_t^{\smo}$, obtain $y_\nt$, where $\Psi_{t}=\lbrace\phi(x,c_{n,t}): x\in \cal{X}\rbrace$, where $\phi(x,c_{n,t})\sim {\cal{N}}(\mu_{x,c}, I)$.

    \ENDFOR
    \STATE Compute $\ymt{t}{m} = [y_{\pG_{m-1} + 1, t}, \cdots, y_{\pG_m, t}]^\top$, $\phimt{t}{m} = [\phi(x_{\pG_{m-1} + 1, t}^\prime, c_{\pG_{m-1} + 1}), \cdots, \phi(x_{\pG_m, t}^\prime, c_{\pG_m})]^\top$ for $t \in [T]$\label{comp}
    \IF{$m=1$}
        \STATE {\bf Sample-split:} Partition the measurements and measure matrices into $2 L + 1$ equal-sized disjoint sets: one for initialization and $2 L$ sets for the iterations. Denote these by $\ymt{t, \tau}{m}$, $\phimt{t, \tau}{m}$, $\tau = 00, 01, \cdots 2 L$. 
        \STATE Initialize $\Bm{0}$ using Algorithm~\ref{alg2}
        \STATE Compute $\Bm{m}$ and $\Wm{m}$ using Algorithm~\ref{alg3}
    \ENDIF
    \IF{$m \geqslant 2$}
        \STATE {\bf Sample-split:} Partition the measurements and measure matrices into $2 L$ equal-sized disjoint sets. Denote these by $\ymt{t, \tau}{m}$, $\phimt{t, \tau}{m}$, $\tau = 01, \cdots 2 L$. 
        \STATE Compute $\Bm{m}$ and $\Wm{m}$ using Algorithm~\ref{alg3}
    \ENDIF
        \STATE For each task $t \in [T]$: let $\widehat{\theta}_t^{\sm} = \Bm{m} \widehat{w}_t^{\sm}$
    \ENDFOR
\end{algorithmic}
\end{algorithm}

\begin{algorithm}[t]
    \caption{Spectral Initialization for LRRL-AltGDMin} 
    \label{alg2}
\begin{algorithmic}[1]
    \STATE {\bfseries Input:} $\ymt{t, 00}{1}$, $\phimt{t, 00}{1}$, for $t \in [T]$
    \STATE {\bfseries Parameters:} Multiplier in specifying $\alpha$ for init step, $\tilde{C}$
    \STATE Using $\ymt{t}{1} \equiv \ymt{t, 00}{1}$, $\phimt{t}{1} \equiv \phimt{t, 00}{1}$, set $\alpha = \frac{\tilde{C}}{\pG_1 T} \sum_{n=1, t=1}^{\pG_1, T} y_{n, t}^2$
    \STATE $y_{t, trunc}(\alpha) := \ymt{t}{1} \circ \indic_{\{|\ymt{t}{1}| \leqslant \sqrt{\alpha}\}}$
    \STATE $\widehat{\Theta}_0 := \frac{1}{\pG_1} \sum_{t=1}^T {\phimt{t}{1}}^\top y_{t, trunc}(\alpha) e_t^\top$
    \STATE Set $\Bm{0} \leftarrow \text{top-}r\text{-singular-vectors of} \; \widehat{\Theta}_0$
\end{algorithmic}
\end{algorithm}

\begin{algorithm}[!ht]
    \caption{GD-Minimization for LRRL-AltGDMin}
    \label{alg3}
\begin{algorithmic}[1]
    \STATE {\bfseries Input:} $\ymt{t, \tau}{m}$, $\phimt{t, \tau}{m}$ for $t \in [T]$, $\tau = 01, \cdots 2L$, $\Bm{m-1}$ (from Algorithm~\ref{alg1})
    \STATE {\bfseries Parameters:} GD step size, $\gamma$; Number of iterations, $L$
    \STATE Set $B_0 \leftarrow \Bm{m-1}$
    \FOR{$\l = 1$ to $L$}
    \STATE Let $B \leftarrow B_{\l-1}$
    \STATE {\bfseries Update $w_{t, \l}, \theta_{t, \l}$:} For each $t \in [T]$, set $w_{t, \l} \leftarrow (\phimt{t, l}{m} B)^\dagger \ymt{t, l}{m}$ and set $\theta_{t, \l} \leftarrow B w_{t, \l}$
    \STATE {\bfseries Gradient w.r.t $B$:} With $\ymt{t}{m} \equiv \ymt{t, L + \l}{m}$, $\phimt{t}{m} \equiv \phimt{t, L + \l}{m}$, compute $\nabla_B f(B, W_{\l}) = \sum_{t=1}^T {\phimt{t}{m}}^\top (\phimt{t}{m} B w_{t, \l} - \ymt{t}{m}) w_{t, \l}^\top$
    \STATE {\bfseries GD step:} Set $\widehat{B}^+ \leftarrow B - \frac{\gamma}{\pG_m - \pG_{m-1}} \nabla_B f(B, W_{\l})$
    \STATE {\bfseries Projection step:} Compute $\widehat{B}^+ \overset{QR}{=} B^+ R^+$
    \STATE Set $B_{\l} \leftarrow B^+$
    \ENDFOR
    \STATE Set $\Bm{m} \leftarrow B_{L}$ and set $\Wm{m} \leftarrow W_{L}$
\end{algorithmic}
\end{algorithm}

This section presents our proposed algorithm (see Algorithm~\ref{alg1}). We refer to it as the Alternating Gradient Descent (GD) and Minimization algorithm for Low-Rank Representation Learning in linear bandits (LRRL-AltGDMin). This builds on the AltGDmin algorithm of \citep{lrpr_gdmin} mentioned earlier. 
 Our algorithm uses a doubling schedule rule \cite{gao2019batched,han2020sequential,simchi2019phase}. 
We update our estimation of $\Thetas$ only after completing an epoch, utilizing solely the samples collected within that epoch. 
Our algorithm consists of three main components: an exploration phase (data collection), initialization, and alternating GD and minimization steps. The pseudocode of our algorithm is presented in Algorithm~\ref{alg1}.
 %{\cblue 
%Our algorithm builds upon the recently introduced technique known as alternating gradient descent and minimization for low-rank matrix learning \citep{lrpr_gdmin, lrpr_gdmin_2}. Our work introduces two key extensions: (i) We adapt the AltGDMin approach to address sequential learning problems, specifically bandit learning, departing from static learning scenarios. Hence our focus is on optimizing the selection of actions in addition to learning unknown parameters from observed data. (ii) We account for noisy observed data, a common model in learning models, rather than non-noisy observations.} Next, we will provide a comprehensive description of our algorithm. 

We partition the learning horizon $N$ into $M+1$ epochs, $\pG_0, \pG_1, \ldots, \pG_M$, where $\pG_0=0$ and $\pG_M = N$.
Our algorithm is based on a greedy strategy. At each round $n$, each task $t \in [T]$ independently chooses an action $\xnt^\prime = \argmax_{\phi(x,c_{n,t}) \in \Psi_{t}} \pxtc^\top \widehat{\theta}_t^{\smo}$, which effectively aiming to maximize the expected reward. After choosing these actions, each task receives a corresponding reward $y_\nt$. After completing $(\pG_m - \pG_{m-1})$ rounds to collect data, the algorithm then proceeds to update the estimated parameters, which is achieved by finding a matrix $\widehat{\Theta} = \wB\wW$ that minimizes the cost function, defined as  
\begin{align}
f_m(\wB, \wW)\hspace{-1 mm}=\hspace{-5 mm} \sum_{n=\pG_{m-1}+1}^{\pG_m} \sum_{t=1}^T \norm{y_\nt - \phi(x_{n, t}, c_n)^\top \wB\widehat{w}_t}^2.\label{eq:cost}
\end{align}
Here $\wB\in \mathbb{R}^{d \times r}$ and $\wW =[\ww_1, \ldots, \ww_T]\in \mathbb{R}^{r \times T}$, and $\Thetahat=\wB \wW$ is the estimate of the parameter $\Thetas$ in the $m$-the epoch. This process effectively enhances the accuracy of future action selections. 

Because of the non-convexity of the cost function $f_m(\wB, \wW)$, our approach needs careful initialization. We draw inspiration from the spectral initialization idea. The process begins by calculating the top $r$ singular vectors of
\begin{align*}
\widehat{\Theta}_{0, full} &= \frac{1}{\pG_1} [({\phimt{1}{1}}^\top \ymt{1}{1}), \cdots, ({\phimt{T}{1}}^\top \ymt{T}{1})] \\
&= \frac{1}{\pG_1} \sum_{t=1}^T \sum_{n=1}^{\pG_1} \phi(x_{n, t}, c_{n}) y_{n, t} e_t^\top
\end{align*}
Here, $\phimt{t}{m}$, for $t \in [T]$, is the feature matrix obtained by stacking the feature vectors corresponding to task $t$ in the $m$-the epoch, i.e., $\phimt{t}{m} = [\phi(x_{\pG_{m-1} + 1, t}, c_{\pG_{m-1} + 1}), \cdots, \phi(x_{\pG_m, t}, c_{\pG_m})]^\top$. Upon careful analysis of this matrix, it can be observed that the expected value of its $t-$th task equals $\thetats$ and $\bE[\widehat{\Theta}_{0, full}] = \Theta^\star$. However, the large magnitude of the sum of independent sub-exponential random variables, which is defined by a maximum sub-exponential norm $\max_t \norm{\thetats} \leqslant \mu \sqrt{\frac{r}{T}} \sigma_{\max}^\star$, causes a challenge. This magnitude limits the ability to bound $\norm{\widehat{\Theta}_{0, full} - \Theta^\star}$ within the desired sample complexity. In order to solve this issue, we apply a truncation strategy borrowed from \cite{lrpr_gdmin}.
This involves initializing $\Bm{0}$ as the top $r$ left singular vectors of
\begin{align*}
\widehat{\Theta}_{0} &= \frac{1}{\pG_1} \sum_{t=1}^T \sum_{n=1}^{\pG_1} \phi(x_{n, t}, c_{n}) y_{n, t} e_t^\top \indic_{\{y_{t, n}^2 \leqslant \alpha\}} \\
&= \frac{1}{\pG_1} \sum_{t=1}^T \phi(x_{n, t}, c_{n}) y_{t, trunc}(\alpha) e_t^\top
\end{align*}
where $\alpha = \frac{\tilde{C}}{\pG_1 T} \sum_{n=1, t=1}^{\pG_1, T} y_\nt^2$, $\tilde{C} = 9 \kappa^2 \mu^2$, and $y_{t, trunc}(\alpha) := \ymt{t}{1} \circ \indic_{\{|\ymt{t}{1}| \leqslant \sqrt{\alpha}\}}$. Using Singular Value Decomposition (SVD), we extract the top $r$ left singular vectors from $\widehat{\Theta}_{0}$ to obtain our initial estimate $\Bm{0}$. This method efficiently filters large values while preserving others and provides a good starting point that ensures a robust guarantee in parameter estimation. 

After finding a good initial point, the algorithm performs the Alternating Gradient Descent optimization method to update the estimated parameter. The goal is to minimize the squared-loss cost function $f(B, W) = \sum_{n=\pG_{m-1}+1}^{\pG_m} \sum_{t=1}^T \norm{y_\nt - \phi(x_{n, t}, c_{n}) B w_t}^2$ by optimizing the estimated reward parameter matrix for all tasks. The process proceeds in the following manner. At each new iteration $\l$, 
\begin{itemize}
    \item {\bfseries Min-$W$:} Given that $w_t$ appears only in the $t-$th term of $f(B, W)$, optimizing each $w_t$ for the function $w_t \leftarrow \argmin_{\widetilde{w}_t} \norm{\ymt{t}{m} - \phimt{t}{m} B \widetilde{w}_t}^2$ individually is much simpler than optimizing $W$ for the function $W \leftarrow \argmin_{\widetilde{W}} \norm{f(B, \widetilde{W})}$. Consequently, we update the estimate $w_t$ by calculating $w_t = (\phimt{t}{m} B)^\dagger \ymt{t}{m}$ for every task $t \in [T]$. 
    \item {\bfseries ProjGD-$B$:} a single step of the Gradient Descent (GD) is performed to update $B$, which is given by $\widehat{B}^+ \leftarrow B - \gamma \nabla_B f(B, W)$. The updated matrix $B^+$ is obtained using QR decomposition, represented as $\widehat{B}^+ \overset{QR}{=} B^+ R^+$. 
\end{itemize}
Through this iterative process, the algorithm efficiently updates the estimated parameters, guaranteeing an optimized solution. 

\section{Analysis of LRRL-AltGDMin}
%\subsection{Guarantees of the LRRL-AltGDMin Estimator}
We have the following guarantee for our initialization algorithm presented in Algorithm~\ref{alg2}.
\begin{theorem} \label{new_4}
Assume that Assumptions~\ref{assume:iid} and~\ref{assume:incoherence} hold.
Assume that $\sigma_\eta^2 \leqslant c \frac{\delta_0^2}{r^2 \kappa^4 \pG_1} \| \thetats \|^2$.
Then with probability at least $1 - \exp(\log T - c \pG_1) - \exp(d - \frac{c \delta_0^2 \pG_1 T}{r^2 \mu^2 \kappa^4})$, we have 
$$
\SE(\Bm{0}, B^\star) \leqslant \delta_0. 
$$
\end{theorem}
Observe that Theorem~\ref{new_4} needs the noise-to-signal (NSR) ratio $\dfrac{\sigma_\eta^2}{\| \thetats \|^2} \leqslant \frac{c}{r^2}$, where $c<1$. This is necessary to demonstrate that the spectral initialization in Algorithm~\ref{alg2} produces a sufficiently good initialization.
%We present the proof in Appendix~\ref{app_new_4}.

%{\cblue Theorem~\ref{new_4} used the NSR upper bound because it simplifies the technical analysis. If we impose the standard zero mean i.i.d. assumption on the $\sigma_\eta$s, it should be possible to reduce the required bound on NSR to $1/k$. Moreover, if the initialization sample complexity is increased by a factor of $k$, then the NSR upper bound can be further reduced to a constant $c$.}

%Theorem~\ref{new_4} needs the NSR upper bound because we do not make any statistical assumptions on the noise. If we impose the standard zero mean i.i.d. assumption on the $\sigma_\eta$s, it should be possible to reduce the required bound on NSR to $1/k$. Moreover, if the initialization sample complexity is increased by a factor of $k$, then the NSR upper bound can be further reduced to a constant $c$.

In order to show that $\Bm{0}$ is a good enough initialization, we need to show that $\SE(\Bm{0}, B^\star) \leqslant \delta_0$ for a constant $\delta_0 < 1$ that is small enough. This is typically done using a $\sin \Theta$ theorem, e.g., Davis-Kahan or Wedin \cite{spectral_init_review}, which uses a bound on the error between $\widehat{\Theta}_{0}$ and a matrix whose span of top $r$ singular vectors equals that of $B^\star$.
Such a matrix may be $\E[\widehat{\Theta}_{0}]$ or something else that can be shown to be close to $\widehat{\Theta}_{0}$.
For our approach, it is not easy to compute $\E[\widehat{\Theta}_{0}]$ because the threshold, $\alpha$, used in the indicator function depends on all the $y_{n,t}^2$.
Our approach to solving this by using the sample-splitting idea: use a different independent set of measurements to compute $\alpha$ than those used for the rest of $\widehat{\Theta}_{0}$. Since this is a one-time step, it does not change the sample complexity order.\
We present the proof of Theorem~\ref{new_4} in Appendix~\ref{app_new_4}.

%The result below presents the error decay in our approach. 
\begin{theorem} \label{new_3}
Assume that Assumptions~\ref{assume:iid} and~\ref{assume:incoherence} hold, $\SE(B, B^\star) \leqslant \delt$, and $\sigma_\eta^2 \leqslant \frac{r}{T} \delt^2 {\sigma_{\min}^\star}^2$. If $\delt \leqslant \frac{0.02}{\mu \sqrt{r} \kappa^2}$, $\gamma = \frac{c_\gamma}{{\sigma_{\max}^\star}^2}$ with $c_\gamma \leqslant 0.5$, and if $$(\pG_m - \pG_{m-1}) T \geqslant C \mu^2 \kappa^4 (d + r) r\quad \mathrm{and}$$  $$(\pG_m - \pG_{m-1}) \geqslant C (r + \log T + \log d),$$ then with probability at least $1 -\ell d^{-10}$, 
$$
\SE(B^+, B^\star) \leqslant \delto := (1 - \frac{0.5472 c_\gamma}{\kappa^2}) \delt. 
$$
\end{theorem}
The above result proves that the error decays exponentially. We present the proof in Appendix~\ref{app_new_3}. Using Theorems~\ref{new_4} and~\ref{new_3}, we have the guarantee below on  estimation error.

\begin{theorem} \label{new_5}
Assume that Assumptions~\ref{assume:iid} and~\ref{assume:incoherence} hold and $\sigma_\eta^2 \leqslant \frac{c \| \thetats \|^2}{\mu r^3 \kappa^6 \pG_1}$. Set $\gamma = \frac{0.4}{{\sigma_{\max}^\star}^2}$ and $L = C \kappa^2 \log(\frac{1}{\max(\epsilon, \epsilon_{noise})})$. If $$\scalemath{0.9}{(\pG_m - \pG_{m-1}) T \geqslant C \mu^2 \kappa^6 d r^2 (\mu^2 \kappa^2 r + \log(\frac{1}{\max(\epsilon, \epsilon_{noise})}))}$$ and $$\pG_m - \pG_{m-1} \geqslant C \kappa^2 (r + \log T + \log d) \log(\frac{1}{\max(\epsilon, \epsilon_{noise})}),$$ then with probability at least $1 - L d^{-10}$, 
\begin{align*}
\SE(B, B^\star) &\leqslant \max(\epsilon, \epsilon_{noise}) \quad \text{and}\\
\| \thetahatt - \thetats \| &\leqslant \max(\epsilon, \epsilon_{noise}) \| \thetats \| \quad \text{for all} \quad t \in [T], 
\end{align*}
where $\epsilon_{noise} = C \kappa^2 \sqrt{NSR}$, $NSR := \frac{\sigma_\eta^2}{\min_t \| \thetats \|^2}$. The time complexity is $(\pG_m - \pG_{m-1}) T d r \cdot L = C \kappa^2 (\pG_m - \pG_{m-1}) T d r \log(\frac{1}{\max(\epsilon, \epsilon_{noise})})$. The communication complexity is $d r$ per node per iteration. 
\end{theorem}
This result shows that the error decays exponentially until it reaches the (normalized) ``noise-level"  ${\sigma_\eta}^2/\|\thetats\|^2$, but saturates after that. We present the proof in Appendix~\ref{app_new_5}. 

\noindent{\bf Sample complexity.} To understand the necessary lower bound on $(\pG_m - \pG_{m-1}) T$, it is crucial to consider it in terms of the sample complexity. This can be performed by assuming that $d \approx T$ approximately. When logarithmic factors are ignored and considering $\kappa$ and $\mu$ as constant values, our results indicate that an order value of $r^3$ samples per epoch is sufficient. Without making the low-rank assumption and without using our algorithm, if we were to perform matrix inversion for $\phimt{t}{m}$ in order to extract each vector $\thetats$ from $\ymt{t}{m}$, we would need at least $\pG_m - \pG_{m-1} \geq d$ samples per epoch, instead of just $r^3$. If the low-rank assumption holds and $r \ll d$ (e.g., $r = \log d$), our approach significantly lowers the amount of sample complexity needed in comparison to the requirement for inverting $\phimt{t}{m}$. 

\noindent{\bf Time and communication complexity.} When analyzing the time complexity of a given $m$-th epoch, we start by calculating the computation time needed for the initialization step. To calculate $\Theta_0$, it is necessary to give a time of order $(\pG_m - \pG_{m-1}) T d$. Furthermore, the time complexity of the $r$-SVD step $d T r$ times the number of iterations required. An important observation is that to obtain an initial estimate of the span of $B^*$ that is $\delta_0$-accurate, where $\delta_0 = \frac{c}{\mu \sqrt{r} \kappa^2}$, it is sufficient to use an order $\log(\mu r \kappa)$ number of iterations. Therefore, the total complexity of this initialization phase can be expressed as $O(d T ((\pG_m - \pG_{m-1}) + r) \log (\mu r \kappa)) = O((\pG_m - \pG_{m-1}) T d \log \mu r \kappa)$, given that $(\pG_m - \pG_{m-1}) \geqslant r$. The time required for each gradient computation is $(\pG_m - \pG_{m-1}) T d r$. The QR decomposition process requires a time complexity of order $d r^2$. Additionally, the time required to update the columns of matrix $W$ using the least squares method is $O((\pG_m - \pG_{m-1}) T d r)$. The number of iterations of these steps for each epoch can be expressed as $L = O(\kappa^2 \log(\frac{1}{\max(\epsilon, \epsilon_{noise})}))$. In summary, the overall time complexity for the process can be determined as $O((\pG_m - \pG_{m-1}) T d \log(\mu r \kappa) + \max((\pG_m - \pG_{m-1}) T d r, d r^2) \cdot M \cdot L) = O(\kappa^2 M (\pG_m - \pG_{m-1}) T d r \log(\frac{1}{\max(\epsilon, \epsilon_{noise})}) \log(\kappa))$. 

The communication complexity for each task in each iteration is of the order of $d r$. Hence, the total is $O(d r \cdot \kappa \log(\frac{1}{\max(\epsilon, \epsilon_{noise})}))$. 

We now present the regret bound for our algorithm.
%, showing that it is nearly optimal in terms of both regret and communication cost.
%\subsection{Regret Analysis}
%The lemma below presents a bound on the estimation error of the LRRL-AltGDMin algorithm. This bound is used to prove the regret bound of our approach.
%\begin{lemma} \label{Lemma2}
%Assume that Assumptions~\ref{assume:iid} and~\ref{assume:incoherence} hold.
%If $\sigma_\eta^2 \leqslant \frac{r}{T} \delt^2 {\sigma_{\min}^\star}^2$, $\delt \leqslant \frac{0.02}{\sqrt{r} \kappa^2}$, $\gamma = \frac{c_\gamma}{(\pG_m - \pG_{m-1}) {\sigma_{\max}^\star}^2}$ with $c_\gamma \leqslant 0.5$, and if $$(\pG_m - \pG_{m-1}) T \geqslant C \kappa^4 \mu^2 d r$$ and $$(\pG_m - \pG_{m-1}) \gtrsim \max(\log d, \log T, r),$$ then for any given epoch $m \in [M]$, it holds with probability at least $O(1 - d^{-10} - \epsilon^2 - T \exp(r - c (\pG_m - \pG_{m-1})) - (NT)^{-2})$ that
%$$\scalemath{0.9}{\| \wB \wW - B^\star W^\star \|_F^2 \leqslant \frac{(1 + 2 \epsilon_2)^2 C}{c^\prime} \mu^2 \delt^2 \frac{r}{T} {\sigma_{\max}^\star}^2 \sqrt{\frac{\epsilon^2 T}{\pG_m - \pG_{m-1}}}}. $$
%\end{lemma}

%\section{Main Results and Discussion}
%\subsection{Main Result}
%Consider an epoch $m \in [M]$ and a task $t \in [T]$. The cumulative regret in the $m$-th epoch is given by
%\begin{align}
%\pR_m &= \sum_{n=\pG_{m-1}+1}^{\pG_m} \sum_{t=1}^T \langle \pxtsc \thetats \rangle - \langle \pxtc \thetats \rangle \nonumber 
%\end{align}
%Below we present the regret bound of the proposed approach.
\begin{theorem}\label{thm-reg}
Assume that Assumptions~\ref{assume:iid} and~\ref{assume:incoherence} hold and $\sigma_\eta^2 \leqslant \frac{c \| \thetats \|^2}{\mu r^3 \kappa^6 \pG_1}$, If $$\scalemath{0.95}{(\pG_m - \pG_{m-1}) T \geqslant C \mu^2 \kappa^6 d r^2 (\mu^2 \kappa^2 r + \log(\frac{1}{\max(\epsilon, \epsilon_{noise})}))}$$ and $$\pG_m - \pG_{m-1} \geqslant C \kappa^2 (r + \log T + \log d) \log(\frac{1}{\max(\epsilon, \epsilon_{noise})}),$$ 
then with probability at least $1 - \delta - L d^{-10}$, the upper bound of cumulative regret for Algorithm~\ref{alg1} is
\begin{align*}
\scalemath{0.95}{\pR_{N, T} \leqslant 2 \mu \sigma_{\max}^\star \sqrt{rNT \log{\frac{1}{\delta}}} (1 + \log \log N).} 
\end{align*}
\end{theorem}
Proof of Theorem~\ref{thm-reg} and supporting results are presented in Appendix~\ref{app_reg}.
 Our sample complexity on source task scales sublinearly with $T$ and improves the linear dependence in \cite{yang2020impact}, while the target sample complexity scales with $k$ same as in \cite{yang2020impact}.
 \vspace{-2 mm}
\section{Simulations}\label{sec:sim}
In this section, we present the experimental results of our LRRL-AltGDMin algorithm on both synthetic and real-world MNIST datasets. We performed a comparative analysis of our algorithm with the  Method-of-Moments (MoM) algorithm proposed in \cite{yang2020impact, tripuraneni2021provable}, the trace-norm convex relaxation-based approach in \cite{cella2023multi}, along with a baseline naive approach. The naive approach utilizes the Thompson Sampling (TS) algorithm to solve $T$ tasks independently. All experiments were conducted using Python.
%The MLinGreedy Algorithm is based
%on solving the nonconvex optimization problem given in Eq.~\eqref{eq:cost}. The results in \cite{yang2020impact} considered the optimal solution to Eq.~\eqref{eq:cost} is available through an estimator, however, the estimation approach is not detailed. 
%to solve the nonconvex optimization problem where we use the Alternating Gradient Descent method for optimization. 
%On the other hand, the Method-of-Moments Algorithm is specifically designed for problems with an infinite action set. 

\subsection {Datasets}
\noindent{\bf Synthetic data:} 
We set the parameters as $d = 100$, and $K = 5$.
%Note that,  $\Theta^\star = B^\star {W}^{\star}$, where $B^\star$ is a $d \times r$ orthonormal matrix.
We generate the entries of $B^\star$ by orthonormalizing an i.i.d standard Gaussian matrix. The entries of  ${W}^{\star} \in \mathbb{R}^{r \times T}$ are generated from an i.i.d. Gaussian distribution. The matrices $\Phi_t$s
 were i.i.d. standard Gaussian.
 %We used a standard Gaussian distribution to generate our feature vector $\phi$. Subsequently, we started the procedure by generating a random orthogonal matrix $B^\star$ and a random matrix $W^\star$. The matrices were multiplied to generate the low-rank matrix $\Theta^\star$. 
 We considered a noise model with a mean of $0$ and a variance of $10^{-6}$ for the bandit feedback noise. The experiments were averaged over 100 independent trials. The plots also include the variance over the trials.
 In the synthetic experiment, we also considered another dataset with a smaller problem dimension $d = 20$ and $K = 5$. 
%The smaller input dimension $d=20$ is chosen such that we can implement the approach in \cite{cella2023multi} and compare our algorithm with it. The plots corresponding to this setting is presented in }
%The true parameter $\Theta^\star$ can be represented as $\Theta^\star = B^\star {W}^{\star}$, where $B^\star \in \mathbb{R}^{d \times r}$ is a orthonormal matrix. The orthonormal basis was obtained by orthonormalizing an i.i.d standard Gaussian matrix. The elements of the matrix ${W}^{\star} \in \mathbb{R}^{r \times T}$ were randomly chosen from an i.i.d Gaussian distribution. Throughout the experiments, we used i.i.d standard Gaussian matrices for the $\Phi_t$s. The noise in the bandit feedback was represented by a Gaussian distribution with a mean of zero and a variance of $10^{-6}$. The results were calculated by taking the average of 100 independent trials and including the observed variation among these trials. }

 %We set the  step size of GD for the $g$-th agent as  $\eta^\sg = 0.4/n {\widehat{\sigma}}^{\star^2}_{\max}{}$, where ${\widehat{\sigma}}^{\star^2}_{\max}{}$ is obtained as the largest diagonal entry of ${\R}^\sg_{T_{pm}}$.
\noindent{\bf MNIST data:} 
We used the MNIST dataset to validate the performance of our algorithm when implemented with real-world data. We set the number of actions $K = 2$ and created a total of $T = \binom{10}{2}$ tasks similar to \cite{yang2020impact}. Each task is characterized by a distinct pair $(i, j)$, where $0 \leqslant i < j \leqslant 9$. The set of MNIST images that represent the digit $i$ is denoted as $D_i$. For each round $n \in [N]$, we randomly choose one image from the set $D_i$ and another image from the set $D_j$ for every task $(i, j)$. The algorithm is presented with two images, and it assigns a reward of $1$ to the image with the larger digit value and a reward of $0$ to the other image. The feature matrix of each image is transformed into a feature vector $\phi \in \bR^{784}$ through vectorization. In order to calculate the estimated reward, we add random Gaussian noise with a mean of $0$ and a variance of $10^{-6}$. 

\subsection{Results and Discussions}

\begin{figure*}[hbt!] 
%\vspace{-3 mm}
\subcaptionbox{\footnotesize Synthetic data: rank $r=2$ \label{fig:1}}{\includegraphics[scale=0.2]{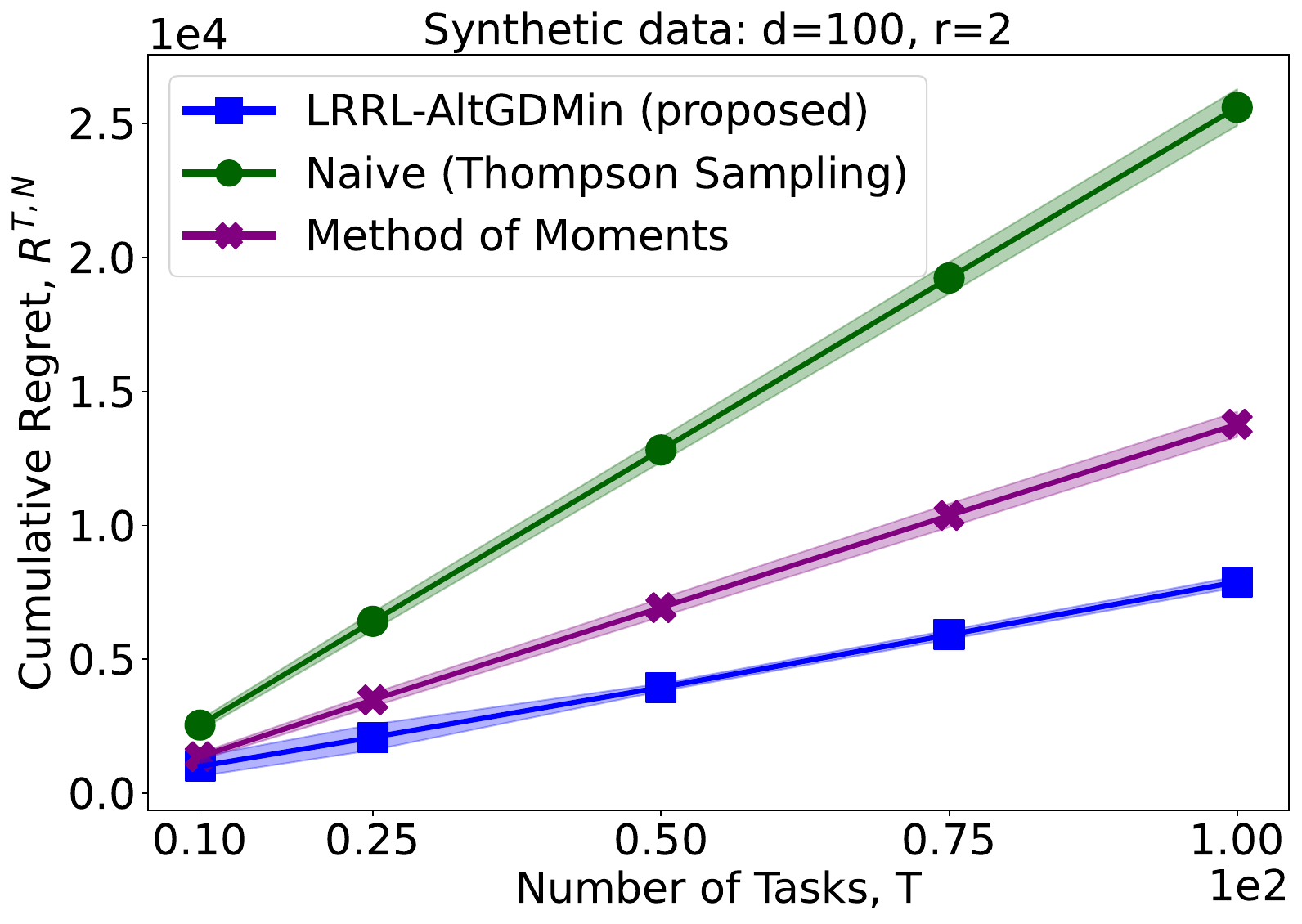}}\hspace{1.0 em}%
\subcaptionbox{\footnotesize Synthetic data: rank $r=4$ \label{fig:2}}{\includegraphics[scale=0.2]{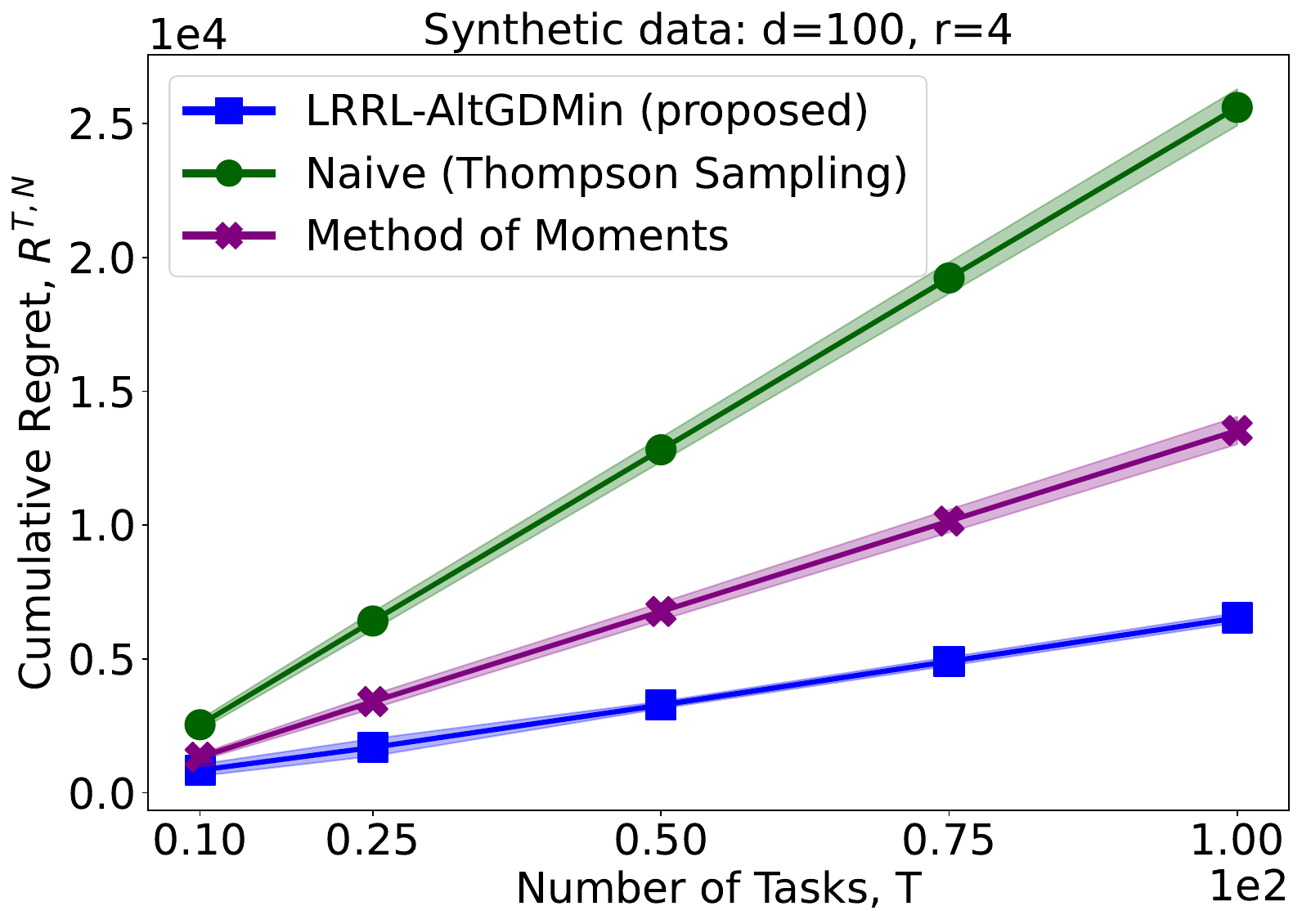}}\hspace{1.0 em}%
\subcaptionbox{\footnotesize Synthetic data: rank $r=8$ \label{fig:3}}{\includegraphics[scale=0.2]{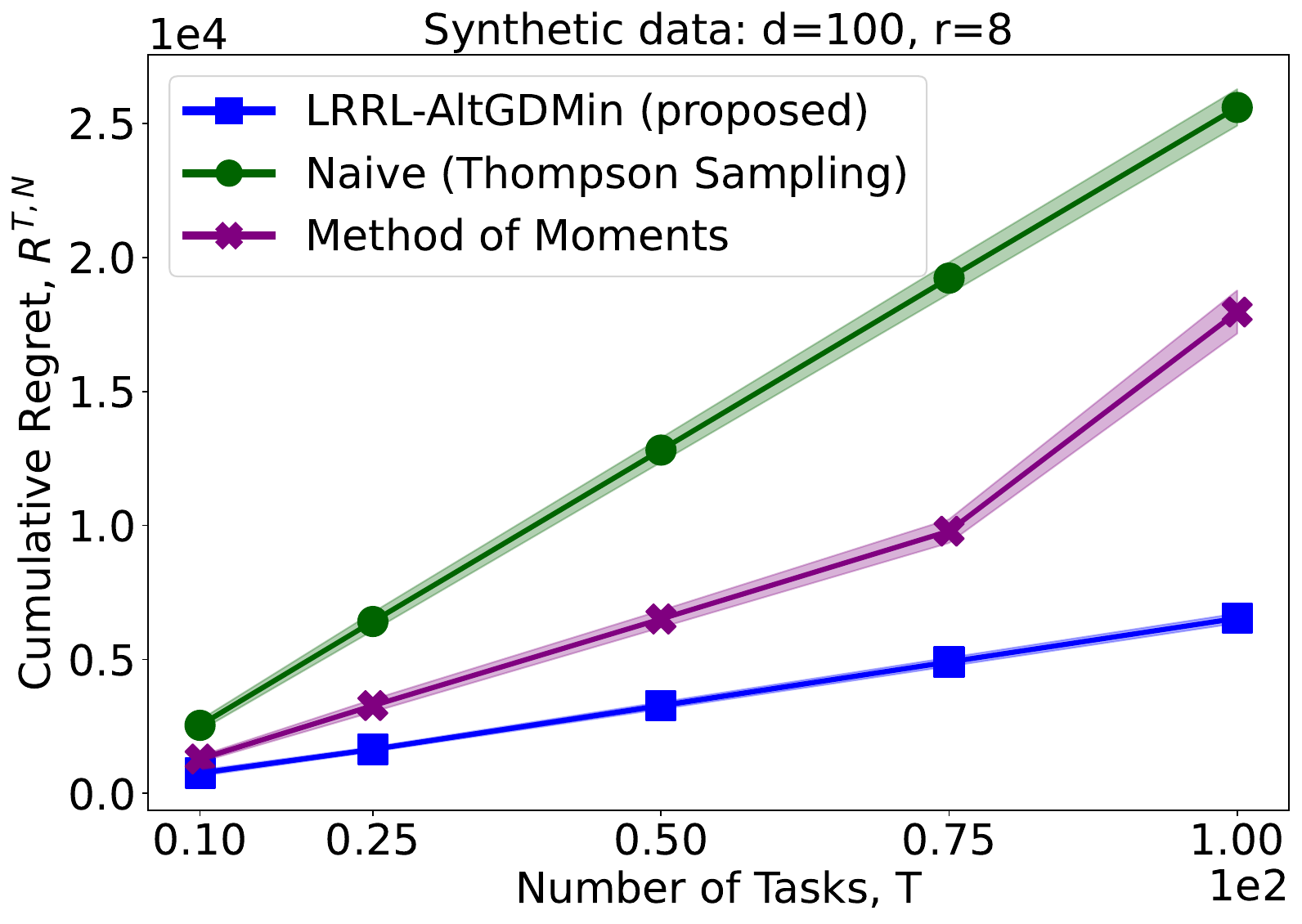}}\hspace{1.0 em}
\vspace{ 2mm}
\subcaptionbox{\footnotesize MNIST data: rank $r=2$ \label{fig:4}}{\includegraphics[scale=0.2]{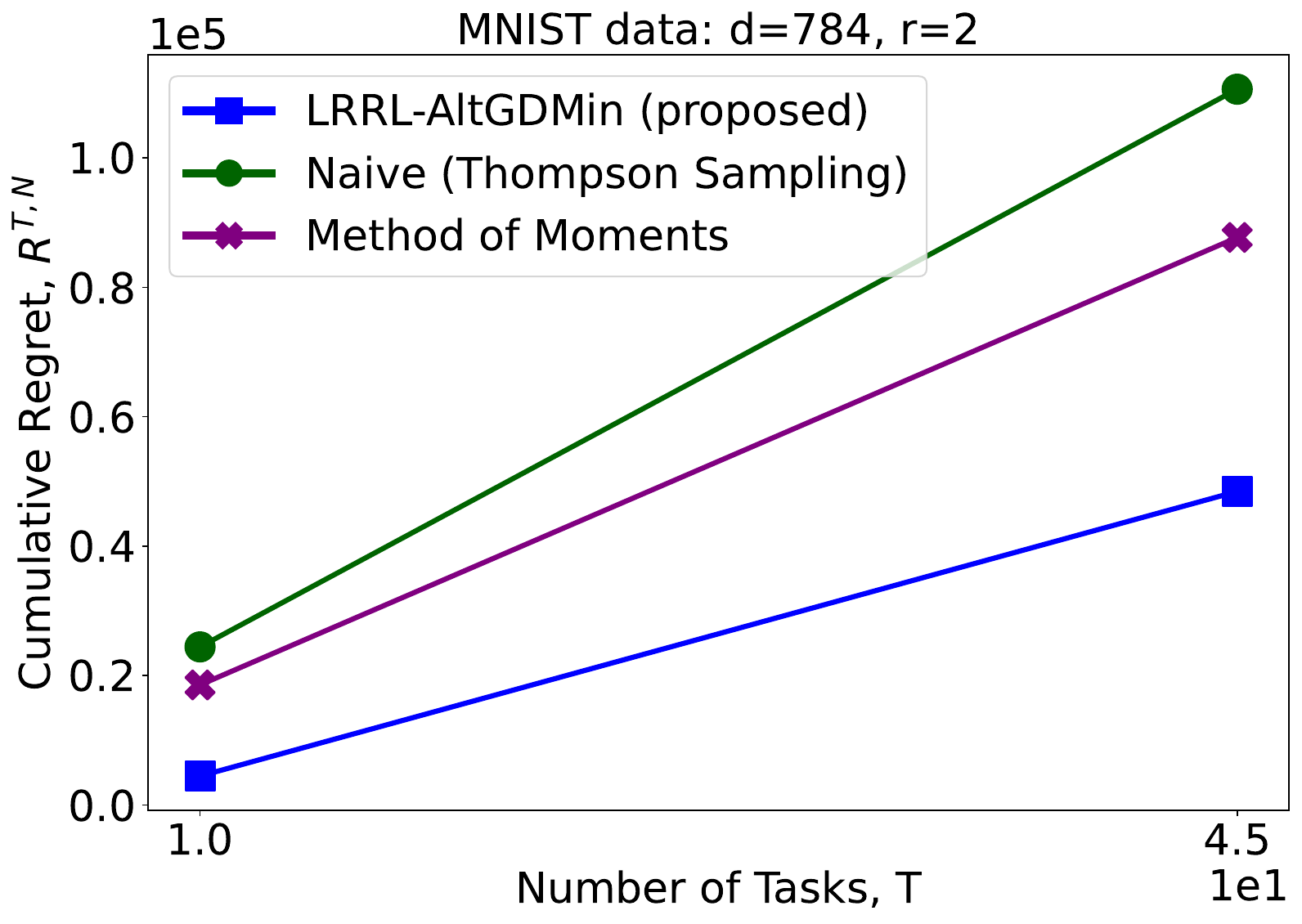}}\hspace{0.8 em}%
\subcaptionbox{\footnotesize MNIST data: rank $r=4$ \label{fig:5}}{\includegraphics[scale=0.2]{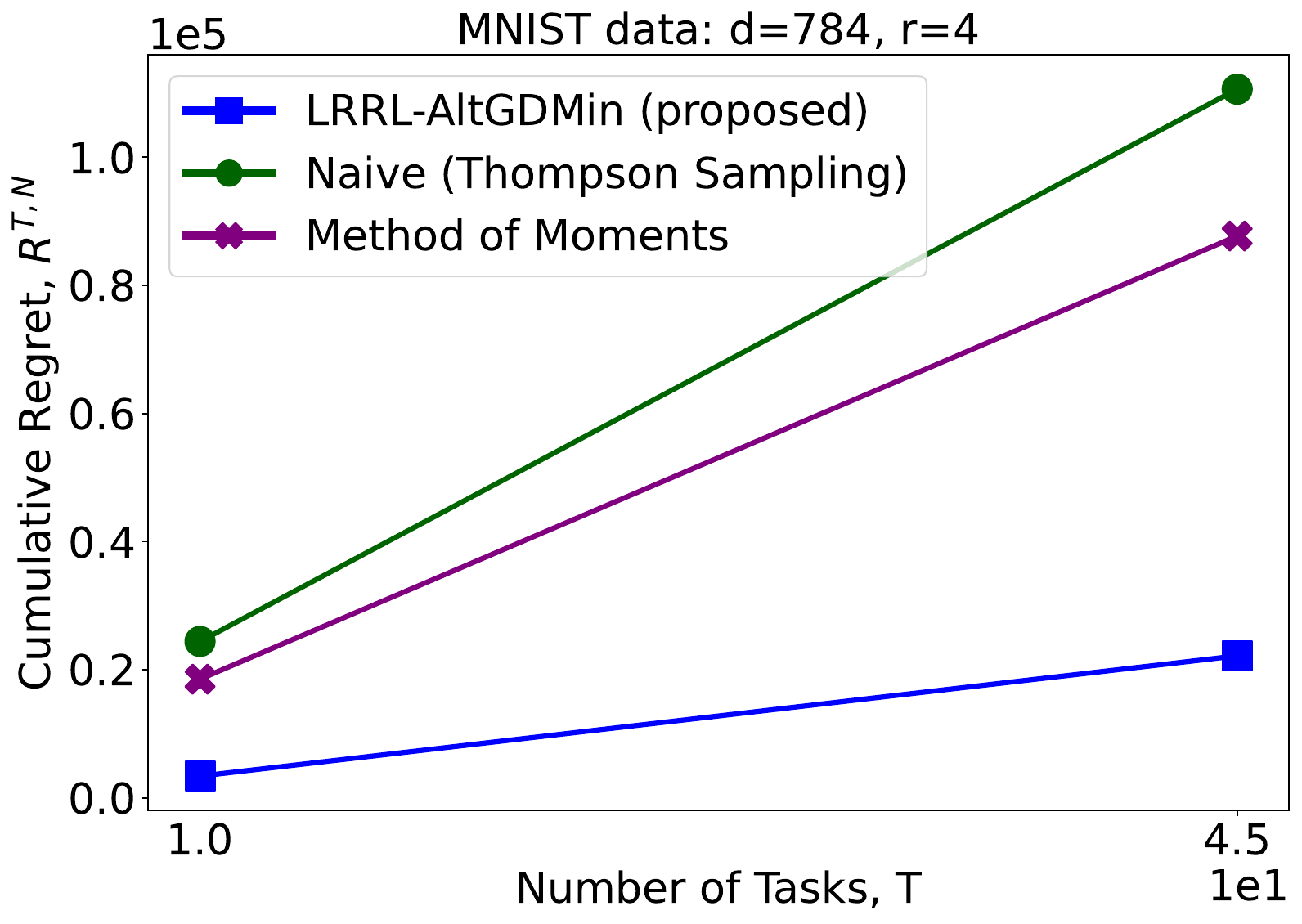}}\hspace{0.8 em}%
\subcaptionbox{\footnotesize MNIST data: rank $r=8$ \label{fig:6}}{\includegraphics[scale=0.2]{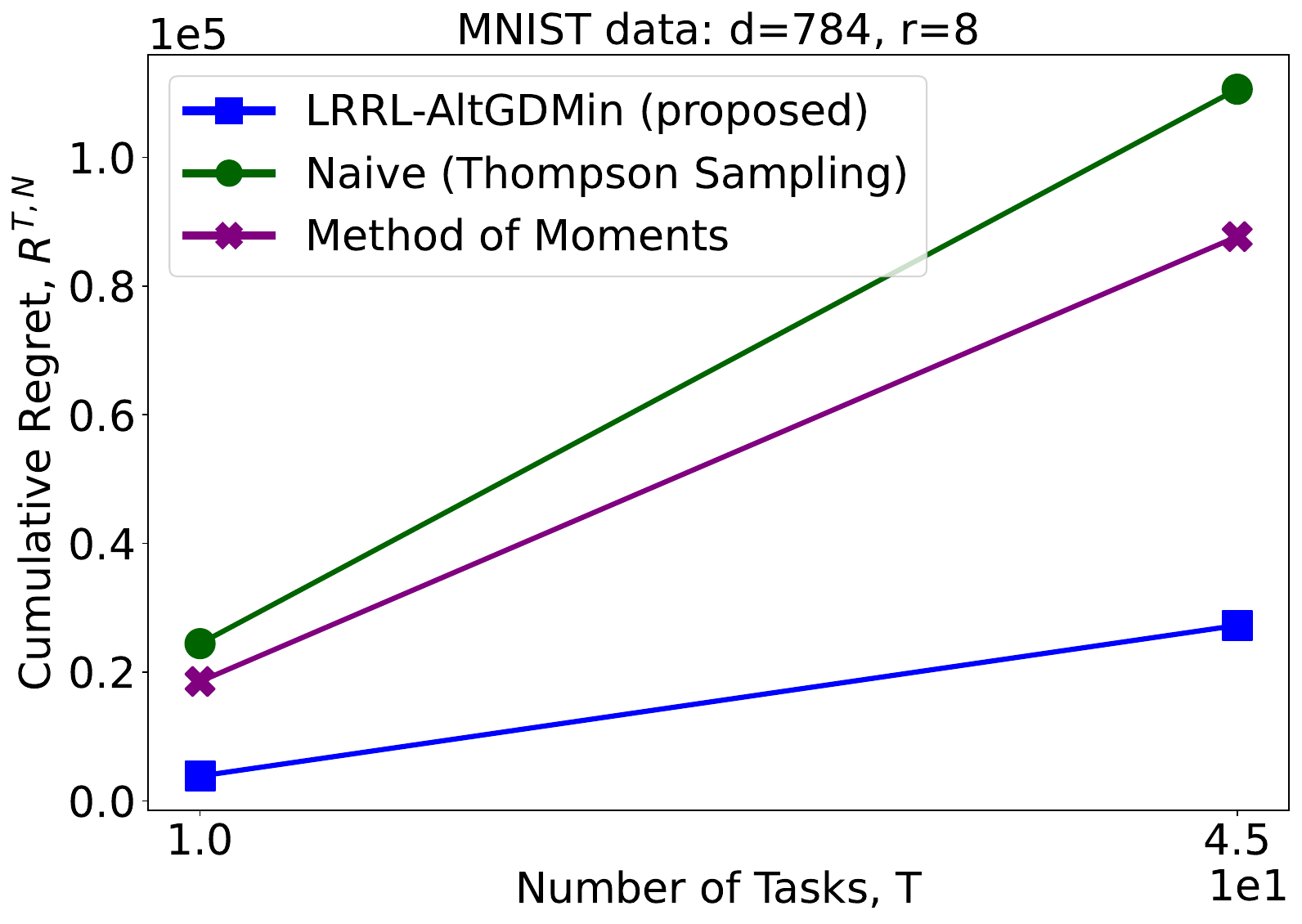}}\hspace{0.8 em}
\vspace{ 2mm}
\subcaptionbox{\footnotesize Synthetic data: rank $r=2$ \label{fig:7}}{\includegraphics[scale=0.2]{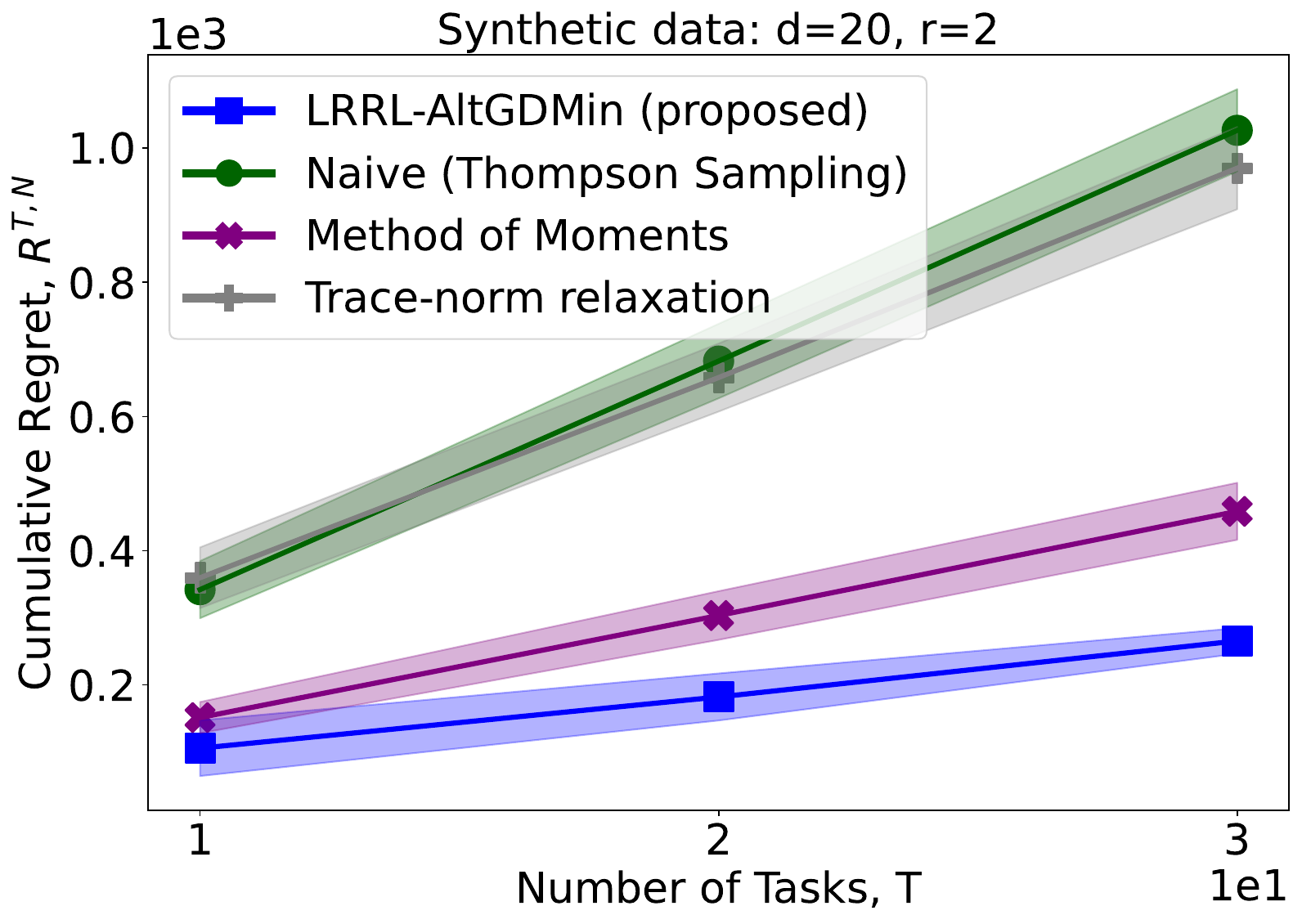}}\hspace{1.0 em}%
\subcaptionbox{\footnotesize Synthetic data: rank $r=3$ \label{fig:8}}{\includegraphics[scale=0.2]{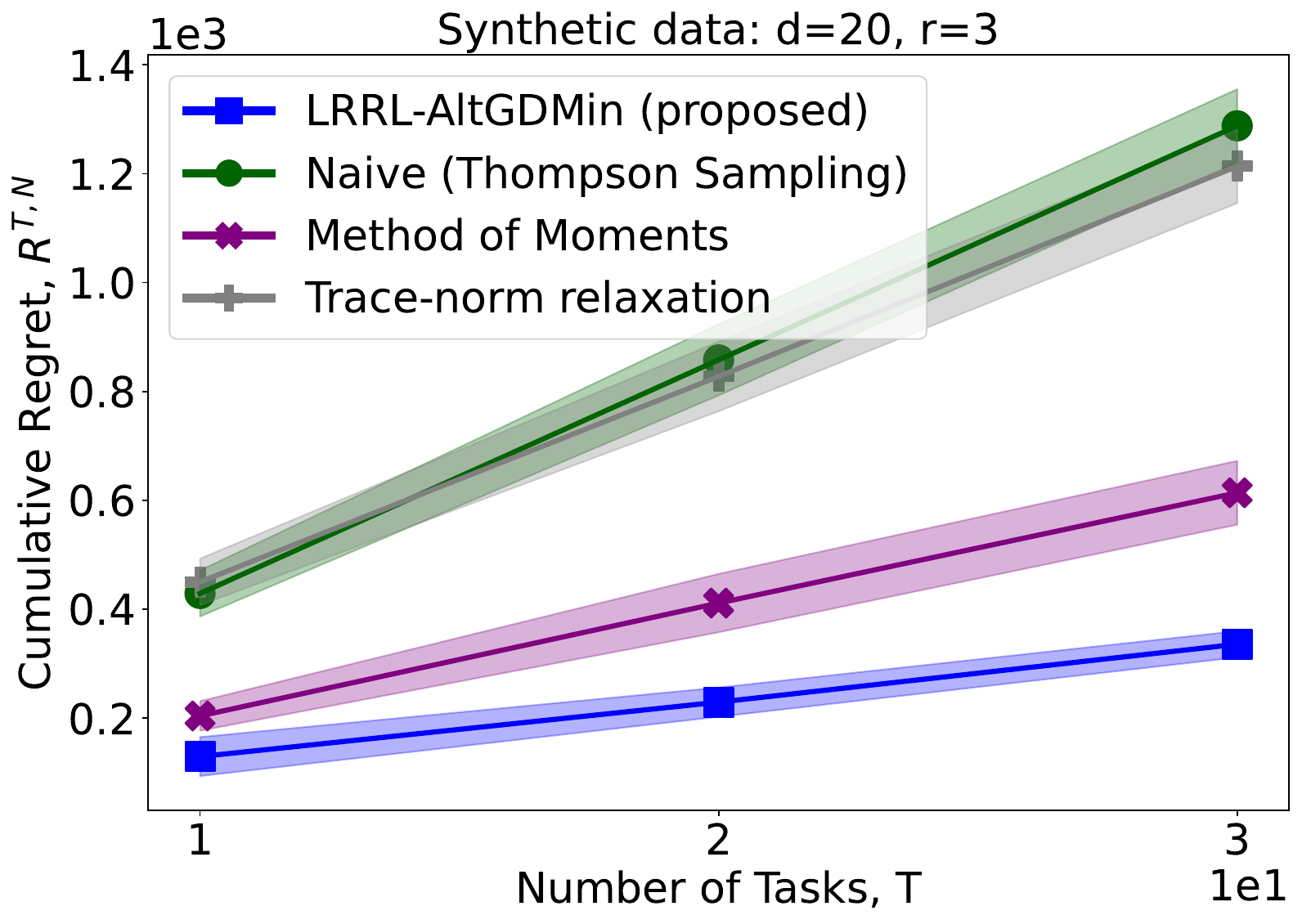}}\hspace{1.0 em}%
\subcaptionbox{\footnotesize Synthetic data: rank $r=4$ \label{fig:9}}{\includegraphics[scale=0.2]{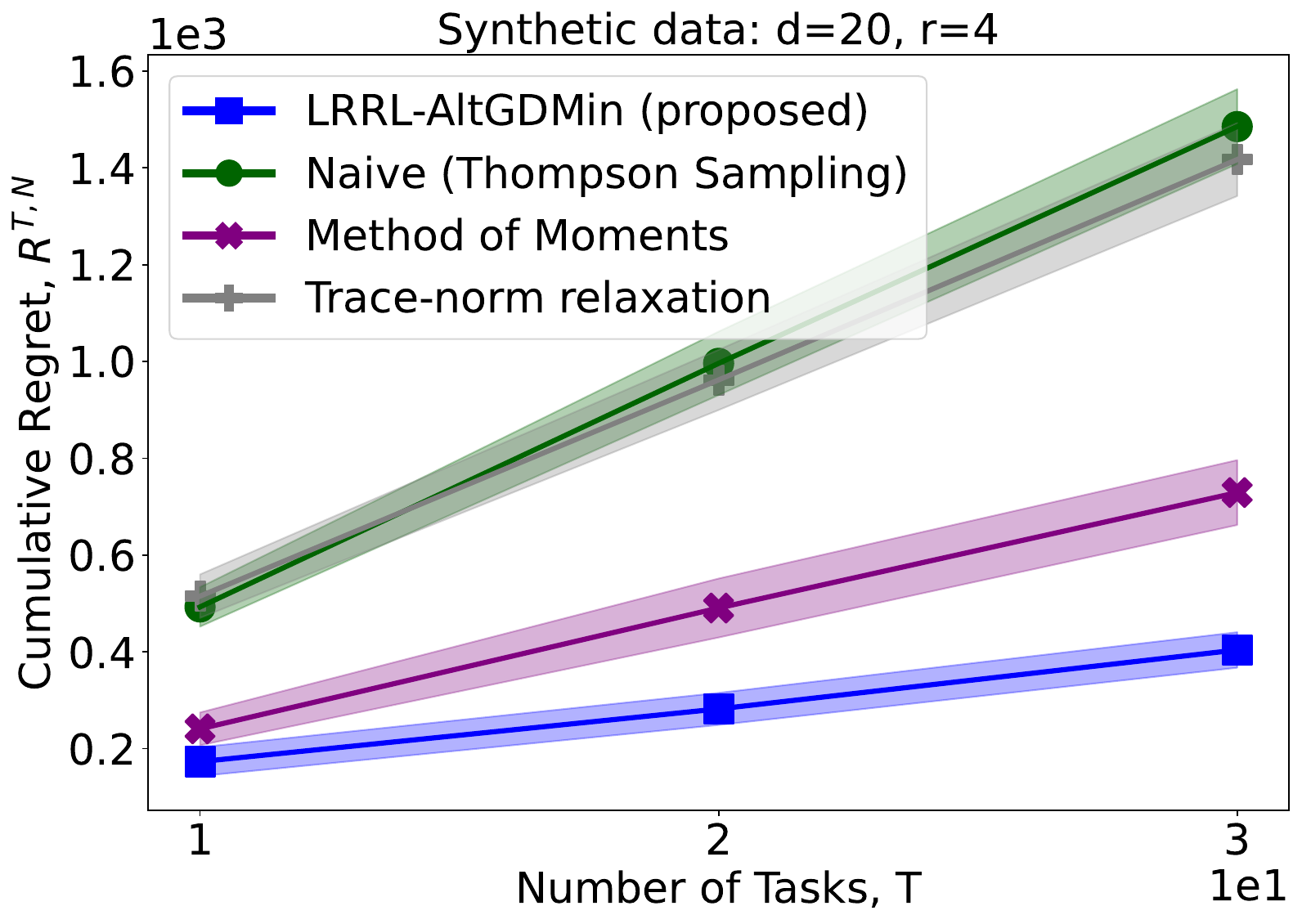}}\hspace{1.0 em}
\vspace{-2 mm}
\caption{\small 
 {\bf Synthetic data 1:} We set the parameters as $d = 100$, $K = 5$, $N=200$, and noise variance $= 10^{-6}$. We considered $M=4$ epochs each with $50$ data samples each. We varied the number of tasks as $T=10,25,50,75,100$. We also varied the rank of the feature matrix as $r=2,4,8$. As shown in the plots (Figures~\ref{fig:1}, \ref{fig:2}, and~\ref{fig:3}), our proposed approach outperforms the existing benchmarks. {\bf MNIST data:} Parameters are $d = 784$, $K = 2$, $N=5000$, and noise variance $= 10^{-6}$. We considered $M=5$ epochs each with $1000$ data samples each. We varied the number of tasks as $T=10, 45$. We also varied the rank of the feature matrix as $r=2,4,8$. The plots for MNIST data are presented in Figures~\ref{fig:4}, \ref{fig:5}, and~\ref{fig:6}.  {\bf Synthetic data 2:} We consider a smaller problem dimension here and also compare with the trace-norm relaxation method. In Figures~\ref{fig:7}, \ref{fig:8}, and~\ref{fig:9}, we set $d = 20$, $K = 5$, $N=40$. We considered $M=4$ epochs each with $10$ data samples each, thus $N=40$.}\label{fig:main1}
\end{figure*}

\noindent{\bf Estimation error.} We compared the estimation performance of our proposed LRRL-AltGDMin estimator with three existing approaches: (i)~an alternating  GD (LRRL-AltGD) estimator, (ii)~Method-of-Moments (MoM) based estimator, and (iii)~trace-norm convex relaxation-based estimator. The LRRL-AltGD is based on the alternating gradient descent algorithm proposed in \cite{yi2016fast} for solving the low-rank matrix completion problem. LRRL-AltGD alternatively solves for $\wB$ and $\wW$ in Eq.~\eqref{eq:cost}. The MoM estimator estimates the matrix $\wB$ using the top-$r$ Singular Value Decomposition (SVD) of $\Thetahat=\frac{1}{NT}\sum_{\nt} y_{\nt}^2 \phi(x_{n, t}, c_n)\phi(x_{n, t}, c_n)^\top$. 
Then, it proceeds to calculate the estimated matrix $\wW$ through the method of least squares estimator in order to determine the values of $\Thetahat$.
 The trace-norm technique relaxes the rank constraint to a trace-norm convex constraint and then iteratively solves for the estimate $\Thetahat$ and the regularizing parameter $\lambda$. 
We initialized the LRRL-AltGD algorithm using our proposed spectral initialization approach (Algorithm~\ref{alg2}). This is because spectral initialization guarantees a good initialization for solving the nonconvex problem.

\begin{figure*}[hbt!]
\centering
\subcaptionbox{\footnotesize Error vs. GD iteration for epoch-1 \label{fig:est1}}{\includegraphics[scale=0.19]{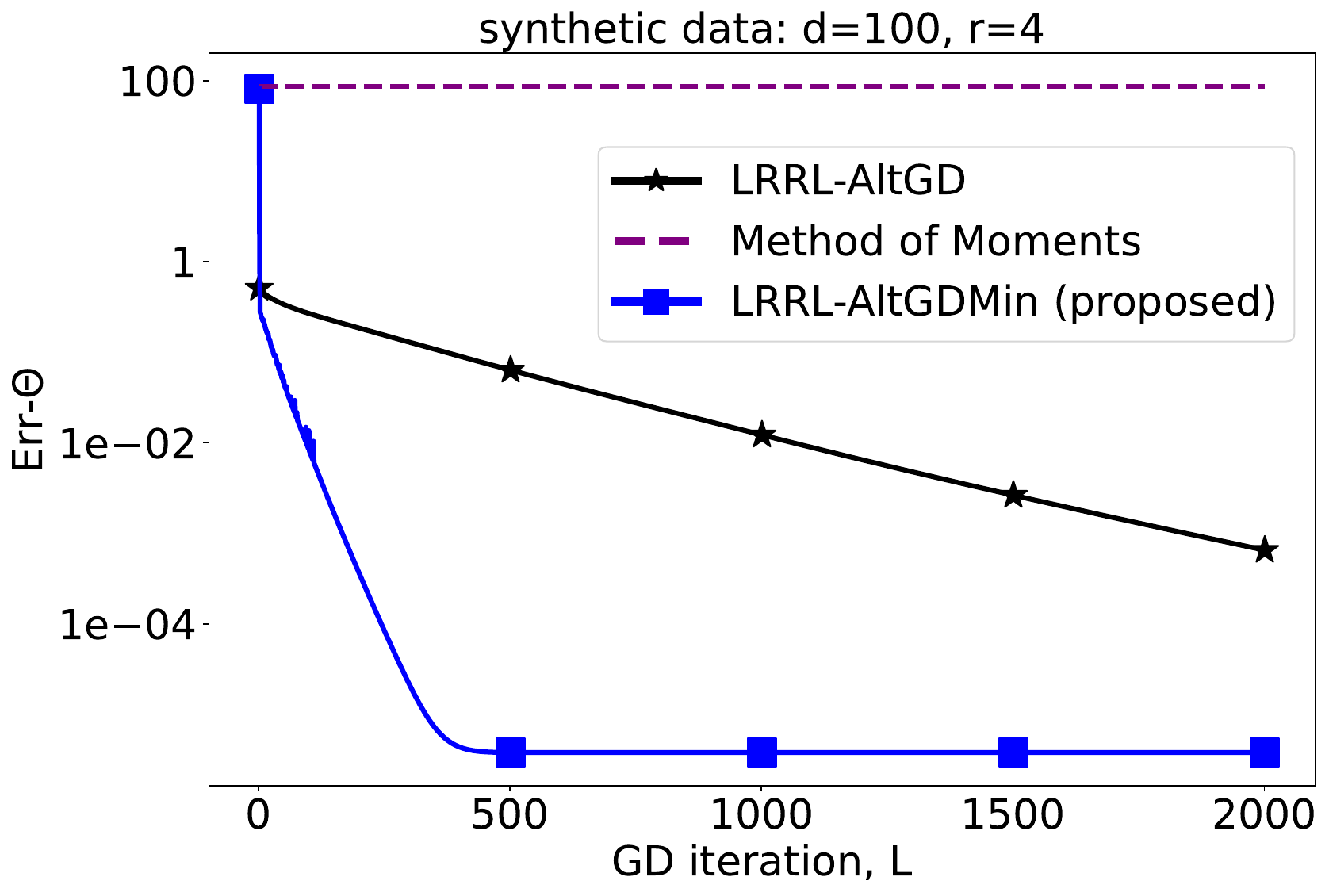}}
%\hspace{2 em}%
\subcaptionbox{\footnotesize Error vs. Epoch \label{fig:est2}}{\includegraphics[scale=0.19]{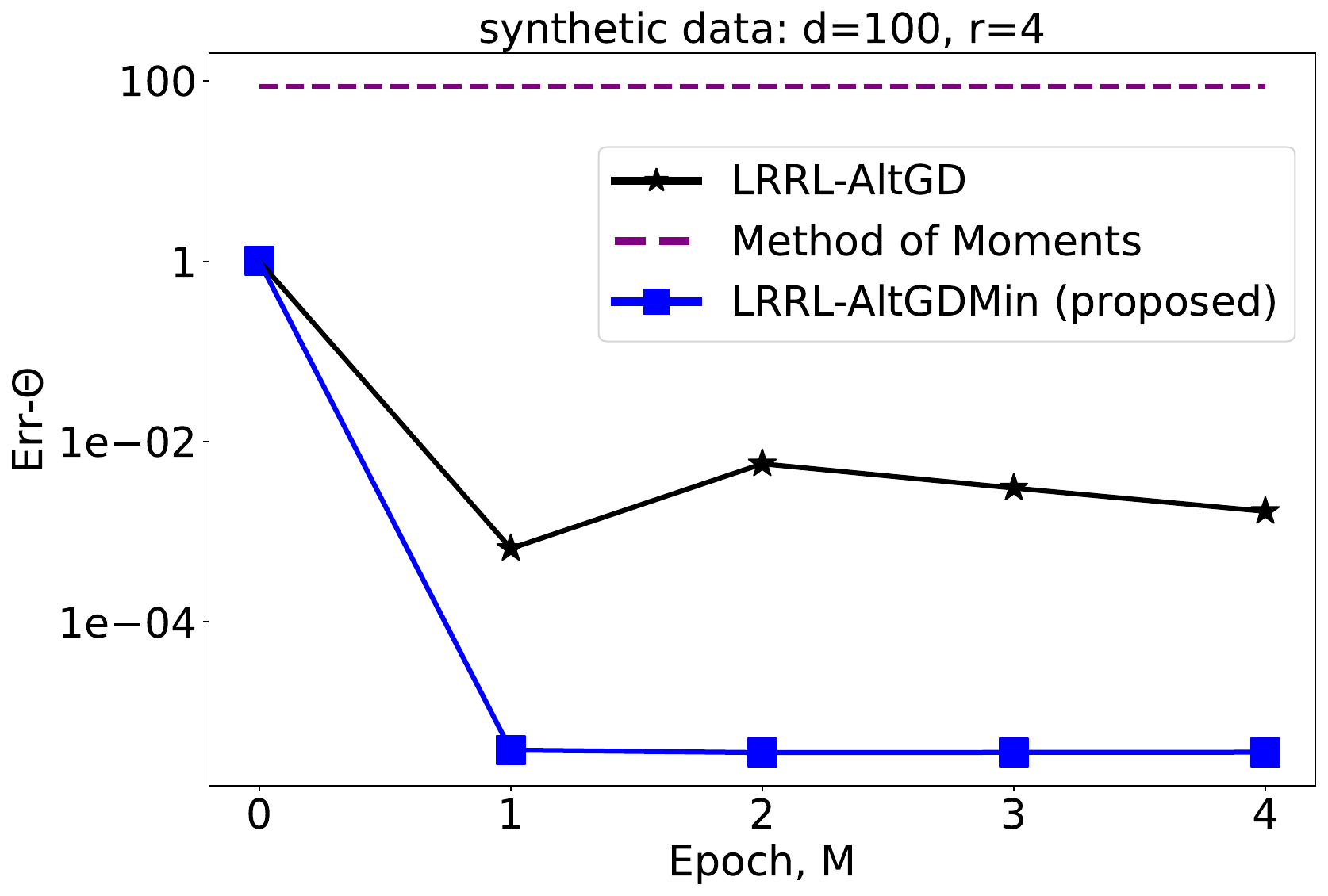}}
%\hspace{1.2 em}%
\subcaptionbox{\footnotesize Error vs. Number of tasks \label{fig:est5}}{\includegraphics[scale=0.19]{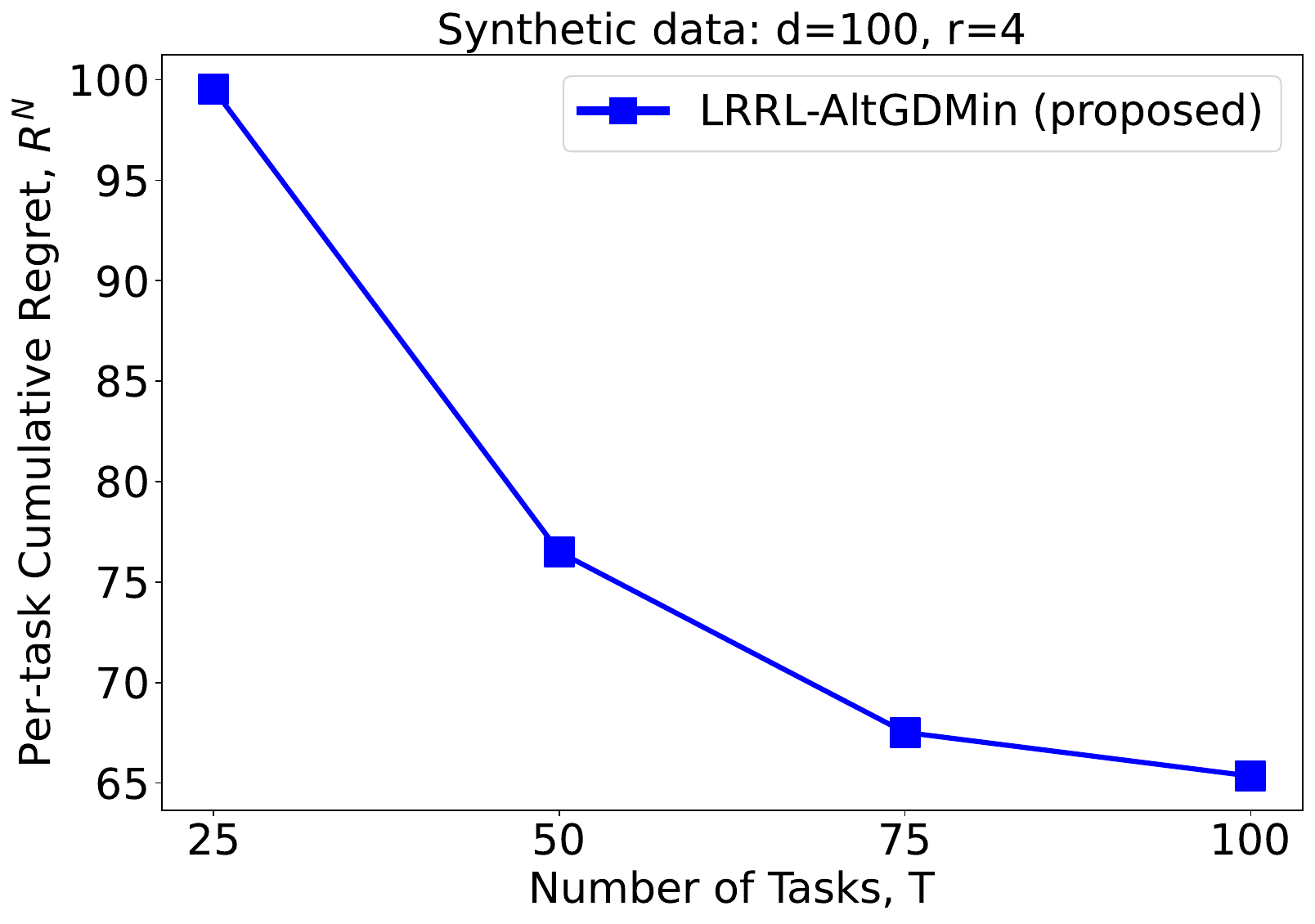}}

\subcaptionbox{\footnotesize Error vs. GD iteration for epoch-1 \label{fig:est4}}{\includegraphics[scale=0.23]{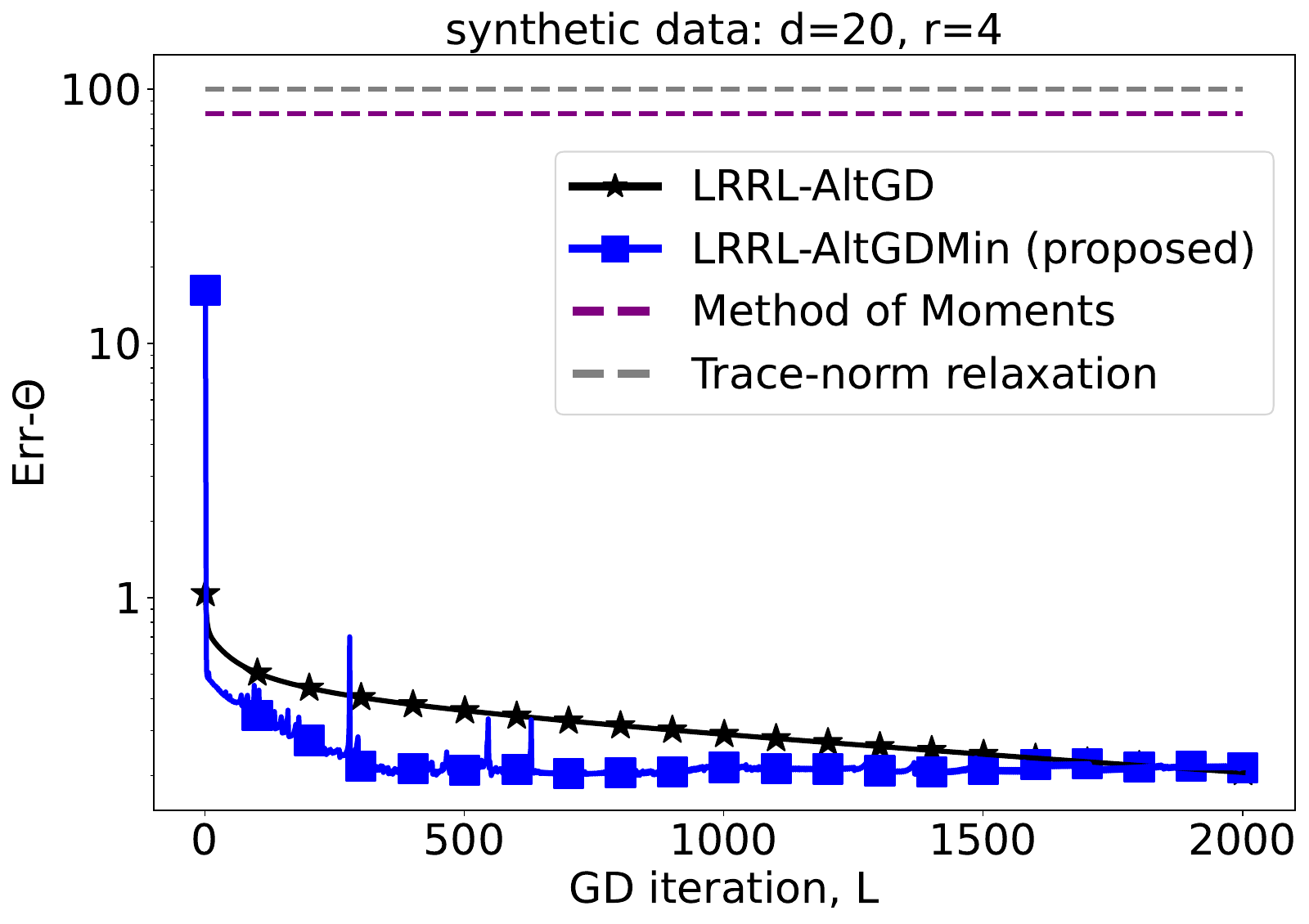}}
\hspace{2 em}%
\subcaptionbox{\footnotesize Error vs. Epoch \label{fig:est3}}{\includegraphics[scale=0.23]{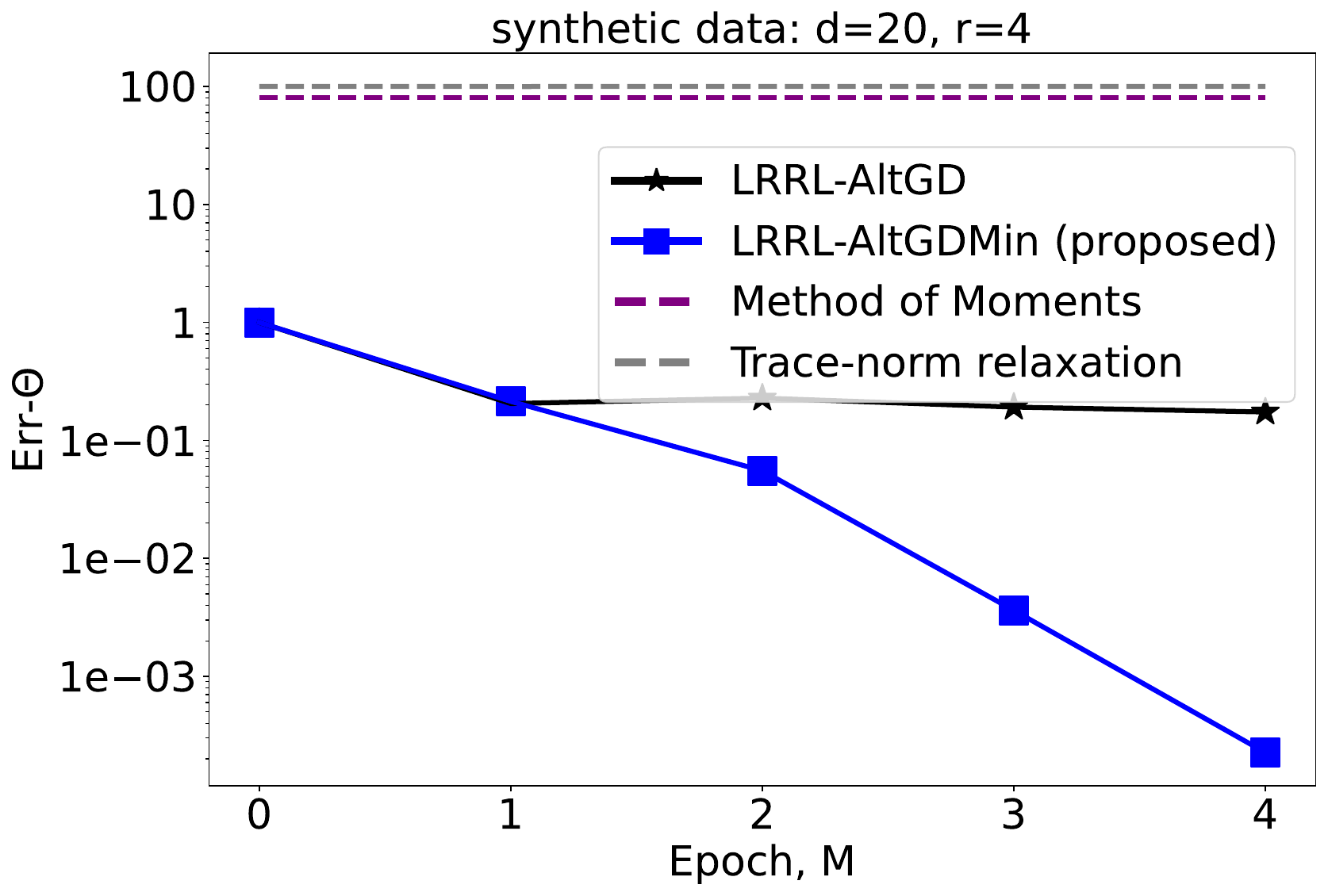}}

\hspace{1.2 em}%
%\hspace{1.2 em}%
\caption{\small 
{\bf Synthetic data 1:} In Figures~\ref{fig:est1} and~\ref{fig:est2}, we set the parameters as $d = 100, T=100$, $K = 5$, $N=200$, and noise variance $= 10^{-6}$. We ran for $L=2000$ GD iterations. We considered $M=4$ epochs each with $50$ data samples each. 
 In Figure~\ref{fig:est5}, we separately present the per-task regret vs. number of task plot for $d=100$, $K = 5$, $N=100$ (also shown in figure~\ref{fig:9}) to showcase the sublinear decay.
 {\bf Synthetic data 2:} We consider a smaller problem dimension and also compare with the trace-norm relaxation method. In Figures~\ref{fig:est3} and~\ref{fig:est4}, we set the parameters as $d = 20, T=30$, $K = 5$, $N=40$, and noise variance $= 10^{-6}$. We ran for $L=2000$ GD iterations. We considered $M=4$ epochs each with $10$ data samples each, thus $N=40$. 
 As expected, the estimation error for our proposed algorithm saturates close to the noise. }\label{fig:main2}
\end{figure*}
We plot the empirical average of $\|\Theta -\Theta^\star\|_F /\|\Theta^\star\|_F$ at each iteration $\ell$ (Err-$\Theta$ in the plots) on the y-axis and the iteration by the algorithm until GD iteration $\ell$ on the x-axis. Averaging is over a 100 trials.
We note that while LRRL-AltGD, trace-norm, and LRRL-AltGDMin are iterative algorithms, the MoM estimator is non-iterative. To showcase the baselines in our plots, we also show the error achieved by the MoM estimator.  Figure~\ref{fig:main2} presents the error plot. Figure~\ref{fig:est1} presents the  Err-$\Theta$ vs. GD iteration for the first epoch. In Figure~\ref{fig:est2}, we present the Err-$\Theta$ vs. epoch. We set a total of $5$ epochs, including the zeroth epoch, which is the initialization step.
From the plots, we notice that the proposed LRRL-AltGDMin estimator outperforms both the benchmark approaches. Further, the estimation error saturates close to $10^{-6}$. This can be explained using our result, Theorem~\ref{new_5}, which shows that the error decays exponentially until it reaches the (normalized) ``noise-level"  ${\sigma_\eta}^2/\|\thetats\|^2$, but saturates after that. Although the error in the trace-norm approach improves as the iteration progresses, the improvement is very minimal.

%{Comparison of LRRL-AltGDMin algorithm with benchmarks}
%In our LRRL-AltGDMin algorithm, after the initialization epoch, each iteration begins with a minimization procedure that utilizes the current estimated matrix $\wB$ to calculate $\wW$. Next, the estimated matrix $\wB$ is updated using the Gradient Descent algorithm. In contrast, the MLinGreedy algorithm utilizes an Alternating Gradient Descent method to update both $\wB$ and $\wW$ separately through Gradient Descent after initialization in each iteration. 

\noindent{\bf Cumulative regret.}
We compared the performance of our proposed algorithm against three benchmarks: the Method-of-Moments (MoM)-based representation learning algorithm for bandits in \cite{yang2020impact, tripuraneni2021provable},  a Thompson Sampling (TS) algorithm that solves the $T$ tasks separately, and the trace-norm relaxation-based approach in \cite{cella2023multi}.
%The MoM begins by initially determining the estimated matrix $\wB$ using the top-$r$ Singular Value Decomposition (SVD) of $\Thetahat=\frac{1}{N}\sum_{\nt} y_{\nt}^2 \phi(x_{n, t}, c_n)\phi(x_{n, t}, c_n)^\top$. Then, it proceeds to calculate the estimated matrix $\wW$ through the method of least squares estimator in order to determine the values of $\Thetahat$. 
As noted, the MoM-based algorithm only estimates the unknown feature matrix $\Thetahat$ in the first epoch.
In subsequent epochs, this $\Thetahat$ is consistently used to choose actions. On the other hand, the naive approach implements the TS method separately to determine the estimate of $\thetats$ for each task $t \in [T]$. 
The approach in \cite{cella2023multi} considered a trace-norm convex relaxation of the original non-convex cost function (Eqs.~(4) and~(11) in \cite{cella2023multi}).
Figure~\ref{fig:main1} presents the cumulative regret plots for the different algorithms. We varied the number of tasks and the rank of the feature matrix and compared the results of our proposed algorithm with the MoM-based, trace-norm relaxation, and TS-based algorithms. 
Our plot demonstrates that as the number of tasks increases, the advantage of the proposed LRRL-AltGDMin algorithm increases compared to the naive approach, the MoM, and the trace-norm relaxation approaches. We varied the rank $r$ and compared the performances. 
%Further, as the rank  $r$ increases (relative to $d$), the advantage of our algorithm with respect to the MoM-based algorithm increases.  The result in \cite{yang2020impact} can explain this: as the rank is larger, the algorithm pays more regret. 
The performance of the TS algorithm is unaltered by varying the rank.
This is expected since the regret of the TS algorithm does not depend on the rank.
In all experiments, our algorithm consistently outperforms the benchmarks, validating its effectiveness.
%, especially in cases where the number of tasks is large.
%This result is consistent with our expectations, as a larger number of tasks implies better collaboration, thus enabling more extensive learning within a limited number of iterations. 

\section{Conclusion and Future Work}\label{sec:conc}
In this work, we introduced an alternating gradient descent and minimization algorithm for multi-task representation learning in linear contextual bandits.  Leveraging this estimator, we developed a bandit algorithm and established its regret bound for low dimensional contextual bandits. Our approach consistently outperformed existing methods in numerical experiments.  Inspired by \cite{hu2021near}, as part of our future work, we plan to extend our algorithm to an upper confidence bound-based approach by computing the confidence interval. 
Further, one of the very interesting future directions is to relax the i.i.d standard Gaussian assumption on the feature vectors.
While this assumption holds for the initial epoch during the random exploration, it becomes restrictive when we perform greedy exploration in subsequent epochs. As part of our future work, we intend to explore methods for relaxing the i.i.d assumption for epochs after the first one. One potential direction is to fix the $B$ estimate and solve only for $W$ after epoch one, similar to few-shot learning and online subspace tracking \cite{lrpr_gdmin_mri_jp}.

%{\cblue \bf{Confidence Interval: } Inspired by \cite{hu2021near}, we can extend our algorithm to develop an upper confidence bound-based approach. Applying Lemma~\ref{new_1} in our paper, for $\widetilde{V}_t = \phimt{t}{m-1} {\phimt{t}{m-1}}^\top + \lambda I$, with probability at least $O(1 - \exp(\log T + r - c \epsilon_3^2 (\pG_m - \pG_{m-1})))$, we get $\sum_{t=1}^T \|\Bm{m-1} \wm{m-1}_t - B^{\star} w_t^{\star}\|_{\widetilde{V}_t}^2 = O(r + r (\pG_{m-1} - \pG_{m-2}))$. Upon comparing our results with Lemma~1 from \cite{hu2021near}, it is obvious that our bound of $O(r + r (\pG_{m-1} - \pG_{m-2}))$ achieves a tighter confidence interval. Hence, the investigation of confidence interval provides an attractive prospect for future work. We plan to do further investigation in our subsequent studies. } 

\section*{Acknowledgements}
We would like to thank the anonymous reviewers for the helpful comments and suggestions that helped improve the paper.
S. Moothedath and N.Vaswani acknowledge funding from NSF Award-2213069.

\section*{Impact Statement}
This paper presents work whose goal is to advance the field of Machine Learning. There are many potential societal consequences of our work, none which we feel must be specifically highlighted here.

\bibliography{Bandits}

\begin{thebibliography}{50}
\providecommand{\natexlab}[1]{#1}
\providecommand{\url}[1]{\texttt{#1}}
\expandafter\ifx\csname urlstyle\endcsname\relax
  \providecommand{\doi}[1]{doi: #1}\else
  \providecommand{\doi}{doi: \begingroup \urlstyle{rm}\Url}\fi

\bibitem[Abbasi-Yadkori et~al.(2011)Abbasi-Yadkori, P{\'a}l, and
  Szepesv{\'a}ri]{abbasi2011improved}
Abbasi-Yadkori, Y., P{\'a}l, D., and Szepesv{\'a}ri, C.
\newblock Improved algorithms for linear stochastic bandits.
\newblock \emph{Advances in Neural Information Processing Systems},
  24:\penalty0 2312--2320, 2011.

\bibitem[Anandkumar et~al.(2011)Anandkumar, Michael, Tang, and
  Swami]{anandkumar2011distributed}
Anandkumar, A., Michael, N., Tang, A.~K., and Swami, A.
\newblock Distributed algorithms for learning and cognitive medium access with
  logarithmic regret.
\newblock \emph{IEEE Journal on Selected Areas in Communications}, 29\penalty0
  (4):\penalty0 731--745, 2011.

\bibitem[Arora et~al.(2020)Arora, Du, Kakade, Luo, and
  Saunshi]{arora2020provable}
Arora, S., Du, S., Kakade, S., Luo, Y., and Saunshi, N.
\newblock Provable representation learning for imitation learning via bi-level
  optimization.
\newblock In \emph{International Conference on Machine Learning}, pp.\
  367--376. PMLR, 2020.

\bibitem[Aziz et~al.(2021)Aziz, Kaufmann, and Riviere]{aziz2021multi}
Aziz, M., Kaufmann, E., and Riviere, M.-K.
\newblock On multi-armed bandit designs for dose-finding clinical trials.
\newblock \emph{The Journal of Machine Learning Research}, 22\penalty0
  (1):\penalty0 686--723, 2021.

\bibitem[Babu et~al.(2023)Babu, Lingala, and Vaswani]{lrpr_gdmin_mri_jp}
Babu, S., Lingala, S.~G., and Vaswani, N.
\newblock Fast low rank compressive sensing for accelerated dynamic {MRI}.
\newblock \emph{IEEE Transactions on Computational Imaging}, 2023.

\bibitem[Baxter(2000)]{baxter2000model}
Baxter, J.
\newblock A model of inductive bias learning.
\newblock \emph{Journal of artificial intelligence research}, 12:\penalty0
  149--198, 2000.

\bibitem[Bengio et~al.(2013)Bengio, Courville, and
  Vincent]{bengio2013representation}
Bengio, Y., Courville, A., and Vincent, P.
\newblock Representation learning: A review and new perspectives.
\newblock \emph{IEEE transactions on pattern analysis and machine
  intelligence}, 35\penalty0 (8):\penalty0 1798--1828, 2013.

\bibitem[Bubeck \& Cesa-Bianchi(2012)Bubeck and Cesa-Bianchi]{bubeck2012regret}
Bubeck, S. and Cesa-Bianchi, N.
\newblock Regret analysis of stochastic and nonstochastic multi-armed bandit
  problems.
\newblock \emph{arXiv preprint arXiv:1204.5721}, 2012.

\bibitem[Calandriello et~al.(2014)Calandriello, Lazaric, and
  Restelli]{calandriello2014sparse}
Calandriello, D., Lazaric, A., and Restelli, M.
\newblock Sparse multi-task reinforcement learning.
\newblock \emph{Advances in neural information processing systems}, 27, 2014.

\bibitem[Candes \& Recht(2008)Candes and Recht]{matcomp_candes}
Candes, E.~J. and Recht, B.
\newblock Exact matrix completion via convex optimization.
\newblock \emph{Found. of Comput. Math}, \penalty0 (9):\penalty0 717--772,
  2008.

\bibitem[Caruana(1997)]{caruana1997multitask}
Caruana, R.
\newblock Multitask learning.
\newblock \emph{Machine learning}, 28:\penalty0 41--75, 1997.

\bibitem[Cella \& Pontil(2021)Cella and Pontil]{cella2021multi}
Cella, L. and Pontil, M.
\newblock Multi-task and meta-learning with sparse linear bandits.
\newblock In \emph{Uncertainty in Artificial Intelligence}, pp.\  1692--1702.
  PMLR, 2021.

\bibitem[Cella et~al.(2023)Cella, Lounici, Pacreau, and Pontil]{cella2023multi}
Cella, L., Lounici, K., Pacreau, G., and Pontil, M.
\newblock Multi-task representation learning with stochastic linear bandits.
\newblock In \emph{International Conference on Artificial Intelligence and
  Statistics}, pp.\  4822--4847. PMLR, 2023.

\bibitem[Chen et~al.(2020)Chen, Chi, Fan, and Ma]{spectral_init_review}
Chen, Y., Chi, Y., Fan, J., and Ma, C.
\newblock Spectral methods for data science: A statistical perspective.
\newblock \emph{arXiv preprint arXiv:2012.08496}, 2020.

\bibitem[Collins et~al.(2021)Collins, Hassani, Mokhtari, and
  Shakkottai]{collins2021exploiting}
Collins, L., Hassani, H., Mokhtari, A., and Shakkottai, S.
\newblock Exploiting shared representations for personalized federated
  learning.
\newblock In \emph{International conference on machine learning}, pp.\
  2089--2099. PMLR, 2021.

\bibitem[D'Eramo et~al.(2024)D'Eramo, Tateo, Bonarini, Restelli, and
  Peters]{d2024sharing}
D'Eramo, C., Tateo, D., Bonarini, A., Restelli, M., and Peters, J.
\newblock Sharing knowledge in multi-task deep reinforcement learning.
\newblock \emph{arXiv preprint arXiv:2401.09561}, 2024.

\bibitem[Deshmukh et~al.(2017)Deshmukh, Dogan, and Scott]{deshmukh2017multi}
Deshmukh, A.~A., Dogan, U., and Scott, C.
\newblock Multi-task learning for contextual bandits.
\newblock \emph{Advances in neural information processing systems}, 30, 2017.

\bibitem[Du et~al.(2020)Du, Hu, Kakade, Lee, and Lei]{du2020few}
Du, S.~S., Hu, W., Kakade, S.~M., Lee, J.~D., and Lei, Q.
\newblock Few-shot learning via learning the representation, provably.
\newblock \emph{arXiv preprint arXiv:2002.09434}, 2020.

\bibitem[Fang \& Tao(2015)Fang and Tao]{fang2015active}
Fang, M. and Tao, D.
\newblock Active multi-task learning via bandits.
\newblock In \emph{Proceedings of the 2015 SIAM International Conference on
  Data Mining}, pp.\  505--513, 2015.

\bibitem[Gao et~al.(2019)Gao, Han, Ren, and Zhou]{gao2019batched}
Gao, Z., Han, Y., Ren, Z., and Zhou, Z.
\newblock Batched multi-armed bandits problem.
\newblock \emph{Advances in Neural Information Processing Systems}, 32, 2019.

\bibitem[Han et~al.(2020)Han, Zhou, Zhou, Blanchet, Glynn, and
  Ye]{han2020sequential}
Han, Y., Zhou, Z., Zhou, Z., Blanchet, J., Glynn, P.~W., and Ye, Y.
\newblock Sequential batch learning in finite-action linear contextual bandits.
\newblock \emph{arXiv:2004.06321}, 2020.

\bibitem[Hao et~al.(2020)Hao, Lattimore, and Wang]{hao2020high}
Hao, B., Lattimore, T., and Wang, M.
\newblock High-dimensional sparse linear bandits.
\newblock \emph{Advances in Neural Information Processing Systems},
  33:\penalty0 10753--10763, 2020.

\bibitem[Hao et~al.(2021)Hao, Lattimore, and Deng]{hao2021information}
Hao, B., Lattimore, T., and Deng, W.
\newblock Information directed sampling for sparse linear bandits.
\newblock \emph{Advances in Neural Information Processing Systems},
  34:\penalty0 16738--16750, 2021.

\bibitem[Hu et~al.(2021)Hu, Chen, Jin, Li, and Wang]{hu2021near}
Hu, J., Chen, X., Jin, C., Li, L., and Wang, L.
\newblock Near-optimal representation learning for linear bandits and linear
  rl.
\newblock In \emph{International Conference on Machine Learning}, pp.\
  4349--4358, 2021.

\bibitem[Jagatap et~al.(2020)Jagatap, Chen, Nayer, Hegde, and
  Vaswani]{TCIgauri}
Jagatap, G., Chen, Z., Nayer, S., Hegde, C., and Vaswani, N.
\newblock Sample efficient fourier ptychography for structured data.
\newblock \emph{IEEE Transactions on Computational Imaging}, 6:\penalty0
  344--357, 2020.

\bibitem[Jun et~al.(2019)Jun, Willett, Wright, and Nowak]{jun2019bilinear}
Jun, K.-S., Willett, R., Wright, S., and Nowak, R.
\newblock Bilinear bandits with low-rank structure.
\newblock In \emph{International Conference on Machine Learning}, pp.\
  3163--3172. PMLR, 2019.

\bibitem[Kveton et~al.(2017)Kveton, Szepesv{\'a}ri, Rao, Wen, Abbasi-Yadkori,
  and Muthukrishnan]{kveton2017stochastic}
Kveton, B., Szepesv{\'a}ri, C., Rao, A., Wen, Z., Abbasi-Yadkori, Y., and
  Muthukrishnan, S.
\newblock Stochastic low-rank bandits.
\newblock \emph{arXiv preprint arXiv:1712.04644}, 2017.

\bibitem[Lale et~al.(2019)Lale, Azizzadenesheli, Anandkumar, and
  Hassibi]{lale2019stochastic}
Lale, S., Azizzadenesheli, K., Anandkumar, A., and Hassibi, B.
\newblock Stochastic linear bandits with hidden low rank structure.
\newblock \emph{arXiv preprint arXiv:1901.09490}, 2019.

\bibitem[Lattimore \& Szepesv{\'a}ri(2020)Lattimore and
  Szepesv{\'a}ri]{lattimore2020bandit}
Lattimore, T. and Szepesv{\'a}ri, C.
\newblock \emph{Bandit algorithms}.
\newblock Cambridge University Press, 2020.

\bibitem[Li et~al.(2010)Li, Chu, Langford, and Schapire]{li2010contextual}
Li, L., Chu, W., Langford, J., and Schapire, R.~E.
\newblock A contextual-bandit approach to personalized news article
  recommendation.
\newblock In \emph{Proceedings of the 19th international conference on World
  wide web}, pp.\  661--670, 2010.

\bibitem[Lin \& Moothedath(2024)Lin and Moothedath]{lin2024distributed}
Lin, J. and Moothedath, S.
\newblock Distributed multi-task learning for stochastic bandits with context
  distribution and stage-wise constraints.
\newblock \emph{arXiv:2401.11563}, 2024.

\bibitem[Lu et~al.(2021)Lu, Meisami, and Tewari]{lu2021low}
Lu, Y., Meisami, A., and Tewari, A.
\newblock Low-rank generalized linear bandit problems.
\newblock In \emph{International Conference on Artificial Intelligence and
  Statistics}, pp.\  460--468, 2021.

\bibitem[Maurer et~al.(2016)Maurer, Pontil, and
  Romera-Paredes]{maurer2016benefit}
Maurer, A., Pontil, M., and Romera-Paredes, B.
\newblock The benefit of multitask representation learning.
\newblock \emph{Journal of Machine Learning Research}, 17\penalty0
  (81):\penalty0 1--32, 2016.

\bibitem[Nayer \& Vaswani(2021)Nayer and Vaswani]{lrpr_best}
Nayer, S. and Vaswani, N.
\newblock Sample-efficient low rank phase retrieval.
\newblock \emph{IEEE Transactions on Infomation Theory}, Dec. 2021.

\bibitem[Nayer \& Vaswani(2023)Nayer and Vaswani]{lrpr_gdmin}
Nayer, S. and Vaswani, N.
\newblock Fast and sample-efficient federated low rank matrix recovery from
  column-wise linear and quadratic projections.
\newblock \emph{IEEE Transactions on Infomation Theory}, 2023.

\bibitem[Nayer et~al.(2019)Nayer, Narayanamurthy, and Vaswani]{lrpr_icml}
Nayer, S., Narayanamurthy, P., and Vaswani, N.
\newblock Phaseless {PCA}: Low-rank matrix recovery from column-wise phaseless
  measurements.
\newblock In \emph{Intnl. Conf. Mach. Learning (ICML)}, 2019.

\bibitem[Nayer et~al.(2020)Nayer, Narayanamurthy, and Vaswani]{lrpr_it}
Nayer, S., Narayanamurthy, P., and Vaswani, N.
\newblock Provable low rank phase retrieval.
\newblock \emph{IEEE Transactions on Infomation Theory}, March 2020.

\bibitem[Parisotto et~al.(2015)Parisotto, Ba, and
  Salakhutdinov]{parisotto2015actor}
Parisotto, E., Ba, J.~L., and Salakhutdinov, R.
\newblock Actor-mimic: Deep multitask and transfer reinforcement learning.
\newblock \emph{arXiv preprint arXiv:1511.06342}, 2015.

\bibitem[Simchi-Levi \& Xu(2019)Simchi-Levi and Xu]{simchi2019phase}
Simchi-Levi, D. and Xu, Y.
\newblock Phase transitions and cyclic phenomena in bandits with switching
  constraints.
\newblock \emph{Advances in Neural Information Processing Systems}, 32, 2019.

\bibitem[Srivastava et~al.(2014)Srivastava, Reverdy, and
  Leonard]{srivastava2014surveillance}
Srivastava, V., Reverdy, P., and Leonard, N.~E.
\newblock Surveillance in an abruptly changing world via multiarmed bandits.
\newblock In \emph{IEEE Conference on Decision and Control (CDC)}, pp.\
  692--697, 2014.

\bibitem[Taylor \& Stone(2009)Taylor and Stone]{taylor2009transfer}
Taylor, M.~E. and Stone, P.
\newblock Transfer learning for reinforcement learning domains: A survey.
\newblock \emph{Journal of Machine Learning Research}, 10\penalty0 (7), 2009.

\bibitem[Thekumparampil et~al.(2021)Thekumparampil, Jain, Netrapalli, and
  Oh]{thekumparampil2021sample}
Thekumparampil, K.~K., Jain, P., Netrapalli, P., and Oh, S.
\newblock Sample efficient linear meta-learning by alternating minimization.
\newblock \emph{arXiv preprint arXiv:2105.08306}, 2021.

\bibitem[Thrun \& Pratt(1998)Thrun and Pratt]{thrun1998learning}
Thrun, S. and Pratt, L.
\newblock Learning to learn: Introduction and overview.
\newblock In \emph{Learning to learn}, pp.\  3--17. Springer, 1998.

\bibitem[Tripuraneni et~al.(2021)Tripuraneni, Jin, and
  Jordan]{tripuraneni2021provable}
Tripuraneni, N., Jin, C., and Jordan, M.
\newblock Provable meta-learning of linear representations.
\newblock In \emph{International Conference on Machine Learning}, pp.\
  10434--10443. PMLR, 2021.

\bibitem[Vaswani(2024)]{lrpr_gdmin_2}
Vaswani, N.
\newblock Efficient federated low rank matrix recovery via alternating gd and
  minimization: A simple proof.
\newblock \emph{IEEE Transactions on Infomation Theory}, 2024.

\bibitem[Vershynin(2018)]{vershynin2018high}
Vershynin, R.
\newblock \emph{High-dimensional probability: An introduction with applications
  in data science}, volume~47.
\newblock Cambridge university press, 2018.

\bibitem[Wang et~al.(2016)Wang, Kolar, and Srerbo]{wang2016distributed}
Wang, J., Kolar, M., and Srerbo, N.
\newblock Distributed multi-task learning.
\newblock In \emph{Artificial intelligence and statistics}, pp.\  751--760,
  2016.

\bibitem[Yang et~al.(2020)Yang, Hu, Lee, and Du]{yang2020impact}
Yang, J., Hu, W., Lee, J.~D., and Du, S.~S.
\newblock Impact of representation learning in linear bandits.
\newblock \emph{arXiv preprint arXiv:2010.06531}, 2020.

\bibitem[Yi et~al.(2016)Yi, Park, Chen, and Caramanis]{yi2016fast}
Yi, X., Park, D., Chen, Y., and Caramanis, C.
\newblock Fast algorithms for robust pca via gradient descent.
\newblock In \emph{Advances in neural information processing systems}, pp.\
  4152--4160, 2016.

\bibitem[Zhang \& Yang(2018)Zhang and Yang]{zhang2018overview}
Zhang, Y. and Yang, Q.
\newblock An overview of multi-task learning.
\newblock \emph{National Science Review}, 5\penalty0 (1):\penalty0 30--43,
  2018.

\end{thebibliography}
\bibliographystyle{icml2024}

\newpage
\appendix
\onecolumn

\section{Preliminaries}
%\doublespacing

\begin{proposition}[Theorem~2.8.1, \cite{vershynin2018high}] \label{proposition1}
Let $X_1, \cdots, X_N$ be independent, mean zero, sub-exponential random variables. Then, for every $g \geqslant 0$, we have
$$
\mathbb{P} \Bigl\{ |\sum_{i=1}^N X_i| \geqslant g \Bigl\} \leqslant 2 \exp \left[-c \min \left( \frac{g^2}{\sum_{i=1}^N \norm{X_i}_{\psi_1}^2}, \frac{g}{\max_i \norm{X_i}_{\psi_1}} \right) \right], 
$$
where $c > 0$ is an absolute constant. 
\end{proposition}

\begin{proposition}[Chernoff bound for Gaussian] \label{proposition2}
Let $X \sim \N(\mu_x, \sigma_x^2)$, then 
$$
\mathbb{P} \Bigl\{X - \mu_x \geqslant g \Bigl\} \leqslant \exp(-\frac{g^2}{2 \sigma_x^2}). 
$$
\end{proposition}

\begin{proposition}[Epsilon-netting for bounding $\max_{z \in S_d, v \in S_r} |z^\top M v|$] \label{proposition_epsilon_net}
For an $d \times r$ matrix $M$ and fixed vectors $z, v$ with $z \in S_d$ and $v \in S_r$, suppose that $|z^\top M v| \leqslant b_0$ with probability at least $1 - p_0$. Consider an $\epsilon_{net}$ net covering $S_d$ and $S_r$, $\bar{S}_d$, $\bar{S}_r$. Then with probability at least $1 - (1 + \frac{2}{\epsilon_{net}})^{d + r} p_0$, 
\begin{itemize}
\item $\max_{z \in \bar{S}_d, v \in \bar{S}_r} |z^\top M v| \leqslant b_0$ and
\item $\max_{z \in S_d, v \in S_r} |z^\top M v| \leqslant \frac{1}{1 - 2 \epsilon_{net} - \epsilon_{net}^2} b_0$. 
\end{itemize}
Using $\epsilon_{net} = \frac{1}{8}$, this implies the following simpler conclusion: with probability at least $1 - 17^{d + r} p_0 = 1 - \exp((\log17)(d + r)) \cdot p_0$, $\max_{z \in S_d, v \in S_r} |z^\top M v| \leqslant 1.4 b_0$. 
\end{proposition}
\begin{proof}
The proof follows that of Lemma~4.4.1 of \cite{vershynin2018high}
\end{proof}

\section{Guarantees for LRRL-AltGDMin Estimator}\label{app_est}
Define
\begin{align*}
G &:= B^\top \Thetas\\
P &:= I - B^\star {B^{\star}}^\top \\
\GradB &:= \nabla_B f(B, W) = \sum_{t=1}^T \phimt{t}{m} (\phimt{t}{m} B w_t - \ymt{t}{m}) w_t^\top \\
&= \sum_{t=1}^T \sum_{n=\pG_{m-1} + 1}^{\pG_m} (y_\nt - \phi(x_{n, t}, c_n)^\top B w_t) \phi(x_{n, t}, c_n) w_t^\top \\
&= \sum_{t=1}^T \sum_{n=\pG_{m-1} + 1}^{\pG_m} \phi(x_{n, t}, c_n) \phi(x_{n, t}, c_n)^\top (\theta_t - \thetats) w_t^\top + \eta_{n, t} \phi(x_{n, t}, c_n) w_t^\top \\
\GradB^\prime &= \sum_{t=1}^T \sum_{n=\pG_{m-1} + 1}^{\pG_m} \phi(x_{n, t}, c_n) \phi(x_{n, t}, c_n)^\top (\thetats - \theta_t) w_t^\top. 
\end{align*}
and $g_t = B^\top \thetats$ for all $t \in [T]$, $\SE(B_1, B_2) = \| (I - B_1 B_1^\top) B_2 \|_F$ as the Subspace Distance (SD) measure for basis matrices $B_1, B_2$. Here, $\GradB$ represents the gradient that includes noise, while $\GradB^\prime$ represents the gradient without noise. 

\begin{proposition} \label{proposition3}
Assume $\SE(B, B^\star) \leqslant \delt$. Then, with probability at least $1 - 2 T \exp(r - c (\pG_m - \pG_{m-1}))$, it holds that 
$$
\|M^{-1}\| \leqslant \frac{1}{0.9 (\pG_m - \pG_{m-1})} \quad \text{and} \quad \|M^{-1} B^\top {\phimt{t}{m}}^\top {\phimt{t}{m}} (I - B B^\top) \thetats\| \leqslant 0.12 \delt \| w_t^\star \|,
$$ 
where $M = B^\top {\phimt{t}{m}}^\top {\phimt{t}{m}} B$. 
\end{proposition}
\begin{proof}
To demonstrate the upper bound of $\|M^{-1}\|$, let's consider a fixed $z \in \pS_r$. We then have
$$
z^\top B^\top {\phimt{t}{m}}^\top {\phimt{t}{m}} B z = \sum_{n=\pG_{m-1} + 1}^{\pG_m} z^\top B^\top \phi(x_{n, t}, c_n) \phi(x_{n, t}, c_n)^\top B z. 
$$
Furthermore, we find that
$$
\bE[\langle B^\top \phi(x_{n, t}, c_n), z \rangle^2] = \bE[z^\top B^\top \phi(x_{n, t}, c_n) \phi(x_{n, t}, c_n)^\top B z] = z^\top B^\top \bE[\phi(x_{n, t}, c_n) \phi(x_{n, t}, c_n)^\top] B z = 1, 
$$
and also
\begin{align*}
\bE[z^\top B^\top \phi(x_{n, t}, c_n)] &= 0 \\
\Var[z^\top B^\top \phi(x_{n, t}, c_n)] &= \bE[z^\top B^\top \phi(x_{n, t}, c_n)]^2 \\
&= \bE[z^\top B^\top \phi(x_{n, t}, c_n) \phi(x_{n, t}, c_n)^\top B z] \\
&= z^\top B^\top \bE[\phi(x_{n, t}, c_n) \phi(x_{n, t}, c_n)^\top] B z \\
&= 1. 
\end{align*}
The summands are independent sub-exponential random variables with norm $K_n \leqslant 1$. We apply the sub-exponential Bernstein inequality stated in Proposition~\ref{proposition1} by setting $g = \epsilon_2 (\pG_m - \pG_{m-1})$. In order to implement this, we show that 
\begin{align*}
\frac{g^2}{\sum_{n=\pG_{m-1} + 1}^{\pG_m} K_n^2} &\geqslant \frac{\epsilon_2^2 (\pG_m - \pG_{m-1})^2}{(\pG_m - \pG_{m-1})} = \epsilon_2^2 (\pG_m - \pG_{m-1}) \\
\frac{g}{\max_n K_n} &\geqslant \frac{\epsilon_2 (\pG_m - \pG_{m-1})}{\max_n 1} = \epsilon_2 (\pG_m - \pG_{m-1})
\end{align*}
Therefore, for a fixed $z \in \pS_r$, with probability at least $1 - \exp(- c \epsilon_2^2 (\pG_m - \pG_{m-1}))$, 
$$
z^\top B^\top {\phimt{t}{m}}^\top {\phimt{t}{m}} B z - (\pG_m - \pG_{m-1}) I \geqslant - \epsilon_2 (\pG_m - \pG_{m-1}). 
$$
Using epsilon-net over all $z \in \pS_r$ adds a factor of $\exp(r)$. Thus, with probability at least $1 - \exp(r - c \epsilon_2^2 (\pG_m - \pG_{m-1}))$, we have $\min_{z \in \pS_r} \sum_{n=\pG_{m-1}}^{\pG_m} z^\top B^\top \phi(x_{n, t}, c_n) \phi(x_{n, t}, c_n)^\top B z \geqslant (1 - \epsilon_2) (\pG_m - \pG_{m-1})$. Setting $\epsilon_2 = 0.1$, we obtain
\begin{align*}
\|M^{-1}\| &= \| (B^\top {\phimt{t}{m}}^\top {\phimt{t}{m}} B)^{-1} \| \\
&= \frac{1}{\sigma_{\min}(B^\top {\phimt{t}{m}}^\top {\phimt{t}{m}} B)} \\
&= \frac{1}{\min_{z \in \pS_r} \sum_{n=\pG_{m-1} + 1}^{\pG_m} \langle B^\top \phi(x_{n, t}, c_n), z \rangle^2} \\
&\leqslant \frac{1}{0.9 (\pG_m - \pG_{m-1})}. 
\end{align*}
To demonstrate the upper bound of $\|M^{-1} B^\top {\phimt{t}{m}}^\top {\phimt{t}{m}} (I - B B^\top) \thetats\|$, it is necessary to first determine the upper bound of $\|B^\top {\phimt{t}{m}}^\top {\phimt{t}{m}} (I - B B^\top) \thetats\|$. Consider a fixed $z \in \pS_r$, we have 
$$
z^\top B^\top {\phimt{t}{m}}^\top {\phimt{t}{m}} (I - B B^\top) \thetats = \sum_{n = \pG_{m-1} + 1}^{\pG_m} (\phi(x_{n, t}, c_n)^\top B z)^\top (\phi(x_{n, t}, c_n)^\top (I - B B^\top) \thetats). 
$$
Furthermore, we find that
\begin{align*}
\bE[(\phi(x_{n, t}, c_n)^\top B z)^\top (\phi(x_{n, t}, c_n)^\top (I - B B^\top) \thetats)] &= z^\top B^\top \bE [\phi(x_{n, t}, c_n) \phi(x_{n, t}, c_n)^\top] (I - B B^\top) \thetats \\
&= z^\top B^\top (I - B B^\top) \thetats \\
&= 0, 
\end{align*}
and also we have
\begin{align*}
\bE[(\phi(x_{n, t}, c_n)^\top B z)^\top] &= 0 \\
\Var((\phi(x_{n, t}, c_n)^\top B z)^\top) &= \bE[(\phi(x_{n, t}, c_n)^\top B z)^\top]^2 \\
&= \bE[z^\top B^\top \phi(x_{n, t}, c_n) \phi(x_{n, t}, c_n)^\top B z] \\
&= z^\top B^\top \bE[\phi(x_{n, t}, c_n) \phi(x_{n, t}, c_n)^\top] B z \\
&= 1
\end{align*}
and 
\begin{align*}
\bE[\phi(x_{n, t}, c_n)^\top (I - B B^\top) \thetats] &= 0 \\
\Var(\phi(x_{n, t}, c_n)^\top (I - B B^\top) \thetats) &= \bE[\phi(x_{n, t}, c_n)^\top (I - B B^\top) \thetats]^2 \\
&= \bE[{\thetats}^\top (I - B B^\top)^\top \phi(x_{n, t}, c_n) \phi(x_{n, t}, c_n)^\top (I - B B^\top) \thetats] \\
&= {\thetats}^\top (I - B B^\top)^\top \bE[\phi(x_{n, t}, c_n) \phi(x_{n, t}, c_n)^\top] (I - B B^\top) \thetats \\
&= \|(I - B B^\top) \thetats\|^2
\end{align*}
The summands are independent sub-exponential random variables with norm $K_n \leqslant \|(I - B B^\top) \thetats\|$. We apply the sub-exponential Bernstein inequality stated in Proposition~\ref{proposition1} by setting $g = \epsilon_3 (\pG_m - \pG_{m-1}) \|(I - B B^\top) \thetats\|$. In order to implement this, we show that 
\begin{align*}
\frac{g^2}{\sum_{n=\pG_{m-1} + 1}^{\pG_m} K_n^2} &\geqslant \frac{\epsilon_3^2 (\pG_m - \pG_{m-1})^2 \|(I - B B^\top) \thetats\|^2}{(\pG_m - \pG_{m-1}) \|(I - B B^\top) \thetats\|^2} = \epsilon_3^2 (\pG_m - \pG_{m-1}) \\
\frac{g}{\max_n K_n} &\geqslant \frac{\epsilon_3 (\pG_m - \pG_{m-1}) \|(I - B B^\top) \thetats\|}{\max_n \|(I - B B^\top) \thetats\|} = \epsilon_3 (\pG_m - \pG_{m-1})
\end{align*}
Therefore, for a fixed $z \in \pS_r$, with probability at least $1 - \exp(- c \epsilon_3^2 (\pG_m - \pG_{m-1}))$, 
$$
z^\top B^\top {\phimt{t}{m}}^\top {\phimt{t}{m}} B z \leqslant \epsilon_3 (\pG_m - \pG_{m-1}) \|(I - B B^\top) \thetats\|. 
$$
Using epsilon-net over all $z \in \pS_r$ adds a factor of $\exp(r)$. Thus, with probability at least $1 - \exp(r - c \epsilon_3^2 (\pG_m - \pG_{m-1}))$, we have $\max_{z \in \pS_r} \sum_{n = \pG_{m-1} + 1}^{\pG_m} (\phi(x_{n, t}, c_n)^\top B z)^\top (\phi(x_{n, t}, c_n)^\top (I - B B^\top) \thetats) \leqslant \epsilon_3 (\pG_m - \pG_{m-1}) \|(I - B B^\top) \thetats\|$. Therefore, we have 
\begin{align*}
\|B^\top {\phimt{t}{m}}^\top {\phimt{t}{m}} (I - B B^\top) \thetats\| &= \max_{z \in \pS_r} z^\top B^\top \phi(x_{n, t}, c_n)^\top \phi(x_{n, t}, c_n) (I - B B^\top) \thetats \\
&= \max_{z \in \pS_r} \sum_{n = \pG_{m-1} + 1}^{\pG_m} (\phi(x_{n, t}, c_n)^\top B z)^\top (\phi(x_{n, t}, c_n)^\top (I - B B^\top) \thetats) \\
&\leqslant \epsilon_3 (\pG_m - \pG_{m-1}) \|(I - B B^\top) \thetats\|
\end{align*}
By combining these results, using a union bound over all $T$ vectors, and setting $\epsilon_3 = 0.1$, we conclude that with probability at least $1 - 2 T \exp(r - c (\pG_m - \pG_{m-1}))$, 
\begin{align*}
\|M^{-1} B^\top {\phimt{t}{m}}^\top {\phimt{t}{m}} (I - B B^\top) \thetats\| &\leqslant \|M^{-1}\| \times \|B^\top {\phimt{t}{m}}^\top {\phimt{t}{m}} (I - B B^\top) \thetats\| \\
&\leqslant \frac{1}{0.9 (\pG_m - \pG_{m-1})} \epsilon_3 (\pG_m - \pG_{m-1}) \|(I - B B^\top) \thetats\| \\
&\leqslant 0.12 \|(I - B B^\top) \thetats\| \\
&\leqslant 0.12 \delt \| w_t^\star \|.
\end{align*}
\end{proof}

\begin{lemma} \label{new_1}
Assume $\sigma_\eta^2 \leqslant \frac{r}{T} \delt^2 {\sigma_{\max}^\star}^2$, and $\SE(B, B^\star) \leqslant \delt$, if $\delt \leqslant \frac{0.02}{\mu \sqrt{r} \kappa^2}$, and if $(\pG_m - \pG_{m-1}) \geqslant C \max(\log T, \log d, r)$, then with probability at least $1 - 3 \exp(\log T + r - c (\pG_m - \pG_{m-1}))$, the following bounds hold: 
\begin{enumerate}
    \item $\| w_t - g_t \| \leqslant 0.24 \mu \delt \sqrt{\frac{r}{T}} \sigma_{\max}^\star$
    \item $\| w_t \| \leqslant 1.24 \mu \sqrt{\frac{r}{T}} \sigma_{\max}^\star$
    \item $\| W - G \|_F \leqslant 0.24 \mu \delt \sqrt{r} \sigma_{\max}^\star$
    \item $\| \theta_t - \thetats \| \leqslant 1.24 \mu \delt \sqrt{\frac{r}{T}} \sigma_{\max}^\star$
    \item $\| \Theta_l - \Thetas\|_F \leqslant 1.24 \mu \delt \sqrt{r} \sigma_{\max}^\star$
    \item $\sigma_{\min}(W) \geqslant 0.9 \sigma_{\min}^\star$
    \item $\sigma_{\max}(W) \leqslant 1.1 {\sigma_{\max}^\star}$
\end{enumerate}
\end{lemma}
\begin{proof}
Consider the expression for $w_t$, we obtain that
\begin{align*}
w_t &= (\phimt{t}{m} B)^\dagger \ymt{t}{m} \\
&= ((\phimt{t}{m} B)^\top (\phimt{t}{m} B))^{-1} (\phimt{t}{m} B)^\top \ymt{t}{m} \\
&= (B^\top {\phimt{t}{m}}^\top \phimt{t}{m} B)^{-1} B^\top {\phimt{t}{m}}^\top \ymt{t}{m} \\
&= (B^\top {\phimt{t}{m}}^\top \phimt{t}{m} B)^{-1} (B^\top {\phimt{t}{m}}^\top) \phimt{t}{m} B B^\top \thetats + (B^\top {\phimt{t}{m}}^\top \phimt{t}{m} B)^{-1} (B^\top {\phimt{t}{m}}^\top) \phimt{t}{m} (I - B B^\top) \thetats \\
&+ (B^\top {\phimt{t}{m}}^\top \phimt{t}{m} B)^{-1} (B^\top {\phimt{t}{m}}^\top) \etamt{t}{m} \\
&= (B^\top {\phimt{t}{m}}^\top \phimt{t}{m} B)^{-1} (B^\top {\phimt{t}{m}}^\top \phimt{t}{m} B) B^\top \thetats + (B^\top {\phimt{t}{m}}^\top \phimt{t}{m} B)^{-1} (B^\top {\phimt{t}{m}}^\top) \phimt{t}{m} (I - B B^\top) \thetats \\
&+ (B^\top {\phimt{t}{m}}^\top \phimt{t}{m} B)^{-1} (B^\top {\phimt{t}{m}}^\top) \etamt{t}{m} \\
&= g_t + M^{-1} B^\top {\phimt{t}{m}}^\top \phimt{t}{m} (I - B B^\top) \thetats + M^{-1} B^\top {\phimt{t}{m}}^\top \etamt{t}{m}, 
\end{align*}
where $M = B^\top {\phimt{t}{m}}^\top \phimt{t}{m} B$. Consequently, $w_t - g_t = M^{-1} B^\top {\phimt{t}{m}}^\top \phimt{t}{m} (I - B B^\top) \thetats + M^{-1} B^\top {\phimt{t}{m}}^\top \etamt{t}{m}$. The first term is bounded in Proposition~\ref{proposition3}. To bound the second term, let's consider a fixed $z \in \pS_r$. We analyze $z^\top B^\top {\phimt{t}{m}}^\top \etamt{t}{m} = \sum_{n = \pG_{m-1} + 1}^{\pG_m} (B z)^\top \phi(x_{n, t}, c_n) \eta_{n, t}$, leading to $\bE[(B z)^\top \phi(x_{n, t}, c_n) \eta_{n, t}] = 0$ and
\begin{align*}
\Var((B z)^\top \phi(x_{n, t}, c_n)) &= \bE[(B z)^\top \phi(x_{n, t}, c_n)]^2 - (\bE[(B z)^\top \phi(x_{n, t}, c_n)])^2 \\
&= \bE[(B z)^\top \phi(x_{n, t}, c_n)]^2 \\
&= \bE[z^\top B^\top \phi(x_{n, t}, c_n) \phi(x_{n, t}, c_n)^\top B z] \\
&= z^\top B^\top \bE[\phi(x_{n, t}, c_n) \phi(x_{n, t}, c_n)^\top] B z \\
&= z^\top B^\top B z \\
&= I. 
\end{align*}
Given $\eta_{n, t} \iidsim \n(0, \sigma_\eta^2)$, we have $\Var(\eta_{n, t}) = \sigma_\eta^2$. Thus, $z^\top B^\top {\phimt{t}{m}}^\top \etamt{t}{m}$ is a sum of $(\pG_m - \pG_{m-1})$ subexponential random variables with parameter $K_n = \sigma_\eta$. We apply the sub-exponential Bernstein inequality stated in Proposition~\ref{proposition1} by setting $g = \epsilon_3 (\pG_m - \pG_{m-1}) \sigma_\eta$. In order to implement this, we show that
\begin{align*}
\frac{g^2}{\sum_{n=\pG_{m-1} + 1}^{\pG_m} K_n^2} &\geqslant \frac{\epsilon_3^2 (\pG_m - \pG_{m-1})^2 \sigma_\eta^2}{(\pG_m - \pG_{m-1}) \sigma_\eta^2} = \epsilon_3^2 (\pG_m - \pG_{m-1}) \\
\frac{g}{\max_n K_n} &\geqslant \frac{\epsilon_3 (\pG_m - \pG_{m-1}) \sigma_\eta}{\sigma_\eta} = \epsilon_3 (\pG_m - \pG_{m-1})
\end{align*}
Therefore, for a fixed $z \in \pS_r$, with probability at least $1 - \exp(- c \epsilon_3^2 (\pG_m - \pG_{m-1}))$, $z^\top B^\top {\phimt{t}{m}}^\top \etamt{t}{m} \leqslant \epsilon_3 (\pG_m - \pG_{m-1}) \sigma_\eta$. Using epsilon-net over all $z$ adds a factor of $\exp(r)$. Thus, with probability at least $1 - \exp(r - c \epsilon_3^2 (\pG_m - \pG_{m-1}))$, we have $B^\top {\phimt{t}{m}}^\top \etamt{t}{m} \leqslant \epsilon_3 (\pG_m - \pG_{m-1}) \sigma_\eta$. Then, the above holds for all $t \in [T]$ with probability at least $1 - \exp(\log T + r - c \epsilon_3^2 (\pG_m - \pG_{m-1}))$. According to Proposition~\ref{proposition3}, with probability at least $1 - 2 T \exp(r - c (\pG_m - \pG_{m-1}))$, we have $\| M^{-1} \| \leqslant \frac{1}{0.9 (\pG_m - \pG_{m-1})}$. Combining these results, it follows that with probability at least $1 - 2 T \exp(r - c (\pG_m - \pG_{m-1}))$, $\| M^{-1} B^\top {\phimt{t}{m}}^\top \etamt{t}{m} \| \leqslant \frac{\sigma_\eta}{9}$. 
Combining with bound on the first term, we then determine that with probability at least $1 - 3 \exp(\log T + r - c (\pG_m - \pG_{m-1}))$, 
$$
\| w_t - g_t \| \leqslant 0.12 \delt \| w_t^\star \| + \frac{\sigma_\eta}{9}. 
$$
Given that $\sigma_\eta \leqslant \sqrt{\frac{r}{T}} \delt \sigma_{\max}^\star$, we then have
$$
\| w_t - g_t \| \leqslant 0.24 \delt \max(\| w_t^\star \|, \sqrt{\frac{r}{T}} \sigma_{\max}^\star). 
$$
Applying the Incoherence of right singular vectors in Assumption~\ref{assume:incoherence}, we have
\begin{align}\label{Eq:BLemma-1}
\| w_t - g_t \| &\leqslant 0.24 \delt \max(\mu \sqrt{\frac{r}{T}} {\sigma_{\max}^\star}, \sqrt{\frac{r}{T}} \sigma_{\max}^\star) \leqslant 0.24 \mu \delt \sqrt{\frac{r}{T}} \sigma_{\max}^\star. 
\end{align}
This proves 1).

Eq.~\eqref{Eq:BLemma-1} implies
$$
\| W - G \|_F \leqslant 0.24 \mu \delt \sqrt{r} \sigma_{\max}^\star.
$$
This completes the proof of 3).

To bound $\| w_t \|$, we use $\| g_t \| \leqslant \| w_t^\star \|$, and then find
\begin{align*}
\| w_t \| &= \| w_t - g_t + g_t \| \\
&\leqslant \| w_t - g_t \| + \| w_t^\star \| \\
&\leqslant 0.24 \mu \delt \sqrt{\frac{r}{T}} \sigma_{\max}^\star + \mu \sqrt{\frac{r}{T}} \sigma_{\max}^\star \\
&\leqslant 1.24 \mu \sqrt{\frac{r}{T}} \sigma_{\max}^\star. 
\end{align*}
This completes the proof of 2).

For $\| \theta_t - \thetats \|$, we derive
\begin{align*}
\| \theta_t - \thetats \| &= \| B g_t + (I - B B^\top) \thetats - B w_t \| \\
&= \| B (g_t - w_t) + (I - B B^\top) \thetats \| \\
&\leqslant \| g_t - w_t \| + \| (I - B B^\top) B^\star w_t^\star \| \\
&\leqslant 0.24 \mu \delt \sqrt{\frac{r}{T}} \sigma_{\max}^\star + \mu \delt \sqrt{\frac{r}{T}} \sigma_{\max}^\star \\
&\leqslant 1.24 \mu \delt \sqrt{\frac{r}{T}} \sigma_{\max}^\star. 
\end{align*}
This implies that 
$$
\| \Theta_l - \Thetas \|_F \leqslant 1.24 \mu \delt \sqrt{r} \sigma_{\max}^\star. 
$$
This proves 4) and 5).

Furthermore, 
\begin{align*}
\sigma_{\min}(W) &= \sigma_{\min}(G - (G - W)) \\
&\geqslant \sigma_{\min}(G) - \| W - G \| \\
&\geqslant \sigma_{\min}(G) - \| W - G \|_F \\
\end{align*}
we have
\begin{align*}
\sigma_{\min}(G) &= \sigma_{\min}(G^\top) \\
&= \sigma_{\min}({W^\star}^\top {B^\star}^\top B) \\
&\geqslant \sigma_{\min}^\star \sigma_{\min}({B^\star}^\top B), 
\end{align*}
and
\begin{align*}
\sigma_{\min}({B^\star}^\top B) &= \sqrt{\lambda_{\min}(B^\top B^\star {B^\star}^\top B)}) \\
&= \sqrt{\lambda_{\min}(B^\top (I - P) B)} \\
&= \sqrt{\lambda_{\min}(I - B^\top P B)} \\
&= \sqrt{\lambda_{\min}(I - B^\top P^2 B)} \\
&= \sqrt{1 - \lambda_{\max}(B^\top P^2 B)} \\
&= \sqrt{1 - \| P B \|^2} \\
&\geqslant \sqrt{1 - \delt^2}. 
\end{align*}
Combining the above three bounds, if $\delt < \frac{0.02}{\mu \sqrt{r} \kappa^2}$, we then have
$$
\sigma_{\min}(W) \geqslant \sqrt{1 - \delt^2} \sigma_{\min}^\star - 0.24 \mu \delt \sqrt{r} \sigma_{\max}^\star \geqslant 0.9 \sigma_{\min}^\star
$$
and
\begin{align*}
\sigma_{\max}(W) &= \sigma_{\max}(G - (G - W)) \\
&\leqslant \sigma_{\max}(G) + \sigma_{\max}(G - W) \\
&= \sigma_{\max}(B^\top B^\star W^\star) + \sigma_{\max}(G - W) \\
&\leqslant \sigma_{\max}(B^\top B^\star) \sigma_{\max}(W^\star) + \| G - W \|_F \\
&\leqslant \sigma_{\max}^\star + 0.24 \mu \delt \sqrt{r} \sigma_{\max}^\star \\
&\leqslant 1.1 {\sigma_{\max}^\star}
\end{align*}
Thus, the proof is complete. 
\end{proof}

\begin{proposition} \label{proposition4}
Assume $\SE(B, B^\star) \leqslant \delt$. The following statements are true: 
\begin{itemize}
    \item $\bE[\GradB^\prime] = (\pG_m - \pG_{m-1}) (\Thetas - \Theta) W^\top$
    \item With probability at least $1 - 3 T \exp(r - c (\pG_m - \pG_{m-1}))$, we have $\|\bE[\GradB^\prime]\| \leqslant 1.37 (\pG_m - \pG_{m-1}) \mu \delt \sqrt{r} {\sigma_{\max}^\star}^2$
    \item If $\delt \leqslant \frac{0.02}{\mu \sqrt{r} \kappa^2}$, then, with probability at least $1 - \exp(C (d + r) - c \frac{\epsilon_1^2 (\pG_m - \pG_{m-1}) T}{\mu^2 r \kappa^4}) - 3 \exp(\log T + r - c (\pG_m - \pG_{m-1}))$, the inequality $\|\GradB^\prime - \bE[\GradB^\prime]\| \leqslant \epsilon_1 \delt (\pG_m - \pG_{m-1}) {\sigma_{\min}^\star}^2$ holds
\end{itemize}
where $\GradB^\prime = \sum_{t=1, n = \pG_{m-1} + 1}^{T, \pG_m} \phi(x_{n, t}, c_n) \phi(x_{n, t}, c_n)^\top (\thetats - \theta_t) w_t^\top$.
\end{proposition}
\begin{proof}
By using independence of $\phimt{t}{m}$ and $\{B, w_t\}$, we can derive 
\begin{align*}
\bE[\GradB^\prime] &= \bE\left[\sum_{t=1, n = \pG_{m-1} + 1}^{T, \pG_m} \phi(x_{n, t}, c_n) \phi(x_{n, t}, c_n)^\top (\thetats - \theta_t) w_t^\top\right] \\
&= \sum_{t=1, n = \pG_{m-1} + 1}^{T, \pG_m} \bE[\phi(x_{n, t}, c_n) \phi(x_{n, t}, c_n)^\top] (\thetats - \theta_t) w_t^\top \\
&= \sum_{t=1}^{T} (\pG_m - \pG_{m-1}) (\thetats - \theta_t) w_t^\top \\
&= (\pG_m - \pG_{m-1}) (\Thetas - \Theta) W^\top. 
\end{align*}
Utilizing the upper bound from Lemma~\ref{new_1}, if $\delt \leqslant \frac{0.02}{\mu \sqrt{r} \kappa^2}$, with probability at least $1 - 3 \exp(\log T + r - c (\pG_m - \pG_{m-1}))$, 
\begin{align*}
\|\bE[\GradB^\prime]\| &= \|\sum_{t=1}^{T} (\pG_m - \pG_{m-1}) (\thetats - \theta_t) w_t^\top\| \\
&= (\pG_m - \pG_{m-1}) \|(\Thetas - \Theta) W^\top\| \\
&\leqslant (\pG_m - \pG_{m-1}) \|\Thetas - \Theta\| \cdot \|W\| \\
&\leqslant (\pG_m - \pG_{m-1}) \|\Thetas - \Theta\|_F \cdot \|W\| \\
&\leqslant 1.6 (\pG_m - \pG_{m-1}) \mu \delt \sqrt{r} {\sigma_{\max}^\star}^2
\end{align*}
To bound $\|\GradB^\prime - \bE[\GradB^\prime]\| = \max_{\|z\|=1, \|v\|=1} z^\top (\sum_{t=1}^T \sum_{n=\pG_{m-1}+1}^{\pG_m} \phi(x_{n, t}, c_n) \phi(x_{n, t}, c_n)^\top (\thetats - \theta_t) w_t^\top - \bE[\phi(x_{n, t}, c_n) \phi(x_{n, t}, c_n)^\top (\thetats - \theta_t) w_t^\top]) v$, we consider fixed unit norm vectors $z, v$, applying the sub-exponential Berstein inequality as stated in Proposition~\ref{proposition1} and extend the bound to all unit norm vectors $z, v$ using a standard epsilon-net argument. For fixed unit norm $z, v$, we consider
$$
\sum_{t=1}^T \sum_{n=\pG_{m-1}+1}^{\pG_m} \left((z^\top \phi(x_{n, t}, c_n)) (w_t^\top v) \phi(x_{n, t}, c_n)^\top (\thetats - \theta_t) - \bE[z^\top \phi(x_{n, t}, c_n)) (w_t^\top v) \phi(x_{n, t}, c_n)^\top (\thetats - \theta_t)]\right)
$$
The analysis shows that
\begin{align*}
&\bE\left[(z^\top \phi(x_{n, t}, c_n)) (w_t^\top v) \phi(x_{n, t}, c_n)^\top (\thetats - \theta_t) - \bE[z^\top \phi(x_{n, t}, c_n)) (w_t^\top v) \phi(x_{n, t}, c_n)^\top (\thetats - \theta_t)]\right] \\
&= \left( \bE[z^\top \phi(x_{n, t}, c_n)) (w_t^\top v) \phi(x_{n, t}, c_n)^\top (\thetats - \theta_t)] - \bE[z^\top \phi(x_{n, t}, c_n)) (w_t^\top v) \phi(x_{n, t}, c_n)^\top (\thetats - \theta_t)] \right) \\
&= 0, 
\end{align*}
and also we have that
\begin{align*}
\bE[(z^\top \phi(x_{n, t}, c_n)) (w_t^\top v)] &= 0, \\
\Var((z^\top \phi(x_{n, t}, c_n)) (w_t^\top v))&= \bE[(z^\top \phi(x_{n, t}, c_n)) (w_t^\top v) (w_t^\top v) \phi(x_{n, t}, c_n)^\top z] \\
&= (w_t^\top v)^2 z^\top \bE[\phi(x_{n, t}, c_n) \phi(x_{n, t}, c_n)^\top] z \\
&= (w_t^\top v)^2, 
\end{align*}
and
\begin{align*}
\bE[\phi(x_{n, t}, c_n)^\top (\thetats - \theta_t)] &= 0, \\
\Var(\phi(x_{n, t}, c_n)^\top (\thetats - \theta_t))&= \bE[(\thetats - \theta_t)^\top \phi(x_{n, t}, c_n) \phi(x_{n, t}, c_n)^\top (\thetats - \theta_t)] \\
&= (\thetats - \theta_t)^\top \bE[\phi(x_{n, t}, c_n) \phi(x_{n, t}, c_n)^\top] (\thetats - \theta_t) \\
&= \|\thetats - \theta_t\|^2
\end{align*}
Based on the analysis provided, we determine that the summands are independent, zero mean, sub-exponential random variables with sub-exponential norm $K_\nt \leqslant |w_t^\top v| \|\thetats - \theta_t\|$. We apply the sub-exponential Bernstein inequality stated in Proposition~\ref{proposition1}, with $g = \epsilon_1 \delt (\pG_m - \pG_{m-1}) {\sigma_{\min}^\star}^2$. We have
\begin{align}
\frac{g^2}{\sum_{t=1, n = \pG_{m-1} + 1}^{T, \pG_m} K_\nt^2} &\geqslant \frac{\epsilon_1^2 \delt^2 (\pG_m - \pG_{m-1})^2 {\sigma_{\min}^\star}^4}{(\pG_m - \pG_{m-1}) \sum_{t=1}^T |w_t^\top v|^2 \|\thetats - \theta_t\|^2} \nonumber \\
&\geqslant \frac{\epsilon_1^2 \delt^2 (\pG_m - \pG_{m-1})^2 {\sigma_{\min}^\star}^4}{(\pG_m - \pG_{m-1}) \max_t \|\thetats - \theta_t\|^2 \sum_{t=1}^T |w_t^\top v|^2} \nonumber \\
&= \frac{\epsilon_1^2 \delt^2 (\pG_m - \pG_{m-1})^2 {\sigma_{\min}^\star}^4}{(\pG_m - \pG_{m-1}) \max_t \|\thetats - \theta_t\|^2 \|v^\top W\|^2} \nonumber \\
&\geqslant \frac{\epsilon_1^2 \delt^2 (\pG_m - \pG_{m-1})^2 {\sigma_{\min}^\star}^4 T}{1.24^2 (\pG_m - \pG_{m-1}) \mu^2 \delt^2 r {\sigma_{\max}^\star}^2 \|W\|^2} \label{p4_1} \\
&\geqslant \frac{\epsilon_1^2 \delt^2 (\pG_m - \pG_{m-1})^2 {\sigma_{\min}^\star}^4 T}{1.87 (\pG_m - \pG_{m-1}) \mu^2 \delt^2 r {\sigma_{\max}^\star}^4} \label{p4_2} \\
&= \frac{\epsilon_1^2 (\pG_m - \pG_{m-1}) T}{1.87 \mu^2 r \kappa^4} \nonumber \\
\frac{g}{\max_\nt K_\nt} &\geqslant \frac{\epsilon_1 \delt (\pG_m - \pG_{m-1}) {\sigma_{\min}^\star}^2}{\max_\nt |w_t^\top v| \|\thetats - \theta_t\|} \nonumber \\
&\geqslant \frac{\epsilon_1 \delt (\pG_m - \pG_{m-1}) {\sigma_{\min}^\star}^2}{\max_t \|\thetats - \theta_t\| \max_t \|w_t\|} \label{p4_3} \\
&\geqslant \frac{\epsilon_1 \delt (\pG_m - \pG_{m-1}) {\sigma_{\min}^\star}^2 T}{1.54 \mu^2 \delt r {\sigma_{\max}^\star}^2} \label{p4_4} \\
&= \frac{\epsilon_1 (\pG_m - \pG_{m-1}) T}{1.54 \mu^2 r \kappa^2} \nonumber
\end{align}
where Eq.~\eqref{p4_1} follows from $\sum_{t=1}^T |w_t^\top v|^2 = \|v^\top w_t\|^2 \leqslant \|W\|^2$ and the upper bound of $\|\thetats - \theta_t\|$ resulting from Lemma~\ref{new_1}. Eq.~\eqref{p4_2} follows from the upper bound $\|W\| \leqslant 1.1 \sigma_{\max}^\star$ obtained from Lemma~\ref{new_1}. Eq.~\eqref{p4_3} follows from $|w_t^\top v| \leqslant \|w_t\|$. Eq.~\eqref{p4_4} follows from the upper bound $\|W\| \leqslant 1.1 \sigma_{\max}^\star$ and $\|w_t\| \leqslant 1.24 \mu \sqrt{\frac{r}{T}} \sigma_{\max}^\star$ derived from Lemma~\ref{new_1}. Consequently, with probability at least $1 - \exp(- c \frac{\epsilon_1^2 (\pG_m - \pG_{m-1}) T}{\mu^2 r \kappa^4}) - 3 \exp(\log T + r - c (\pG_m - \pG_{m-1}))$, for a given $z, v$, 
$$
z^\top (\GradB^\prime - \bE[\GradB^\prime]) v \leqslant \epsilon_1 \delt (\pG_m - \pG_{m-1}) {\sigma_{\min}^\star}^2. 
$$
Applying a standard epsilon-net argument to bound the maximum of the above over all unit norm $z, v$. We conclude that
$$
\|\GradB^\prime - \bE[\GradB^\prime]\| \leqslant \epsilon_1 \delt (\pG_m - \pG_{m-1}) {\sigma_{\min}^\star}^2
$$
with probability at least $1 - \exp(C (d + r) - c \frac{\epsilon_1^2 (\pG_m - \pG_{m-1}) T}{\mu^2 r \kappa^4}) - 3 \exp(\log T + r - c (\pG_m - \pG_{m-1}))$. The probability factor of $\exp(C (d + r))$ arises from the epsilon-net over $z$ and that over $v$: $z$ is an $d$-length unit norm vector while $v$ is an $r$-length unit norm vector. The size of the smallest epsilon net that covers the hyper-sphere of all $z$s is $(1 + \frac{2}{\epsilon_{net}})^d$, where $\epsilon_{net} = c$. Similarly, the size of the epsilon net that covers $v$ is $C^r$. Applying the union bound over both results in a factor of $C^{d + r}$. This completes the proof.
\end{proof}
Now we have the following lemma for the gradient when noise is considerd. 
\begin{lemma} \label{new_2}
Assume that $\SE(B, B^\star) \leqslant \delt$, and $\sigma_\eta^2 \leqslant \frac{r}{T} \delt^2 {\sigma_{\min}^\star}^2$. The following statements are true:
\begin{itemize}
    \item $\bE[\GradB] = (\pG_m - \pG_{m-1}) (\Theta - \Thetas) W^\top = (\pG_m - \pG_{m-1}) (B W W^\top - \Thetas W^\top)$
    \item With probability at least $1 - 3 T \exp(r - c(\pG_m - \pG_{m-1}))$, we have $\| \bE[\GradB] \| \leqslant 1.37 (\pG_m - \pG_{m-1}) \mu \delt \sqrt{r} {\sigma_{\max}^\star}^2$
    \item If $\delt \leqslant \frac{0.02}{\mu \sqrt{r} \kappa^2}$, then with probability at least $1 - 2 \exp(C (d + r) - c \frac{\epsilon_1^2 (\pG_m - \pG_{m-1}) T}{\mu^2 r \kappa^4}) - 3 \exp(\log T + r - c (\pG_m - \pG_{m-1}))$, 
    $$
    \| \GradB - \bE[\GradB] \| \leqslant 2 \epsilon_1 (\pG_m - \pG_{m-1}) \delt {\sigma_{\min}^\star}^2.
    $$
\end{itemize}
\end{lemma}
\begin{proof}
Recall the definition of $\GradB$.
\begin{align*}
    \GradB& = \sum_{t=1}^T \phimt{t}{m} (\phimt{t}{m} B w_k - \ymt{t}{m}) w_t^\top\\
     &= \sum_{t=1, n = \pG_{m-1} + 1}^{T, \pG_m} \phi(x_{n, t}, c_n) \phi(x_{n, t}, c_n)^\top (\theta_t - \thetats) w_t^\top - \eta_{n, t} \phi(x_{n, t}, c_n) w_t^\top.
\end{align*}
From this we obtain
\begin{align*}
\bE[\GradB] &= \bE\left[\sum_{t=1, n = \pG_{m-1} + 1}^{T, \pG_m} \phi(x_{n, t}, c_n) \phi(x_{n, t}, c_n)^\top (\theta_t - \thetats) w_t^\top - \eta_{n, t} \phi(x_{n, t}, c_n) w_t^\top\right] \\
&= (\pG_m - \pG_{m-1}) \sum_{t=1}^T (\theta_t - \thetats) w_t^\top \\
&= (\pG_m - \pG_{m-1}) (\Theta - \Thetas) W^\top
\end{align*}
Applying bounds on $\| W \|$ and $(\Thetas - \Theta)$ from Lemma~\ref{new_1}, we have
\begin{align*}
\| \bE[\GradB] \| &= \| (\pG_m - \pG_{m-1}) (\Theta - \Thetas) W^\top \| \\
&\leqslant (\pG_m - \pG_{m-1}) \| \Theta - \Thetas \|_F \| W \| \\
&\leqslant 1.37 (\pG_m - \pG_{m-1}) \mu \delt \sqrt{r} {\sigma_{\max}^\star}^2. 
\end{align*}
Subsequently, we finish the proof of the bound for $\| \bE[\GradB] \|$. Considering unit vectors $v, z$, we need to bound $\sum_{t=1, n = \pG_{m-1} + 1}^{T, \pG_m} \eta_{n, t} v^\top\phi(x_{n, t}, c_n) w_t^\top z$. This implies $\bE[\eta_{n, t} v^\top\phi(x_{n, t}, c_n) w_t^\top z] = 0$ and 
\begin{align*}
\Var(v^\top \phi(x_{n, t}, c_n) w_t^\top z) &= \bE[v^\top \phi(x_{n, t}, c_n) w_t^\top z]^2 - (\bE[v^\top \phi(x_{n, t}, c_n) w_t^\top z])^2 \\
&= \bE[v^\top \phi(x_{n, t}, c_n) w_t^\top z z^\top w_t \phi(x_{n, t}, c_n)^\top v] \\
&= |w_t^\top z|^2 v^\top \bE[\phi(x_{n, t}, c_n) \phi(x_{n, t}, c_n)^\top] v \\
&= |w_t^\top z|^2
\end{align*}
Given $\eta_{n, t} \iidsim \n(0, \sigma_\eta^2)$, we have $\Var(\eta_{n, t}) = \sigma_\eta^2$. Therefore, $\eta_{n, t} v^\top\phi(x_{n, t}, c_n) w_t^\top z$ is a sum of subexponential random variables with parameter $K_{n, t} \leqslant |w_t^\top z| \sigma_\eta$. Setting $g = \epsilon_2 (\pG_m - \pG_{m-1}) \sigma_{\min}^\star \sigma_\eta \sqrt{\frac{T}{r}}$, we obtain
\begin{align*}
\frac{g^2}{\sum_{t=1, n = \pG_{m-1} + 1}^{T, \pG_m} K_{n, t}^2} &\geqslant \frac{\epsilon_2^2 (\pG_m - \pG_{m-1})^2 {\sigma_{\min}^\star}^2 \sigma_\eta^2 \frac{T}{r}}{\sigma_\eta^2 \sum_{t=1, n = \pG_{m-1} + 1}^{T, \pG_m} (w_t^\top z)^2} \\
&\geqslant \frac{\epsilon_2^2 (\pG_m - \pG_{m-1})^2 {\sigma_{\min}^\star}^2 \sigma_\eta^2 \frac{T}{r}}{\sigma_\eta^2 (\pG_m - \pG_{m-1}) \| W \|^2} \\
&\geqslant \frac{\epsilon_2^2 (\pG_m - \pG_{m-1})^2 {\sigma_{\min}^\star}^2 \sigma_\eta^2 \frac{T}{r}}{1.3 (\pG_m - \pG_{m-1}) \sigma_\eta^2 {\sigma_{\max}^\star}^2} \\
&= \frac{\epsilon_2^2 (\pG_m - \pG_{m-1}) T}{1.3 r \kappa^2}, 
\end{align*}
\begin{align*}
\frac{g}{\max_{n, t} K_{n, t}} &\geqslant \frac{\epsilon_2 (\pG_m - \pG_{m-1}) \sigma_{\min}^\star \sigma_\eta \sqrt{\frac{T}{r}}}{\sigma_\eta \max_{n, t} \| w_t \|} \\
&\geqslant \frac{\epsilon_2 (\pG_m - \pG_{m-1}) \sigma_{\min}^\star \sigma_\eta \sqrt{\frac{T}{r}}}{1.24 \sigma_\eta \mu \sqrt{\frac{r}{T}} \sigma_{\max}^\star} \\
&\geqslant \frac{\epsilon_2 (\pG_m - \pG_{m-1}) T}{1.24 \mu r \kappa}. 
\end{align*}
Consequently, with probability at least $1 - \exp(- c \frac{\epsilon_2^2 (\pG_m - \pG_{m-1}) T}{\mu r \kappa^2})$, for fixed $v, z$, $\sum_{t=1, n = \pG_{m-1} + 1}^{T, \pG_m} |\eta_{n, t} v^\top \phi(x_{n, t}, c_n) w_t^\top z| \leqslant \epsilon_2 (\pG_m - \pG_{m-1}) \sigma_{\min}^\star \sigma_\eta \sqrt{\frac{T}{r}}$. Utilizing an epsilon-net to maximize over all unit vectors $v, z$. This will give a factor of $\exp(d + r)$ in probability. Thus, with probability at least $1 - \exp((d + r) - c \frac{\epsilon_2^2 (\pG_m - \pG_{m-1}) T}{\mu r \kappa^2})$, 
$$
\left\| \sum_{t=1, n = \pG_{m-1} + 1}^{T, \pG_m} \eta_{n, t} \phi(x_{n, t}, c_n) w_t^\top \right\| \leqslant \epsilon_2 (\pG_m - \pG_{m-1}) \sigma_{\min}^\star \sigma_\eta \sqrt{\frac{T}{r}}. 
$$
Recall $\GradB^\prime = \sum_{t=1, n = \pG_{m-1} + 1}^{T, \pG_m} \phi(x_{n, t}, c_n) \phi(x_{n, t}, c_n)^\top (\thetats - \theta_t) w_t^\top$. From Proposition~\ref{proposition4}, if $\delt \leqslant \frac{0.02}{\mu \sqrt{r} \kappa^2}$, then, with probability at least $1 - \exp(C (d + r) - c \frac{\epsilon_1^2 (\pG_m - \pG_{m-1}) T}{\mu^2 r \kappa^4}) - 3 \exp(\log T + r - c (\pG_m - \pG_{m-1}))$, it holds that $\| \GradB^\prime - \bE[\GradB^\prime]  \| \leqslant \epsilon_1 \delt (\pG_m - \pG_{m-1}) {\sigma_{\min}^\star}^2$. By combining both and setting $\epsilon_2 = \epsilon_1$, we conclude that with probability at least $1 - 2 \exp(C (d + r) - c \frac{\epsilon_1^2 (\pG_m - \pG_{m-1}) T}{\mu^2 r \kappa^4}) - 3 \exp(\log T + r - c (\pG_m - \pG_{m-1}))$, 
$$
\| \GradB - \bE[\GradB] \| \leqslant \epsilon_1 (\pG_m - \pG_{m-1}) (\delt \sigma_{\min}^\star + \sigma_\eta \sqrt{\frac{T}{r}}) \sigma_{\min}^\star. 
$$
Thus, if $\sigma_\eta^2 \leqslant \frac{r}{T} \delt^2 {\sigma_{\min}^\star}^2$, then we have
$$
\| \GradB - \bE[\GradB] \| \leqslant 2 \epsilon_1 (\pG_m - \pG_{m-1}) \delt {\sigma_{\min}^\star}^2. 
$$
This completes the proof. 
\end{proof}

%\begin{theorem} \label{new_3}
%Assume that Assumption~\ref{assume:incoherence} holds, $\SE(B, B^\star) \leqslant \delt$, and $\sigma_\eta^2 \leqslant \frac{r}{T} \delt^2 {\sigma_{\min}^\star}^2$. If $\delt \leqslant \frac{0.02}{\sqrt{r} \kappa^2}$, $\gamma = \frac{c_\gamma}{(\pG_m - \pG_{m-1}) {\sigma_{\max}^\star}^2}$ with $c_\gamma \leqslant 0.5$, and if $(\pG_m - \pG_{m-1}) T \geqslant C \kappa^4 \mu^2 d r$ and $(\pG_m - \pG_{m-1}) \gtrsim \max(\log d, \log T, r)$, then with probability at least $1 - 2 d^{-10}$, 
%$$\SE(B^+, B^\star) \leqslant \delto := (1 - \frac{0.6 c_\gamma}{\kappa^2}) \delt. $$
%\end{theorem}

\subsection{Proof of Theorem~\ref{new_3}}\label{app_new_3}
%\begin{proof}
Consider the Projected GD step for $B$: $\widehat{B}^+ = B - \frac{\gamma}{(\pG_m - \pG_{m-1})} \GradB$ and $\widehat{B}^+ \overset{QR}{=} B^+ R^+$. Given that $B^+ = \widehat{B}^+ (R^+)^{-1}$ and $\| (R^+)^{-1} \| = \frac{1}{\sigma_{\min} (R^+)} = \frac{1}{\sigma_{\min}(\widehat{B}^+)}$, it follows that $\SE(B^+, B^\star) = \| P B^+ \|$ can be bound as
\begin{equation}
\SE(B^+, B^\star) \leqslant \frac{\| P \widehat{B}^+ \|}{\sigma_{\min}(\widehat{B}^+)} \leqslant \frac{\| P \widehat{B}^+ \|}{\sigma_{\min}(B) - \frac{\gamma}{(\pG_m - \pG_{m-1})} \| \GradB \|}. \label{new_3_1}
\end{equation}
By considering the numerator and performing adding and subtracting of $\bE[\GradB]$, left multiplying both sides by $P$, and utilizing the result from Lemma~\ref{new_2}, we derive 
$$
\widehat{B}^+ = B - \frac{\gamma}{(\pG_m - \pG_{m-1})} \bE[\GradB] + \frac{\gamma}{(\pG_m - \pG_{m-1})} (\bE[\GradB] - \GradB). 
$$
Consequently, 
\begin{align*}
P \widehat{B}^+ &= P B - \gamma P B W W^\top + \gamma P \Thetas W^\top + \frac{\gamma}{(\pG_m - \pG_{m-1})} P (\bE[\GradB] - \GradB) \\
&= P B - \gamma P B W W^\top + \frac{\gamma}{(\pG_m - \pG_{m-1})} P (\bE[\GradB] - \GradB)
\end{align*}
where the last step follows by $P \Thetas = (I - B^\star {B^\star}^\top) \Thetas = \Thetas - B^\star {B^\star}^\top B^\star W^\star = 0$. Thus, 
\begin{equation}
\| P \widehat{B}^+ \| \leqslant \| P B \| \| I - \gamma W W^\top \| + \frac{\gamma}{(\pG_m - \pG_{m-1})} \| \bE[\GradB] - \GradB \|. \label{new_3_2}
\end{equation}
Applying the result stated in Lemma~\ref{new_1}, we obtain
$$
\lambda_{\min}(I - \gamma W W^\top) = 1 - \gamma \| W \|^2 \geqslant 1 - 1.21 \gamma {\sigma_{\max}^\star}^2. 
$$
Therefore, for $\gamma < \frac{0.5}{{\sigma_{\max}^\star}^2}$, then the matrix mentioned above is a positive semidefinite. Furthermore, this along with Lemma~\ref{new_1}, leads to that
$$
\| I - \gamma W W^\top \| = \lambda_{\max}(I - \gamma W W^\top) \leqslant 1 - 0.81 \gamma {\sigma_{\min}^\star}^2. 
$$
Based on the result mentioned above, Eq.~\eqref{new_3_2}, and the bound on $\| \bE[\GradB] - \GradB \|$ from Lemma~\ref{new_2}, we conclude the following: If $\gamma < \frac{0.5}{{\sigma_{\max}^\star}^2}$ and $\delt \leqslant \frac{0.02}{\mu \sqrt{r} \kappa^2}$, then with probability at least $1 - 2 \exp(C (d + r) - \frac{c \epsilon_1^2 (\pG_m - \pG_{m-1}) T}{\mu^2 r \kappa^4}) - 3 \exp(\log T + r - c (\pG_m - \pG_{m-1}))$, 
\begin{align}
\| P \widehat{B}^+ \| &\leqslant \| P B \| \| I - \gamma W W^\top \| + \frac{\gamma}{(\pG_m - \pG_{m-1})} \| \bE[\GradB] - \GradB \| \nonumber \\
&\leqslant (1 - 0.81 \gamma {\sigma_{\min}^\star}^2) \delt + 2 \epsilon_1 \gamma \delt {\sigma_{\min}^\star}^2. \label{new_3_3}
\end{align}
This probability is at least $1 - d^{-10}$ if $(\pG_m - \pG_{m-1}) T \geqslant C \mu^2 \kappa^4 (d + r) r$ and $(\pG_m - \pG_{m-1}) \geqslant C (r + \log T + \log d)$. 
Subsequently, we use Eq.~\eqref{new_3_3} with $\epsilon_1 = \epsilon_2 = 0.1$ and Lemma~\ref{new_2} in Eq.~\eqref{new_3_1}, and setting $\gamma = \frac{c_\gamma}{{\sigma_{\max}^\star}^2}$. If $c_\gamma \leqslant 0.5$, if $\delt \leqslant \frac{0.02}{\mu \sqrt{r} \kappa^2}$, and lower bounds on $(\pG_m - \pG_{m-1})$ from above hold, then Eq.~\eqref{new_3_1} implies that with high probability, 
\begin{align}
\SE(B^+, B^\star) &\leqslant \frac{\| P \widehat{B}^+ \|}{\sigma_{\min}(\widehat{B}^+)} \nonumber \\
&\leqslant \frac{\| P \widehat{B}^+ \|}{\sigma_{\min}(B) - \frac{\gamma}{(\pG_m - \pG_{m-1})} \| \GradB \|} \nonumber \\
&= \frac{\| P \widehat{B}^+ \|}{\sigma_{\min}(B) - \frac{\gamma}{(\pG_m - \pG_{m-1})} \| \GradB - \bE[\GradB] + \bE[\GradB] \|} \nonumber \\
&\leqslant \frac{\| P B \| \| I - \gamma W W^\top \| + \frac{\gamma}{(\pG_m - \pG_{m-1})} \| \bE[\GradB] - \GradB \|}{1 - \frac{\gamma}{(\pG_m - \pG_{m-1})} \| \bE[\GradB] \| - \frac{\gamma}{(\pG_m - \pG_{m-1})} \| \GradB - \bE[\GradB] \|} \nonumber \\
&\leqslant \frac{(1 - (0.81 - 0.2) \gamma {\sigma_{\min}^\star}^2) \delt}{1 - \frac{\gamma}{(\pG_m - \pG_{m-1})} \| \bE[\GradB] \| - \frac{\gamma}{(\pG_m - \pG_{m-1})} \| \GradB - \bE[\GradB] \|} \nonumber \\
&\leqslant \frac{(1 - 0.61 \gamma {\sigma_{\min}^\star}^2) \delt}{1 - \gamma \delt \sqrt{r} {\sigma_{\max}^\star}^2 (1.37 \mu + \frac{0.2}{\kappa^2 \sqrt{r}})} \nonumber \\
&\leqslant \frac{(1 - 0.61 \gamma {\sigma_{\min}^\star}^2) \delt}{1 - 1.57 \mu \gamma \delt \sqrt{r} {\sigma_{\max}^\star}^2} \label{new_3_4} \\
&\leqslant (1 - 0.61 \gamma {\sigma_{\min}^\star}^2) (1 + 3.14 \mu \gamma \delt \sqrt{r} {\sigma_{\max}^\star}^2) \delt \label{new_3_5} \\
&\leqslant (1 - 0.61 \gamma {\sigma_{\min}^\star}^2 + 3.14 \gamma \mu \delt \sqrt{r} {\sigma_{\max}^\star}^2) \delt \nonumber \\
&= (1 - \gamma {\sigma_{\min}^\star}^2 (0.61 - 3.14 \mu \delt \sqrt{r} \kappa^2)) \delt \nonumber \\
&\leqslant (1 - \gamma {\sigma_{\min}^\star}^2 (0.61 - 0.0628)) \delt \label{new_3_6} \\
&\leqslant (1 - 0.5472 \gamma {\sigma_{\max}^\star}^2 / \kappa^2) \delt \nonumber \\
&= (1 - 0.5472 \frac{c_\gamma}{\kappa^2}) \delt \label{new_3_7}
\end{align}
where Eq.~\eqref{new_3_4} follows from $\kappa^2 \sqrt{r} > 1$. Eq.~\eqref{new_3_5} follows from $(1 - x)^{-1} < (1 + 2 x)$ if $|x| \leqslant \frac{1}{2}$. Eq.~\eqref{new_3_6} follows from $\delt \leqslant \frac{0.02}{\mu \sqrt{r} \kappa^2}$. Eq.~\eqref{new_3_7} follows from $\gamma = \frac{c_\gamma}{{\sigma_{\max}^\star}^2}$. This completes the proof. 
%\end{proof}
\qed

%\begin{theorem} \label{new_4}
%Assume that Assumption~\ref{assume:incoherence} holds. Assume that $\sigma_\eta^2 \leqslant c \frac{\delta_0^2}{k^2 \kappa^4} \| \thetats \|^2$, then with probability at least $1 - \exp(\log T - c \pG_1) - \exp(d - \frac{c \delta_0^2 \pG_1 T}{k^2 \mu^2 \kappa^4})$, we have 
%$$\SE(\Bm{0}, B^\star) \leqslant \delta_0. $$
%\end{theorem}
%\begin{proof}
\subsection{Proof of Theorem~\ref{new_4}}\label{app_new_4}
We analyze the initialization process by computing $\Bm{0}$ as top $r$ singular vectors of $Y_B = \sum_{t = 1}^T \sum_{n = 1}^{\pG_1} y_{n, t}^2 \phi(x_{n, t}, c_n) \phi(x_{n, t}, c_n)^\top \indic_{\{y_{n, t}^2 \leqslant C_0 \sum_{t = 1}^T \sum_{n = 1}^{\pG_1} \frac{y_{n, t}^2}{\pG_1 T}\}}$. Subsequently, we use Claim~B.15 from \cite{lrpr_best} to analyze this. Claim~B.15 shows that if 
$$
\|\etamt{t}{m}\|^2 \leqslant c \frac{\delta_0^2}{r^2 \kappa^4} \| \thetats \|^2, 
$$
then with probability at least $1 - \exp(d - \frac{c \delta_0^2 \pG_1 T}{r^2 \mu^2 \kappa^4})$, 
$$
\SE(\Bm{0}, B^\star) \leqslant \delta_0. 
$$
In order to determine an upper bound for $\| \etamt{t}{m} \|^2 = \sum_{n=1}^{\pG_1} \eta_\nt^2$, we observe that $\eta_{n, t} \iidsim \n(0, \sigma_\eta^2)$. Thus, $\| \etamt{t}{m} \|^2$ is a sum of subexponential random variables with parameter $K_n \leqslant \sigma_\eta \cdot \sigma_\eta = \sigma_\eta^2$. We apply the sub-exponential Bernstein inequality stated in Proposition~\ref{proposition1}, with $g=0.1 \pG_1 \sigma_\eta^2$. We have 
\begin{align*}
\frac{g^2}{\sum_{n=1}^{\pG_1} K_n^2} &\geqslant \frac{0.01 \pG_1^2 \sigma_\eta^4}{\pG_1 \sigma_\eta^4} =0.01 \pG_1 \\
\frac{g}{\max_n K_n} &\geqslant \frac{0.1 \pG_1 \sigma_\eta^2}{\sigma_\eta^2} = 0.1 \pG_1
\end{align*}
Since $\bE[\eta_\nt] = \sigma_\eta^2$, it can be proved that with probability at least $1 - \exp(c \pG_1)$, $\sum_{n=1}^{\pG_1} \eta_\nt^2 \leqslant 0.1 \pG_1 \sigma_\eta^2$. Thus, we can determine that with probability at least $1 - \exp(c \pG_1)$, 
\begin{align*}
\| \etamt{t}{m} \|^2 &= \sum_{n=1}^{\pG_1} \eta_\nt^2 \\
&\leqslant \sum_{n=1}^{\pG_1} \bE[\eta_\nt^2] + 0.1 \pG_1 \sigma_\eta^2 \\
&= 1.1 \pG_1 \sigma_\eta^2
\end{align*}
By utilizing a union bound over all $T$ vectors, we conclude that with probability at least $1 - \exp(\log T - c \pG_1)$, $\| \etamt{t}{m}\| \leqslant 1.1 \pG_1 \sigma_\eta^2$. By combining the results from \cite{lrpr_best}, we complete the proof. 
%\end{proof}
\qed

%\begin{theorem} \label{new_5}
%Assume $\sigma_\eta^2 \leqslant \frac{c \| \thetats \|^2}{r^3 \kappa^6}$ holds and also Assumption~\ref{assume:incoherence} holds. Set $\gamma = \frac{0.4}{(\pG_m - \pG_{m-1}) {\sigma_{\max}^\star}^2}$ and $L = C \kappa^2 \log(\frac{1}{\max(\epsilon, \epsilon_{noise})})$. If $(\pG_m - \pG_{m-1}) T \geqslant C \kappa^6 \mu^2 (d + T) r (\kappa^2 r^2 + \log(\frac{1}{\max(\epsilon, \epsilon_{noise})}))$ and $\pG_m - \pG_{m-1} \geqslant C \max(\log{d} \log{T}, r) \log(\frac{1}{\max(\epsilon, \epsilon_{noise})})$, then with probability at least $O(1 - d^{-10})$, 
%$$\SE(B, B^\star) \leqslant \max(\epsilon, \epsilon_{noise}) \quad \text{and} \quad \| \thetahatt - \thetats \| \leqslant \max(\epsilon, \epsilon_{noise}) \| \thetats \| \quad \text{for all} \quad t \in [T], $$
%where $\epsilon_{noise} = C \kappa^2 \sqrt{NSR}$, $NSR := \frac{\sigma_\eta^2}{\min_t \| \thetats \|^2}$. The time complexity is $(\pG_m - \pG_{m-1}) T d r \cdot L = C \kappa^2 (\pG_m - \pG_{m-1}) T d r \log(\frac{1}{\max(\epsilon, \epsilon_{noise})})$. The communication complexity is $d r$ per node per iteration. 
%\end{theorem}
%\begin{proof}
\subsection{Proof of Theorem~\ref{new_5}}\label{app_new_5}
From Theorem~\ref{new_4}, we know that at the initialization round, we need
$$
\sigma_\eta^2 \leqslant c \frac{\delta_0^2}{r^2 \kappa^4 \pG_1} \| \thetats \|^2. 
$$
At GD round $\l$, we assume that $\SE(B, B^\star) \leqslant \delt$, and we need $\sigma_\eta^2 \leqslant \frac{r}{T} \delt^2 {\sigma_{\min}^\star}^2$. By using Assumption~\ref{assume:incoherence}, this holds if 
$$
\sigma_\eta^2 \leqslant c \frac{\delt^2}{\mu^2 \kappa^2} \| \thetats \|^2. 
$$
This implies that for the algorithm to converge to error level $\delt$, we need noise below this level. In other words, the error cannot go below the noise level. All rounds $\l > 0$ also need $\delt \leqslant \frac{0.02}{\mu \sqrt{r} \kappa^2}$. This is satisfied by setting $\delta_0 = \frac{0.02}{\mu \sqrt{r} \kappa^2}$. Thus, the initialization round needs
$$
\sigma_\eta^2 \leqslant \frac{c \| \thetats \|^2}{\mu^2 r^3 \kappa^6 \pG_1}. 
$$
In summary, let $\epsilon_{noise} = C \kappa^2 \sqrt{NSR}$, where $NSR := \frac{\sigma_\eta^2}{\min_t \| \thetats \|^2}$. From Theorem~\ref{new_3}, we haven shown that if $\delt \leqslant \frac{0.02}{\mu \sqrt{r} \kappa^2}$, $\gamma = \frac{c_\gamma}{{\sigma_{\max}^\star}^2}$ with $c_\gamma \leqslant 0.5$, and if $(\pG_m - \pG_{m-1}) T \geqslant C \mu^2 \kappa^4 (d + r) r$ and $(\pG_m - \pG_{m-1}) \geqslant C (r + \log T + \log d)$, then with probability at least $1 - \ell d^{-10}$, at each round $\l$, 
$$
\SE(B, B^\star) \leqslant \delt := (1 - \frac{0.5472 c_\gamma}{\kappa^2})^{\l} \delta_0 = (1 - \frac{0.5472 \textit{}c_\gamma}{\kappa^2})^{\l} \frac{0.02}{\mu \sqrt{r} \kappa^2}. 
$$
Thus, to guarantee $\SE(B_L, B^\star) \leqslant \epsilon_{noise}$, we need
$$
L = C \kappa^2 \log(\frac{1}{\max(\epsilon, \epsilon_{noise})}), 
$$
where it follows by using $\log x \geqslant 1 - \frac{1}{x}$, for $x\in (0,1]$. Thus, setting $c_\gamma = 0.4$, our sample complexity become $(\pG_m - \pG_{m-1}) T \geqslant C \mu^2 \kappa^6 d r^2 (\mu^2 \kappa^2 r + \log(\frac{1}{\max(\epsilon, \epsilon_{noise})}))$, and $\pG_m - \pG_{m-1} \geqslant C \kappa^2 (r + \log T + \log d) \log(\frac{1}{\max(\epsilon, \epsilon_{noise})})$. 
%\end{proof}
\qed

\section{Regret Analysis Proofs}\label{app_reg}

Now, our goal is to bound the per-epoch regret. 
In order to minimize overall regret, we must ensure that the regret incurred in each epoch is not too large because the overall regret is dominated by the epoch that has the largest regret \cite{han2020sequential}. To the end, we need to choose the epoch length in such away that the total; number of epochs $M = \lceil \log_2 \log_2 N \rceil$.

Guided by this observation, we can see intuitively an optimal way of selecting the grid must ensure that each batch’s regret is the same (at least orderwise in terms of the dependence of $T$ and $d$): for otherwise, there is a way of reducing the regret order in one batch and increasing the regret order in the other, and the sum of the two will still have a smaller regret order than before (which is dominated by the batch that has a larger regret order). As we shall see later, the following grid choice satisfies this equal-regret-across-batches requirement.

Let $\pR_m = \sum_{n=\pG_{m-1}+1}^{\pG_m} \sum_{t=1}^T \langle \pxtsc \thetats \rangle - \langle \pxtc \thetats \rangle$ denotes the cumulative regret incurred for all tasks during the $m-$th epoch. We will utilize this definition to determine its upper bound.

\begin{lemma} \label{Lemma3}
Assume that Assumptions~\ref{assume:iid} and~\ref{assume:incoherence} hold and $\sigma_\eta^2 \leqslant \frac{c \| \thetats \|^2}{r^3 \kappa^6 \pG_1}$. Set $\gamma = \frac{0.4}{{\sigma_{\max}^\star}^2}$ and $L = C \kappa^2 \log(\frac{1}{\max(\epsilon, \epsilon_{noise})})$. If $$(\pG_m - \pG_{m-1}) T \geqslant C \mu^2 \kappa^6 d r^2 (\mu^2 \kappa^2 r + \log(\frac{1}{\max(\epsilon, \epsilon_{noise})}))$$ and $$\pG_m - \pG_{m-1} \geqslant C \kappa^2 (r + \log T + \log d) \log(\frac{1}{\max(\epsilon, \epsilon_{noise})}),$$ then for any epoch $m \in [M]$, with probability at least $1 - \delta - L d^{-10}$ that
$$
\pR_m \leqslant 2 \mu \sigma_{\max}^\star \max(\epsilon, \epsilon_{noise}) \sqrt{rNT \log{\frac{1}{\delta}}}. 
$$
\end{lemma}
\begin{proof}
For any epoch $m \in [M]$, any task $t$, it follows that
\begin{align}
&\sum_{n=\pG_{m-1}+1}^{\pG_m} \pxtsc^\top \thetats - \pxtc^\top \thetats\nonumber\\
&= \sum_{n=\pG_{m-1}+1}^{\pG_m} \pxtsc^\top (\thetats - \thetahatmt) - \pxtc^\top \thetats + \pxtsc^\top \thetahatmt \nonumber \\
&\leqslant \sum_{n=\pG_{m-1}+1}^{\pG_m} \pxtsc^\top (\thetats - \thetahatmt) - \pxtc^\top \thetats + \pxtc^\top \thetahatmt \nonumber \\
&= \sum_{n=\pG_{m-1}+1}^{\pG_m} \pxtsc^\top (\thetats - \thetahatmt) - \pxtc^\top (\thetats - \thetahatmt) \nonumber
\end{align}
Since $\pxtc$ follows an i.i.d standard Gaussian distribution, we can determine that $\sum_{n=\pG_{m-1}+1}^{\pG_m} \pxtsc^\top (\thetats - \thetahatmt) - \pxtc^\top (\thetats - \thetahatmt) \sim \N(0, 2 (\pG_m - \pG_{m-1}) \norm{\thetats - \thetahatmt}^2)$. By utilizing the Chernoff bound for Gaussian stated in Proposition~\ref{proposition2}, with probability at least $1 - \delta$, 
$$
\sum_{n=\pG_{m-1}+1}^{\pG_m} \pxtsc^\top (\thetats - \thetahatmt) - \pxtc^\top (\thetats - \thetahatmt) \leqslant 2 \sqrt{(\pG_m - \pG_{m-1}) \log{\frac{1}{\delta}}} \norm{\thetats - \thetahatmt}
$$
Using a union bound and combining the result with Theorem~\ref{new_3} and Lemma~\ref{new_1}, we can find that with probability at least $1 - \delta - L d^{-10}$, we have 
\begin{align}
\pR_m &= \sum_{t=1}^T \sum_{n=\pG_{m-1}+1}^{\pG_m} \pxtsc^\top \thetats - \pxtc^\top \thetats \nonumber \\
&\leqslant 2 \sum_{t=1}^T \sqrt{(\pG_m - \pG_{m-1}) \log{\frac{1}{\delta}}} \norm{\thetats - \thetahatmt} \nonumber \\
&= 2 \sqrt{(\pG_m - \pG_{m-1}) \log{\frac{1}{\delta}}} \sum_{t=1}^T \norm{\thetats - \thetahatmt} \nonumber \\
&\leqslant 2 \sqrt{(\pG_m - \pG_{m-1}) \log{\frac{1}{\delta}}} \sum_{t=1}^T \max \{\norm{\thetats}, \norm{\thetats - \thetahatmt}\} \label{ineq7} \\
&\leqslant 2 \sqrt{(\pG_m - \pG_{m-1}) \log{\frac{1}{\delta}}} \cdot T \cdot  \mu \sqrt{\frac{r}{T}} \sigma_{\max}^{\star} \label{ineq5} \\
&\leqslant 2 \mu \sigma_{\max}^\star \sqrt{rNT \log{\frac{1}{\delta}}} \label{ineq6} 
\end{align}
where Eq.~\eqref{ineq7} is derived from the fact that in the first epoch, we perform random exploration since $\thetahatmt = 0$. Eq.~\eqref{ineq5} is derived from Assumption~\ref{assume:incoherence}, while Eq.~\eqref{ineq6} from $\pG_m - \pG_{m-1} \leqslant N$. 
\end{proof}

%\begin{theorem}
%If $\sigma_\eta^2 \leqslant \frac{r}{T} \delt^2 {\sigma_{\min}^\star}^2$, $\delt \leqslant \frac{0.02}{\sqrt{r} \kappa^2}$, $\gamma = \frac{c_\gamma}{(\pG_m - \pG_{m-1}) {\sigma_{\max}^\star}^2}$ with $c_\gamma \leqslant 0.5$, and if $(\pG_m - \pG_{m-1}) T \geqslant C \kappa^4 \mu^2 d r$ and $(\pG_m - \pG_{m-1}) \gtrsim \max(\log d, \log T, r)$, then with probability at least $O(1 - d^{-10} - \epsilon^2 - T \exp(r - c (\pG_m - \pG_{m-1})) - (NT)^{-2})$, the upper bound of cumulative regret for Algorithm~\ref{alg1} is
%$$\pR^{N, T} = \widetilde{O}(N^{\frac{3}{4}} r^{\frac{1}{2}} T^{\frac{1}{4}}). $$
%\end{theorem}
%\begin{proof}
{\em Proof of Theorem~\ref{thm-reg}.}
By applying the result of Lemma~\ref{Lemma3}, we can demonstrate that with probability at least $1 - \delta - L d^{-10}$, 
\begin{align}
\pR_{N, T} &= \sum_{m=1}^M \pR_m \nonumber \\
&\leqslant M 2 \mu \sigma_{\max}^\star \sqrt{rNT \log{\frac{1}{\delta}}} \nonumber \\
&\leqslant 2 \mu \sigma_{\max}^\star \sqrt{rNT \log{\frac{1}{\delta}}} (1 + \log \log N) \nonumber \\
&= O(r^{\frac{1}{2}} N^{\frac{1}{2}} T^{\frac{1}{2}} \sqrt{\log{\frac{1}{\delta}}} \log \log N) \nonumber \\
&= \widetilde{O}(r^{\frac{1}{2}} N^{\frac{1}{2}} T^{\frac{1}{2}}) \nonumber
\end{align}
where the last inequality is derived from $M = \lceil \log_2 \log_2 N \rceil$. 
%\end{proof}
\qed

\end{document}